\documentclass[11pt]{article}

\usepackage{geometry}
\usepackage{xr-hyper}
\usepackage{times,enumitem}
\usepackage[colorlinks]{hyperref}
\usepackage{url}
\usepackage{alex}
\usepackage{xcolor}

\usepackage{amsfonts,amsmath,amssymb,amsthm}
\usepackage{wrapfig}
\usepackage{verbatim,float,url}
\usepackage{graphicx,psfrag}
\usepackage{subcaption}
\usepackage{algorithm,algorithmic}
\usepackage[numbers]{natbib}
\usepackage{dsfont}
\usepackage[english]{babel}
\usepackage{cancel}
\usepackage{color}

\usepackage{mathtools}

\newcommand{\st}{\mathop{\mathrm{s.t.}}}

\newtheorem{theorem}{Theorem}
\newtheorem{lemma}[theorem]{Lemma}
\newtheorem{corollary}[theorem]{Corollary}
\theoremstyle{remark}

\theoremstyle{definition}

\def\E{\mathbb{E}}
\def\P{\mathbb{P}}

\def\half{\frac{1}{2}}

\def\df{\mathrm{df}}

\def\col{\mathrm{col}}
\def\row{\mathrm{row}}
\def\nul{\mathrm{null}}
\def\rank{\mathrm{rank}}
\def\nuli{\mathrm{nullity}}
\def\spa{\mathrm{span}}

\def\supp{\mathrm{supp}}

\def\conv{\mathrm{conv}}

\def\hbeta{\hat{\beta}}
\def\tbeta{\tilde{\beta}}

\def\R{\mathbb{R}}

\def\cG{\mathcal{G}}

\def\cN{\mathcal{N}}

\def\cS{\mathcal{S}}

\def\TV{\mathrm{TV}}
\def\MSE{\mathrm{MSE}}
\def\op{\Delta}
\def\tA{\tilde{A}}
\def\tB{\tilde{B}}
\def\tD{\tilde{D}}

\def\tK{\tilde{K}}
\def\tS{\tilde{\cS}}

\graphicspath{{figs/}}

\begin{document}

\title{Trend Filtering on Graphs} 

\author{
Yu-Xiang Wang$^{1}$ \\
{\tt  yuxiangw@cs.cmu.edu}\\
\and
James Sharpnack$^{3}$ \\
{\tt  jsharpna@gmail.com} \\
\and
 Alex Smola$^{1,4}$\\
{\tt  alex@smola.org} \\
\and
 Ryan J. Tibshirani$^{1,2}$\\
{\tt  ryantibs@stat.cmu.edu}\\
\and
\begin{tabular}{c}
  $^{1}$ Machine Learning Department, Carnegie Mellon University,
  Pittsburgh, PA 15213\\ 
  $^{2}$ Department of Statistics, Carnegie Mellon University,
  Pittsburgh, PA 15213\\  
  $^{3}$ Mathematics Department, University of California at San
  Diego, La Jolla, CA 10280\\ 
  $^{4}$ Marianas Labs, Pittsburgh, PA 15213\\ 
\end{tabular}
}

\date{}

\maketitle
	
	\begin{abstract}
		We introduce a family of adaptive estimators on graphs, based
		on penalizing the $\ell_1$ norm of discrete graph differences.
		This generalizes the idea of trend filtering
		\citep{l1tf,trendfilter}, used for univariate nonparametric
		regression, to graphs. Analogous to the univariate case, graph trend
		filtering exhibits a level of local adaptivity unmatched by the
		usual $\ell_2$-based graph smoothers.  It is also defined by a
		convex minimization problem that is readily solved (e.g., by fast
		ADMM or Newton algorithms).  We demonstrate the merits of graph
		trend filtering through both examples and theory. \\
		
		
	\end{abstract}
	
	\section{Introduction}
	\label{sec:intro}
	
	Nonparametric regression has a rich history in statistics, carrying
	well over 50 years of associated literature.  The goal of this paper
	is to port a successful idea in univariate nonparametric
	regression, trend filtering
	\citep{hightv,l1tf,trendfilter,fallfact}, to the setting of
	estimation on graphs. The proposed estimator, graph trend filtering,
	shares three key properties of trend filtering in the univariate
	setting.
	
	\begin{enumerate*}
		\item {\bf Local adaptivity:} graph trend filtering can adapt to
		inhomogeneity in the level of smoothness of an observed signal
		across nodes.  This stands in contrast to the usual
		$\ell_2$-based methods, e.g., Laplacian
		regularization \citep{graphlap}, which enforce smoothness globally
		with a much heavier hand, and tends to yield estimates
		that are either smooth or else wiggly throughout.
		
		\vspace{1mm}
		\item {\bf Computational efficiency:} graph trend filtering
		is defined by a regularized least squares problem, in which
		the penalty term is nonsmooth, but convex and structured enough
		to permit efficient large-scale computation.
		
		\vspace{1mm}
		\item {\bf Analysis regularization:} the graph trend filtering
		problem directly penalizes (possibly higher order) differences in
		the fitted signal across nodes.  Therefore
		graph trend filtering falls into what is called the {\it analysis}
		framework for defining estimators.
		Alternatively, in the {\it synthesis} framework, we would first
		construct a suitable basis over the graph, and then regress the
		observed signal over this basis; e.g., \citet{shuman2013} survey
		a number of such approaches using wavelets;
		likewise, kernel methods regularize in terms of the eigenfunctions
		of the graph Laplacian \citep{KonLaf02}. An advantage of analysis
		regularization is that it easily yields complex extensions of the
		basic estimator by mixing penalties.
	\end{enumerate*}
	
	As a motivating example,
	consider a denoising problem on 402 census tracts of Allegheny
	County, PA, arranged into a graph with 402 vertices and 2382 edges
	obtained by connecting spatially adjacent tracts. To illustrate the
	adaptive property of graph trend filtering we generated an
	artificial signal with inhomogeneous smoothness across the
	nodes, and two sharp peaks near the center of the graph, as can be
	seen in the top left panel of Figure~\ref{fig:pittsburgh-maps}.  (The
	signal was formed using a mixture of five Gaussians, in the underlying
	spatial coordinates.)
	We drew noisy observations around this signal, shown in the top right
	panel of the figure, and we fit graph trend filtering, graph Laplacian
	smoothing, and wavelet smoothing to these observations.  Graph trend
	filtering is to be defined in Section \ref{sec:graphtrend} (here we
	used $k=2$, quadratic order); the latter two, recall, are defined by
	the optimization problems
	\begin{align*}
	\min_{\beta \in \R^n} \, \|y-\beta\|_2^2 + \lambda \beta^\top L \beta
	\;\;\; \text{(Laplacian smoothing)}, \\
	\min_{\theta \in \R^n} \, \half \|y-W\theta\|_2^2 + \lambda \|\theta\|_1
	\;\;\; \text{(wavelet smoothing)},
	\end{align*}
	where $y \in \R^n$ the vector of observations measured over the
	$n=402$ nodes in the graph, $L \in \R^{n\times n}$ is the graph
	Laplacian matrix, and $W \in \R^{n\times n}$ is a wavelet basis built 
	over the graph.  The wavelet smoothing problem displayed above
	is really an oversimplified representation of the class of wavelets 
	methods, since it only encapsulates estimators that employ an
	orthogonal wavelet basis $W$ (and soft-threshold the wavelet 
	coefficients). For the present experiment, we
	constructed $W$ according to the spanning tree wavelet design of 
	\citet{graphwave}; we found this construction performed best among the
	graph wavelet designs we considered for the data at hand. For
	completeness, the results from alternative wavelet designs are
	given in the Appendix.  
	
	\begin{figure}[!htb]
		\centering
		\begin{tabular}{ccc}
			True signal &  Noisy observations & Graph trend filtering, 68 df \\
			\includegraphics[width=0.3\textwidth]{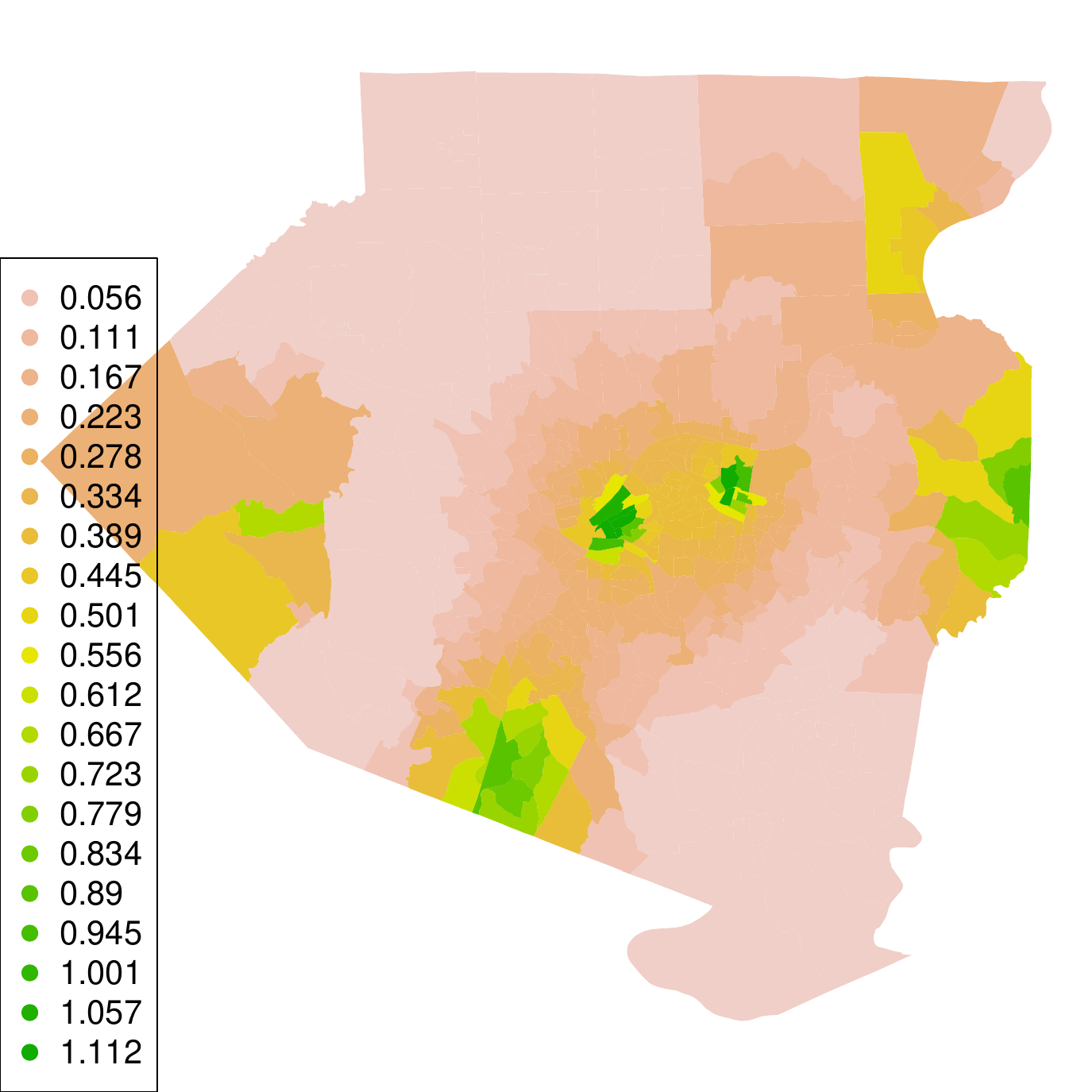} &
			\includegraphics[width=0.3\textwidth]{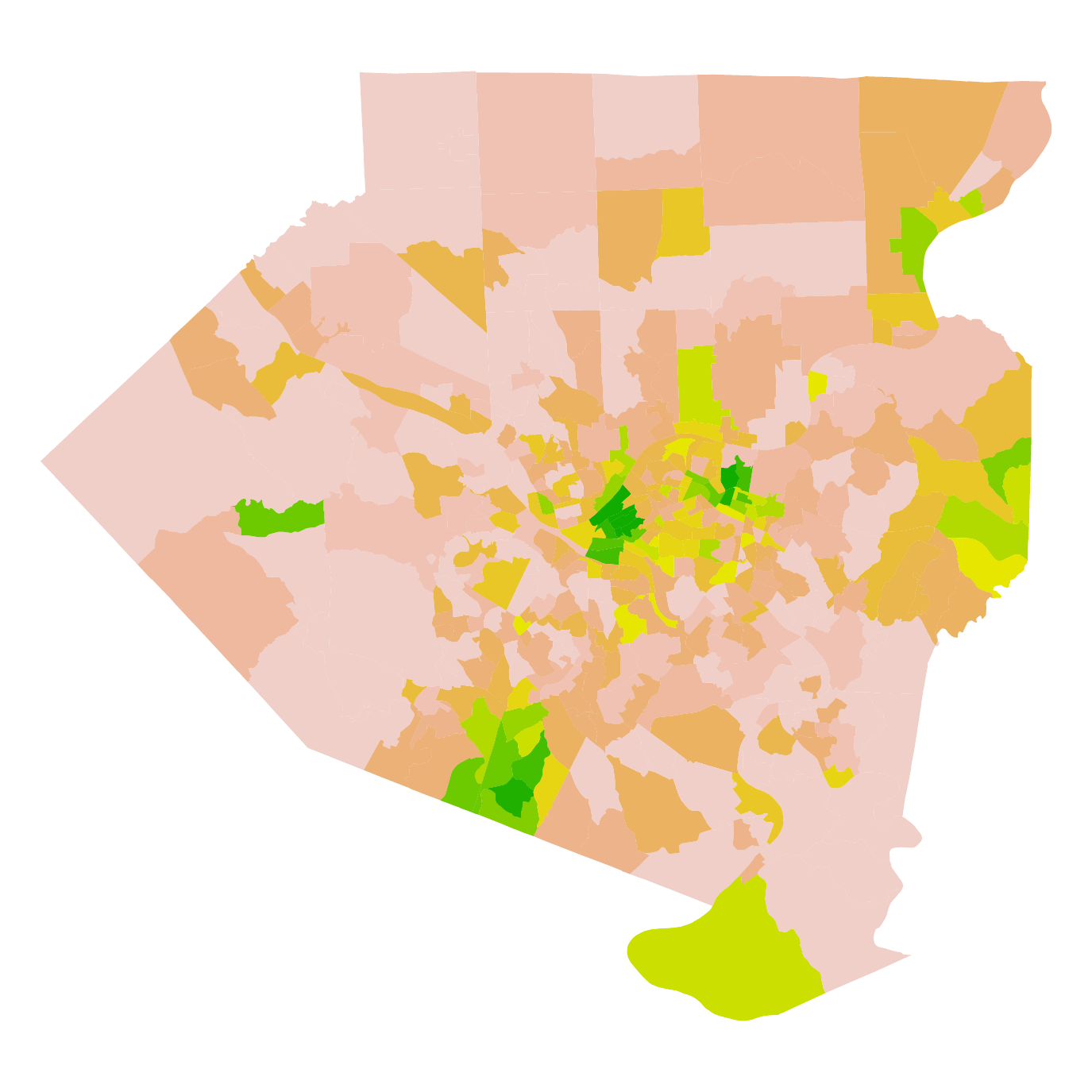} &
			\includegraphics[width=0.3\textwidth]{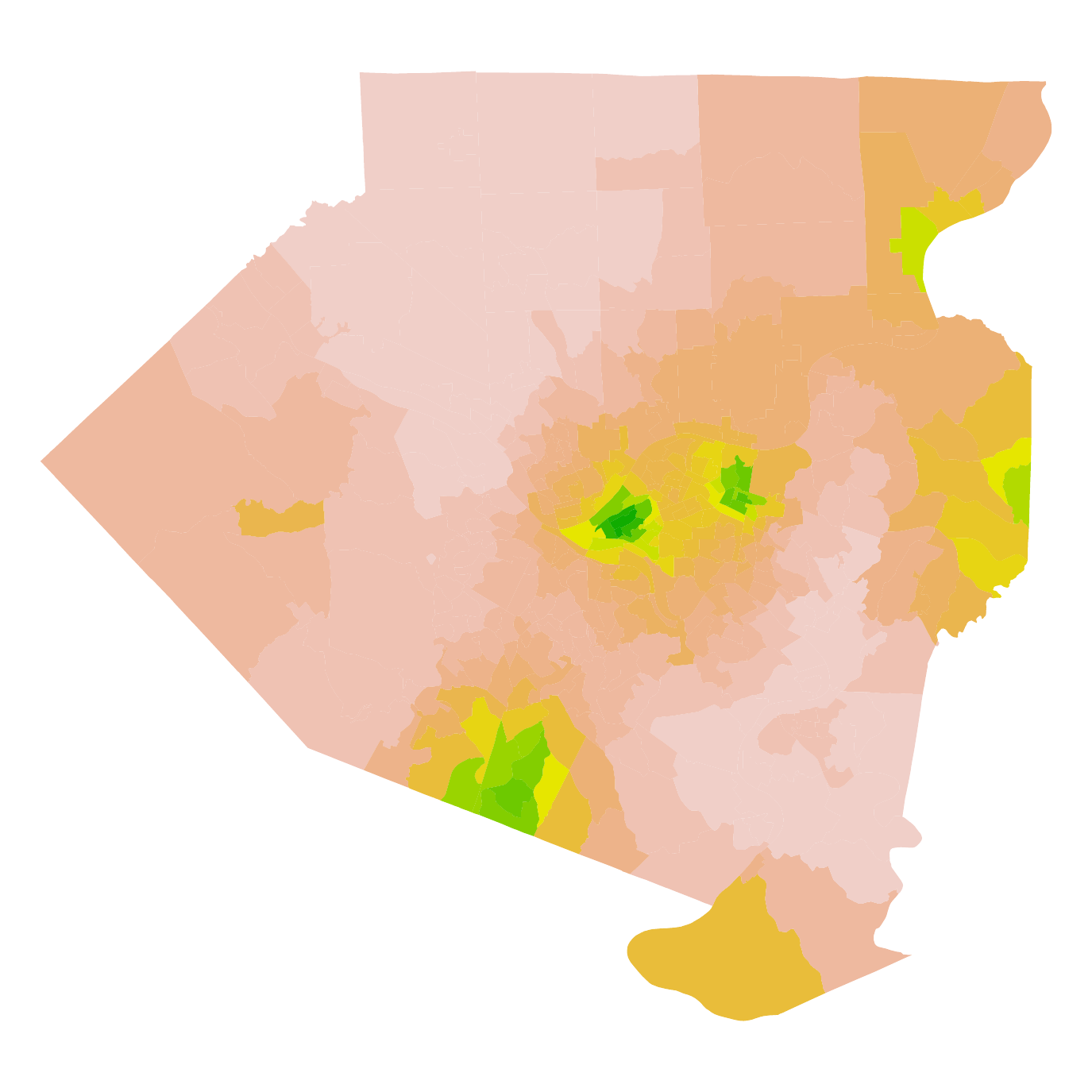} 
			\smallskip\smallskip \\
			Laplacian smoothing, 68 df & Laplacian smoothing, 132 df &
			Wavelet smoothing, 160 df \\
			\includegraphics[width=0.3\textwidth]{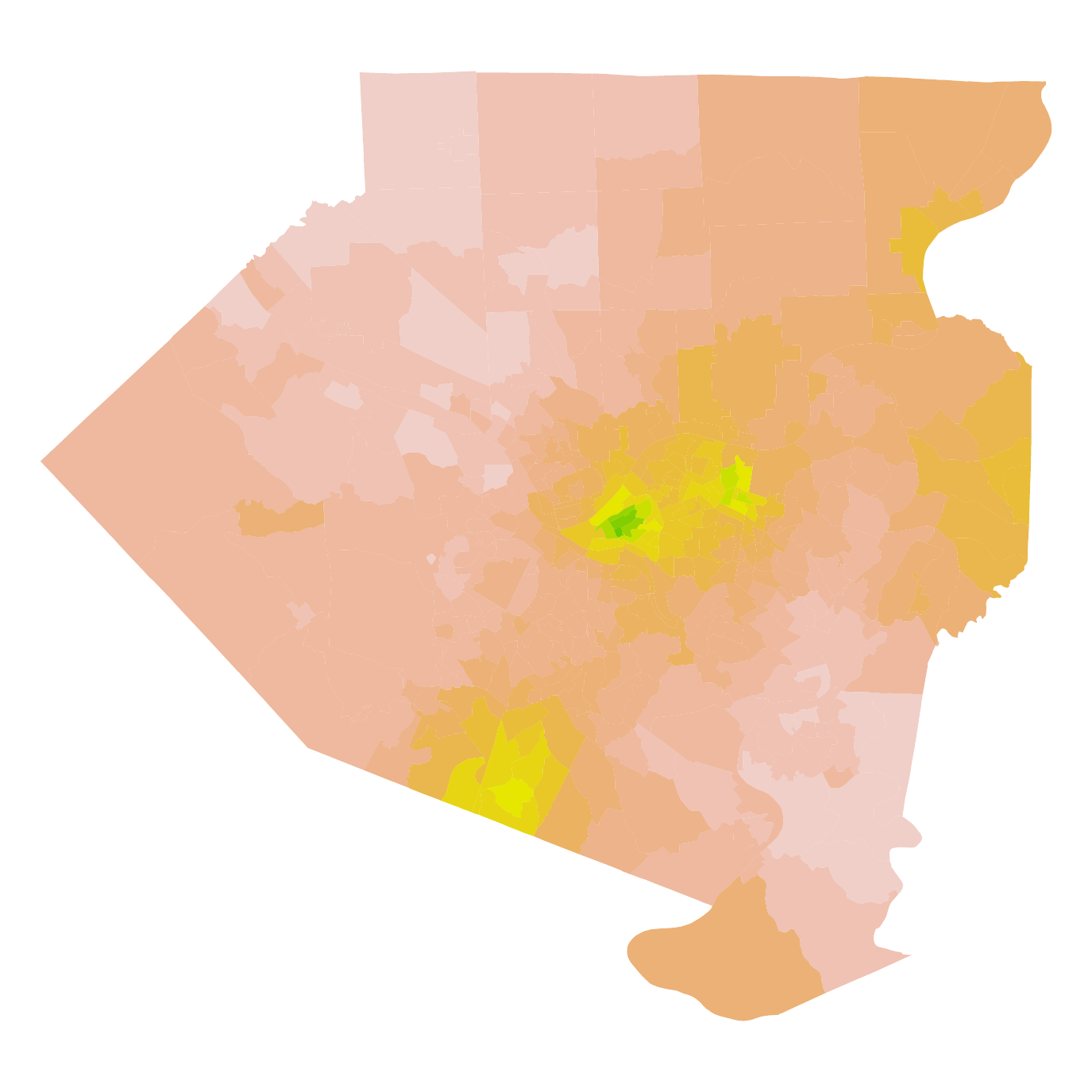} &
			\includegraphics[width=0.3\textwidth]{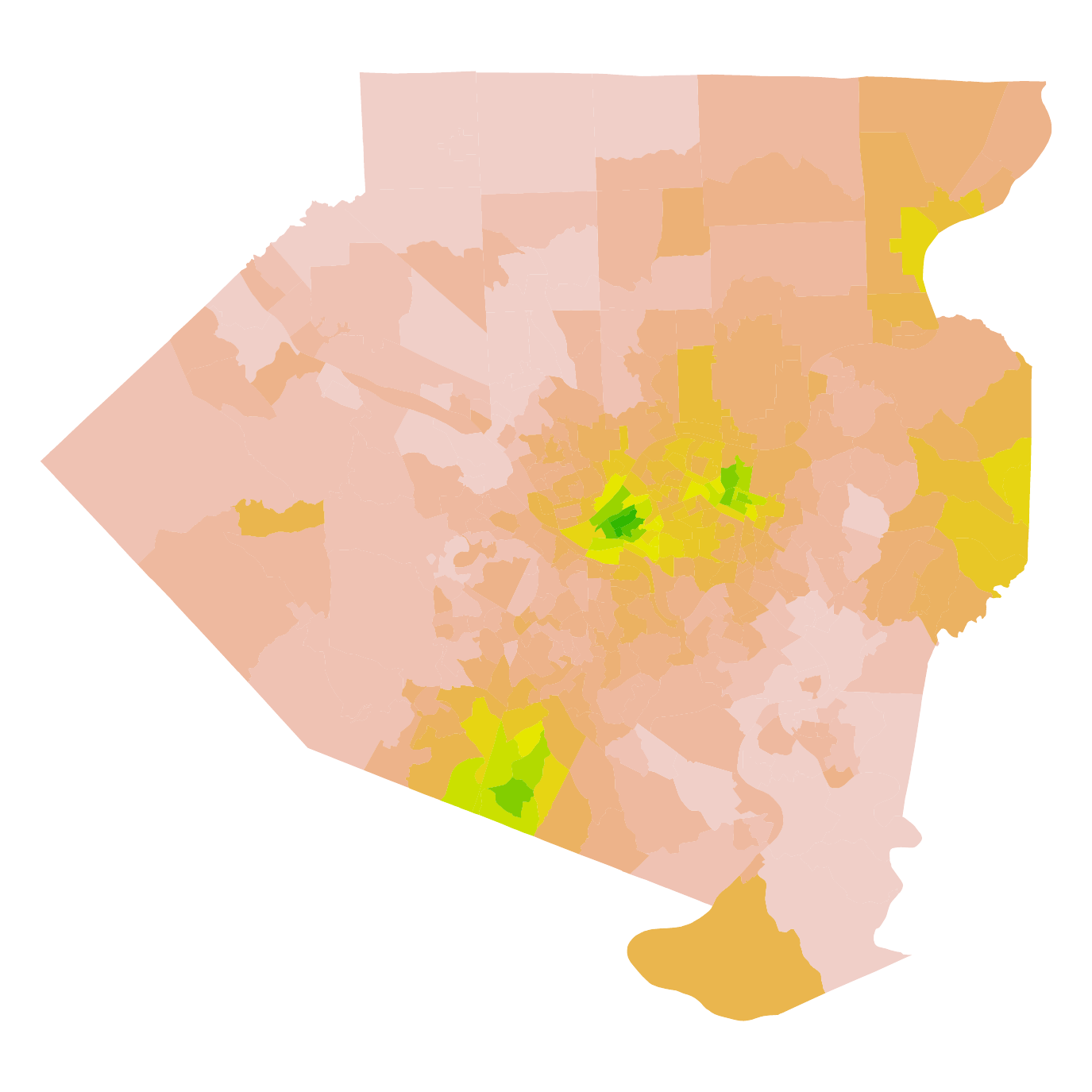} &
			\includegraphics[width=0.3\textwidth]{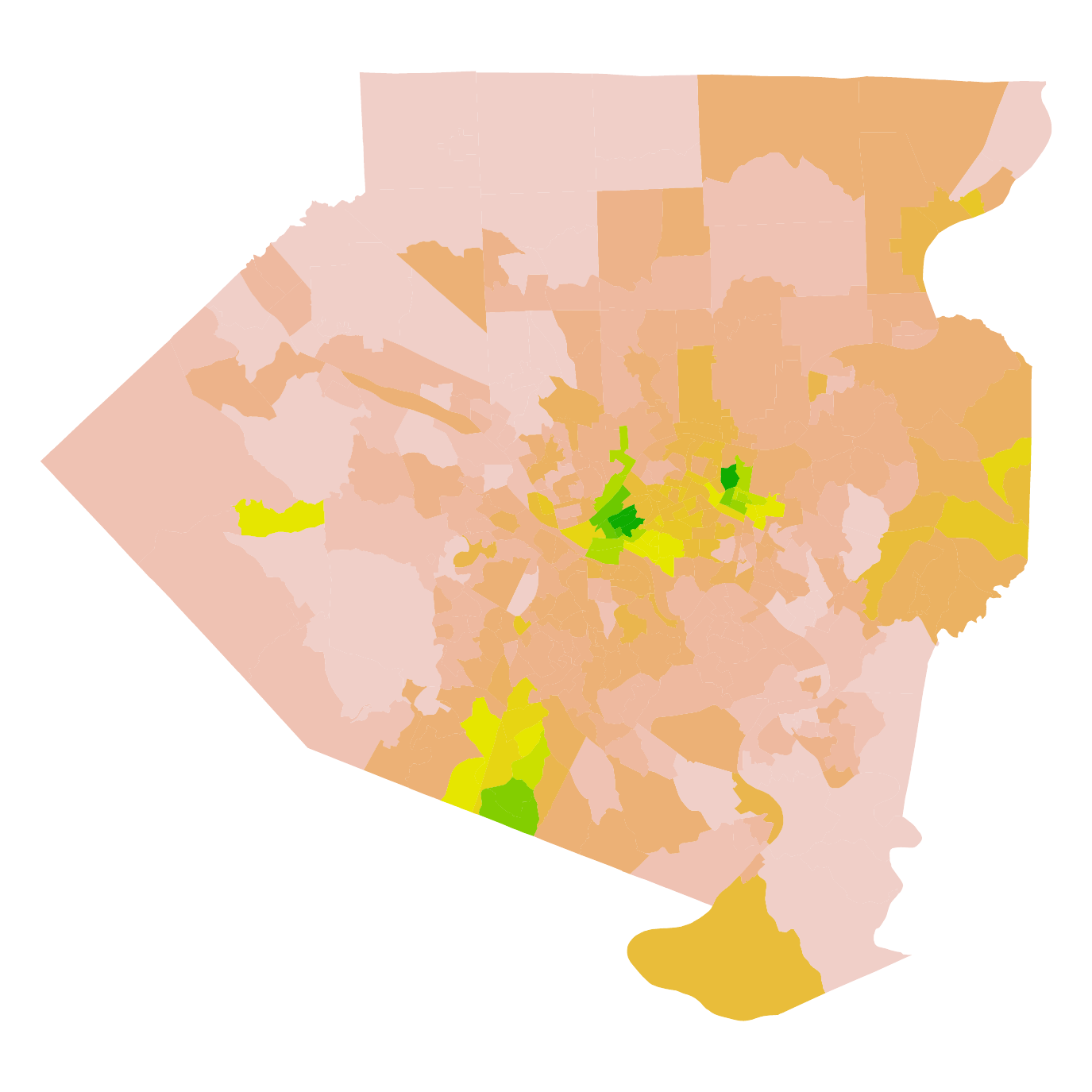}
		\end{tabular}
		\caption{Color maps for the Allegheny County example.}
		\label{fig:pittsburgh-maps}
	\end{figure}
	
	Graph trend filtering, Laplacian
	smoothing, and wavelet smoothing each have their own regularization
	parameters $\lambda$, and these parameters are not generally on the 
	same scale.  Therefore, in our comparisons we use effective degrees of
	freedom (df) as a  
	common measure for the complexities of the fitted models.   
	The top right panel of Figure \ref{fig:pittsburgh-maps} shows the
	graph trend filtering estimate with 68 df. We see 
	that it adaptively fits the sharp peaks in the center of the
	graph, and smooths out the surrounding regions appropriately. The
	graph Laplacian estimate with 68 df (bottom left),
	substantially oversmooths the high peaks in the center, while at 132
	df (bottom middle), it begins to detect the high peaks
	in the center, but undersmooths neighboring regions.  Wavelet
	smoothing performs quite poorly across all df values---it appears to
	be most affected by the level of noise in the observations. 
	
	As a more quantitative assessment, Figure \ref{fig:pittsburgh-errs}
	shows the mean squared errors between the estimates and the true
	underlying signal. The differences in
	performance here are analogous to the univariate case, when comparing
	trend filtering to smoothing splines \citep{trendfilter}. At
	smaller df values, Laplacian smoothing, due to its global
	considerations, fails to adapt to local differences across
	nodes. Trend filtering performs much better at low df
	values, and yet it matches Laplacian smoothing when both are
	sufficiently complex, i.e., in the overfitting regime. This
	demonstrates that the local flexibility of trend filtering estimates
	is a key attribute. 
	
	\begin{figure}[!htb]
		\centering
		\includegraphics[width=0.475\textwidth]{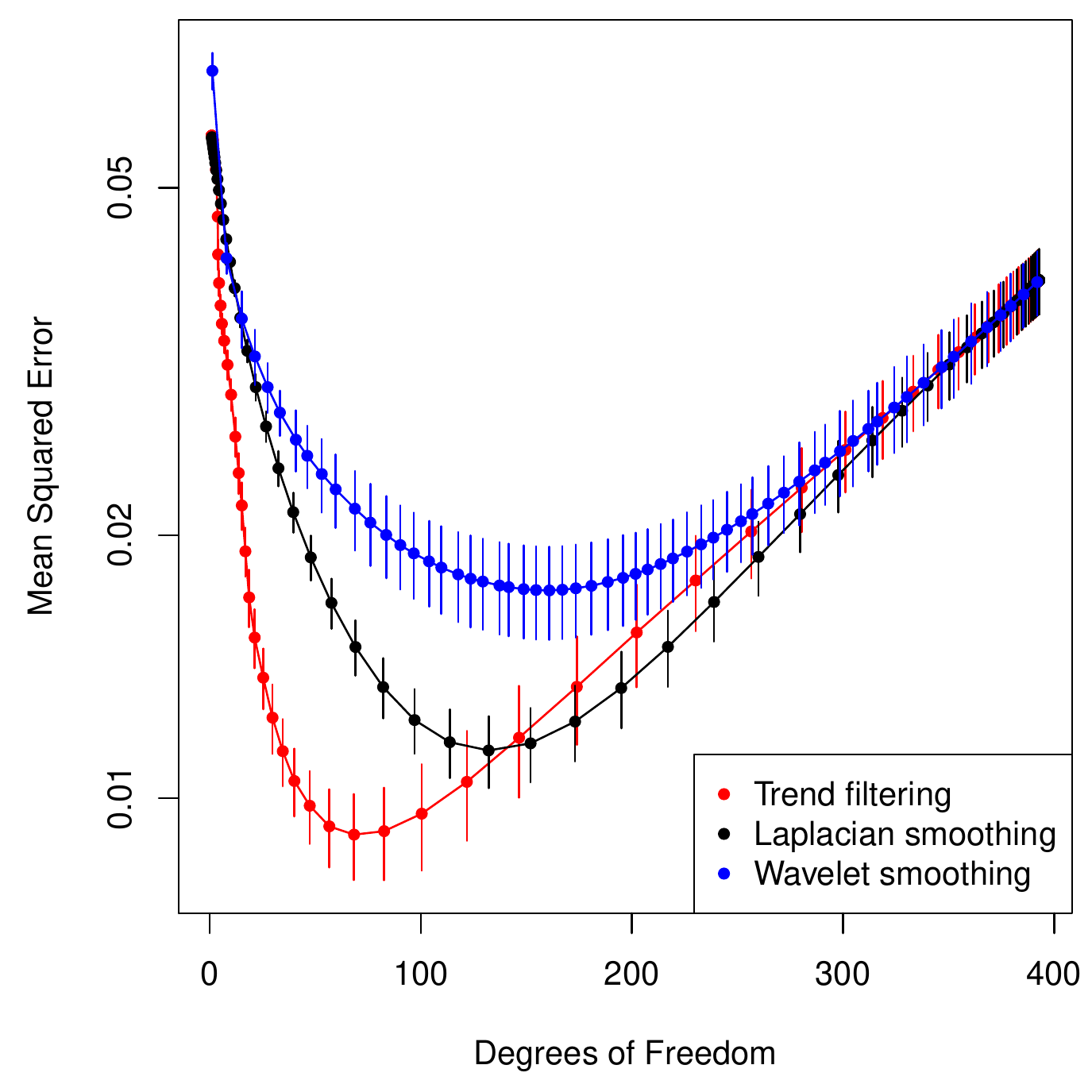}
		\caption{Mean squared errors for the Allegheny County
			example. Results were averaged over 10 simulations; the bars
			denote $\pm 1$ standard errors.
			\label{fig:pittsburgh-errs}}
	\end{figure}
	
	Here is an outline for the rest of this article.
	Section \ref{sec:graphtrend} defines graph trend filtering and gives
	underlying motivation and intuition.  Section \ref{sec:properties}
	covers basic properties and extensions of the graph trend filtering
	estimator.  Section \ref{sec:computation} examines computational
	approaches, and Section \ref{sec:examples} looks at a number of both
	real and simulated data examples.  Section \ref{sec:theory} presents
	asymptotic error bounds for graph trend filtering.  Section
	\ref{sec:discussion} concludes with a discussion.
	As for notation, we write $X_A$ to extract the rows of
	a matrix $X \in \R^{m\times n}$ that correspond to a subset $A
	\subseteq \{1,\ldots m\}$, and $X_{-A}$ to extract the complementary
	rows.  We use a similar convention for vectors: $x_A$ and $x_{-A}$
	denote the components of a vector $x \in \R^m$ that correspond to the
	set $A$ and its complement, respectively.  We write $\row(X)$ and
	$\nul(X)$ for the row and null spaces of $X$,  respectively, and
	$X^\dag$ for the pseudoinverse of $X$, with $X^\dag=(X^\top X)^\dag
	X^\top$ when $X$ is rectangular.

	\section{Trend Filtering on Graphs}
	\label{sec:graphtrend}
	
	In this section, we motivate and formally define graph trend filtering. 
	
	\subsection{Review: Univariate Trend Filtering}
	\label{sec:utf}
	
	We begin by reviewing trend filtering in the
	univariate setting, where discrete difference
	operators play a central role.  Suppose that we observe
	$y=(y_1,\ldots y_n) \in \R^n$ across input
	locations $x=(x_1,\ldots x_n) \in \R^n$; for simplicity, suppose that
	the inputs are evenly spaced, say, $x=(1,\ldots n)$. Given an integer
	$k \geq 0$, the $k$th order trend filtering estimate
	\smash{$\hbeta=(\hbeta_1,\ldots \hbeta_n)$} is
	defined as
	\begin{equation}
	\label{eq:tf}
	\hbeta = \argmin_{\beta \in \R^n} \, \half\|y-\beta\|_2^2 + \lambda
	\|D^{(k+1)} \beta\|_1,
	\end{equation}
	where $\lambda \geq 0$ is a tuning parameter, and $D^{(k+1)}$
	is the discrete difference operator of order $k+1$.  When $k=0$,
	problem \eqref{eq:tf} employs the first difference operator,
	\begin{equation}
	\label{eq:d1}
	D^{(1)} = \left[\begin{array}{rrrrrr}
	-1 & 1 & 0 & \ldots & 0 \\
	0 & -1 & 1 & \ldots & 0 \\
	\vdots & & \ddots & \ddots & \\
	0 & 0 & \ldots & -1 & 1
	\end{array}\right].
	\end{equation}
	Therefore
	\smash{$\|D^{(1)}\beta\|_1 =
		\sum_{i=1}^{n-1} |\beta_{i+1}-\beta_i|$},
	and the 0th order trend filtering estimate in \eqref{eq:tf}
	reduces to the 1-dimensional fused lasso estimator
	\citep{fuse}, also called 1-dimensional total variation denoising
	\citep{tv}. For $k \geq 1$ the operator $D^{(k+1)}$ is defined
	recursively by
	\begin{equation}
	\label{eq:dk}
	D^{(k+1)} = D^{(1)} D^{(k)},
	\end{equation}
	with $D^{(1)}$ above denoting the $(n-k-1) \times (n-k)$ version of
	the first difference operator in \eqref{eq:d1}.  In words, $D^{(k+1)}$ is
	given by taking first differences of $k$th differences.  The
	interpretation is hence that problem \eqref{eq:tf} penalizes the
	changes in the $k$th discrete differences of the fitted trend.  The
	estimated components \smash{$\hbeta_1,\ldots\hbeta_n$} exhibit the
	form of a $k$th order piecewise polynomial function, evaluated over
	the input locations $x_1,\ldots x_n$.  This can be formally verified
	\citep{trendfilter,fallfact} by examining a continuous-time analog of
	\eqref{eq:tf}.
	
	\subsection{Trend Filtering over Graphs}
	\label{sec:gtf}
	
	Let $G=(V,E)$ be an graph, with vertices
	$V=\cbr{1,\ldots n}$ and undirected edges $E=\cbr{e_1,\ldots e_m}$,
	and suppose that we observe $y=(y_1,\ldots y_n) \in \R^n$ over the
	nodes.  Following the univariate definition in
	\eqref{eq:tf}, we define the $k$th order {\it graph trend filtering}
	(GTF) estimate \smash{$\hbeta=(\hbeta_1,\ldots \hbeta_n)$} by
	\begin{equation}
	\label{eq:gtf}
	\hbeta = \argmin_{\beta \in \R^n} \,
	\half\|y-\beta\|_2^2 + \lambda \|\op^{(k+1)}\beta\|_1.
	\end{equation}
	In broad terms, this problem (like univariate trend filtering) is a
	type of generalized lasso problem \citep{genlasso}, in which the
	penalty matrix {$\op^{(k+1)}$} is a suitably defined {\it graph
		difference operator}, of order $k+1$. In fact, the novelty in our
	proposal lies entirely within the definition of this operator.
	
	When $k=0$, we define first order graph difference
	operator $\op^{(1)}$ in such a way it yields the
	graph-equivalent of a penalty on local differences:
	\begin{equation*}
	\|\op^{(1)}\beta\|_1 = \sum_{(i,j) \in E} |\beta_i-\beta_j|.
	\end{equation*}
	so that the penalty term in \eqref{eq:gtf} sums the
	absolute differences across connected nodes in $G$. To
	achieve this, we let $\op^{(1)} \in \cbr{-1, 0, 1}^{m \times n}$ be
	the oriented incidence matrix of the graph $G$, containing one row
	for each edge in the graph; specifically, if $e_\ell=(i,j)$, then
	$\op^{(1)}$ has $\ell$th row
	\begin{equation}
	\label{eq:gd1}
	\op^{(1)}_\ell = (0, \ldots
	\underset{\substack{ \;\;\uparrow \\ \;\;i }}{-1}, \ldots
	\underset{\substack{\uparrow \\ j}}{1}, \ldots 0),
	\end{equation}
	where the orientations of signs are arbitrary. Like trend filtering 
	in the 1d setting, the 0th order graph trend filtering estimate
	coincides with the fused lasso (total variation regularized) estimate
	over $G$ \citep{hoefling,genlasso,edgelasso}.  
	
	For $k \geq 1$, we use a recursion to
	define the higher order graph difference operators, in a
	manner similar to the univariate case.
	The recursion alternates in multiplying by the first difference
	operator $\op^{(1)}$ and its transpose (taking into account that
	this matrix not square):
	\begin{equation}
	\label{eq:gdk}
	\op^{(k+1)} =
	\begin{cases}
	(\op^{(1)})^\top \op^{(k)} = L^{\frac{k+1}{2}} & \text{for odd $k$} \\
	\op^{(1)} \op^{(k)} = D L^{\frac{k}{2}}  & \text{for even $k$}.
	\end{cases}
	\end{equation}
	Above, we  abbreviated the oriented incidence matrix $\op^{(1)}$ by
	$D$ of $G$, and exploited the fact that $L=D^\top D$ is one
	representation for the graph Laplacian matrix.  Note that
	$\op^{(k+1)} \in \R^{n \times n}$ for odd $k$, and $\op^{(k+1)} \in
	\R^{m \times n}$ for even $k$. 
	
	An important point is that our defined graph difference operators
	\eqref{eq:gd1}, \eqref{eq:gdk} reduce to the univariate ones
	\eqref{eq:d1}, \eqref{eq:dk} in the case of a chain graph (in which
	$V=\{1,\ldots n\}$ and $E=\{(i,i+1): i=1,\ldots n-1\}$), modulo
	boundary terms. That is, when $k$ is even, if one removes the
	first $k/2$ rows and last $k/2$ rows of $\op^{(k+1)}$ for the chain
	graph, then one recovers $D^{(k+1)}$; when $k$ is odd, if one removes
	the first and last $(k+1)/2$ rows of $\op^{(k+1)}$ for the chain
	graph, then one recovers $D^{(k+1)}$.  Further intuition for our graph
	difference operators is given next.
	
	\subsection{Piecewise Polynomials over Graphs}
	
	We give some insight for our definition of graph
	difference operators \eqref{eq:gd1}, \eqref{eq:gdk}, based on the
	idea of piecewise polynomials over graphs. In the univariate case, as
	described in Section \ref{sec:utf},
	sparsity of $\beta$ under the difference operator $D^{(k+1)}$ implies
	a specific $k$th order piecewise polynomial structure for the
	components of $\beta$ \citep{trendfilter,fallfact}.  Since the
	components of $\beta$ correspond to (real-valued) input locations
	$x=(x_1,\ldots x_n)$, the interpretation of a piecewise
	polynomial here  is unambiguous.  But for a graph, one might ask:
	does sparsity of $\op^{(k+1)}\beta$ mean that the components of
	$\beta$ are piecewise polynomial?  And what does the latter even mean,
	as the components of  $\beta$ are defined over the nodes?  To address
	these questions, we intuitively {\it define} a piecewise polynomial
	over a graph, and show that it implies sparsity under our constructed
	graph difference operators.
	
	\begin{itemize}
		\item {\bf Piecewise constant ($k=0$):} we say that a signal $\beta$
		is piecewise constant over a graph $G$ if many of the
		differences $\beta_i-\beta_j$ are zero across edges $(i,j) \in E$ in
		$G$. Note that this is exactly the property associated with
		sparsity of $\op^{(1)}\beta$, since $\op^{(1)}=D$, the oriented
		incidence matrix of $G$.
		\item {\bf Piecewise linear ($k=1$):}  we say that a signal $\beta$ has a
		piecewise linear structure over $G$ if $\beta$ satisfies
		\begin{equation*}
		\beta_i - \frac{1}{n_i} \sum_{(i,j)  \in E} \beta_j=0,
		\end{equation*}
		for many nodes $i \in V$, where $n_i$ is the number of nodes
		adjacent to $i$. In words, we are requiring that the signal
		components can be linearly interpolated from its neighboring values
		at many nodes in the graph.  This is quite a natural notion of
		(piecewise) linearity: requiring that $\beta_i$ be equal to the
		average of its neighboring values would enforce linearity at
		$\beta_i$ under an appropriate embedding of the points in Euclidean
		space.  Again, this is precisely the same as requiring
		$\op^{(2)}\beta$ to be sparse, since $\op^{(2)}=L$, the graph
		Laplacian.
		\item {\bf Piecewise polynomial ($k \geq 2$):}
		We say that $\beta$ has a piecewise quadratic structure
		over $G$ if the first differences $\alpha_i-\alpha_j$ of
		the second
		differences $\alpha=\op^{(2)}\beta$ are mostly zero, over
		edges $(i,j) \in E$.  Likewise, $\beta$ has a piecewise cubic
		structure over $G$ if the second differences
		\smash{$\alpha_i-\frac{1}{n_i}\sum_{(i,j)  \in E} \alpha_j$}
		of the second differences $\alpha=\op^{(2)}\beta$
		are mostly zero, over nodes $i \in V$. This argument extends,
		alternating between leading first and second differences for even
		and odd $k$.  Sparsity of $\op^{(k+1)} \beta$ in either case exactly
		corresponds to many of these differences being zero, by
		construction.
	\end{itemize}
	
	In Figure \ref{fig:grid}, we illustrate the graph trend filtering
	estimator on a 2d grid graph of dimension $20\times 20$, i.e., a grid
	graph with 400 nodes and 740 edges.  For each of the cases $k=0,1,2$,
	we generated synthetic measurements over the grid, and computed a
	GTF estimate of the corresponding order.  We chose the 2d grid setting
	so that the piecewise polynomial nature of GTF estimates could be
	visualized.  Below each plot, the utilized graph trend filtering
	penalty is displayed in more explicit detail.
	
	\begin{figure}[htbp]
		\centering
		\begin{tabular}{cc}
			GTF with $k=0$ & GTF with $k=1$ \\
			\includegraphics[width=0.45\textwidth]{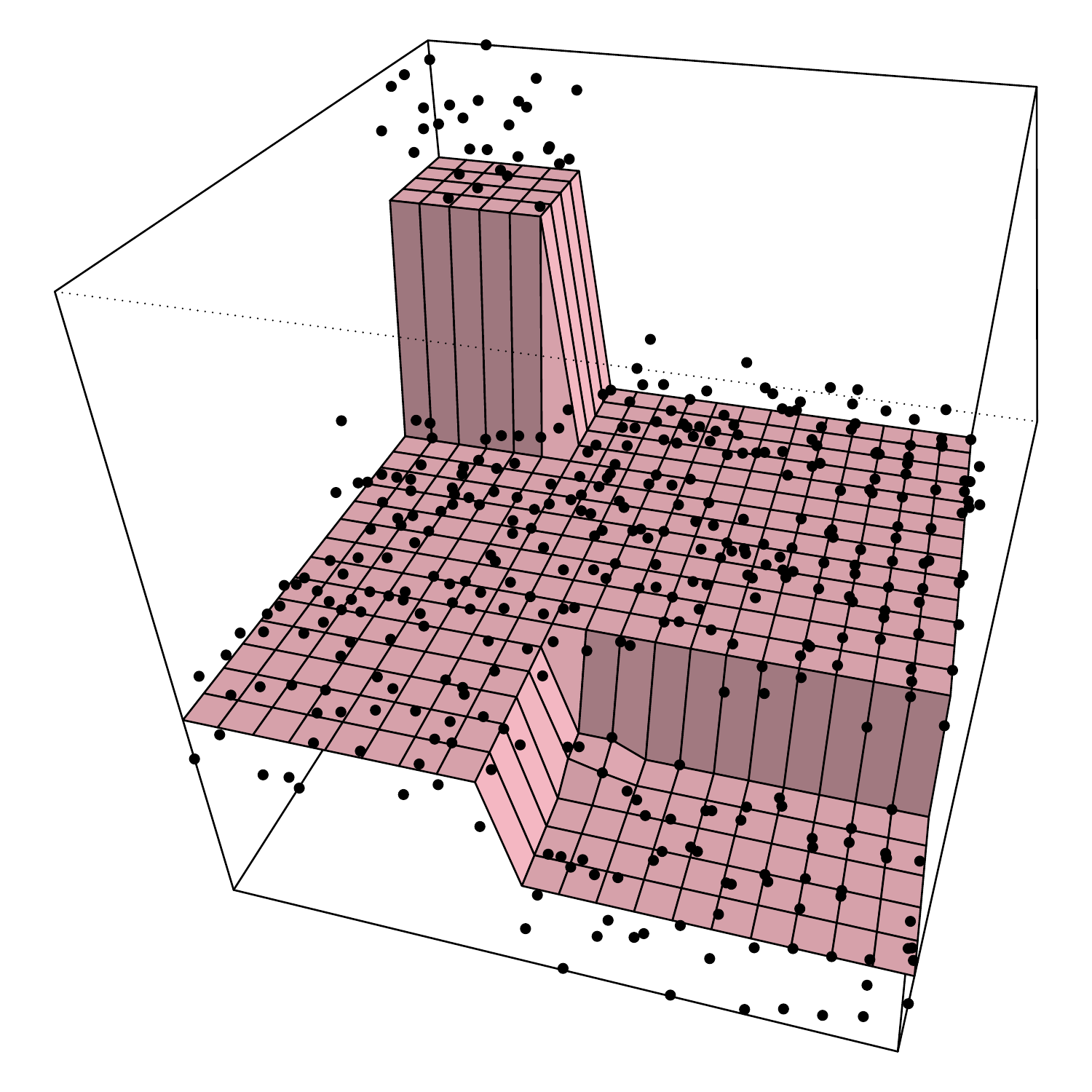} &
			\includegraphics[width=0.45\textwidth]{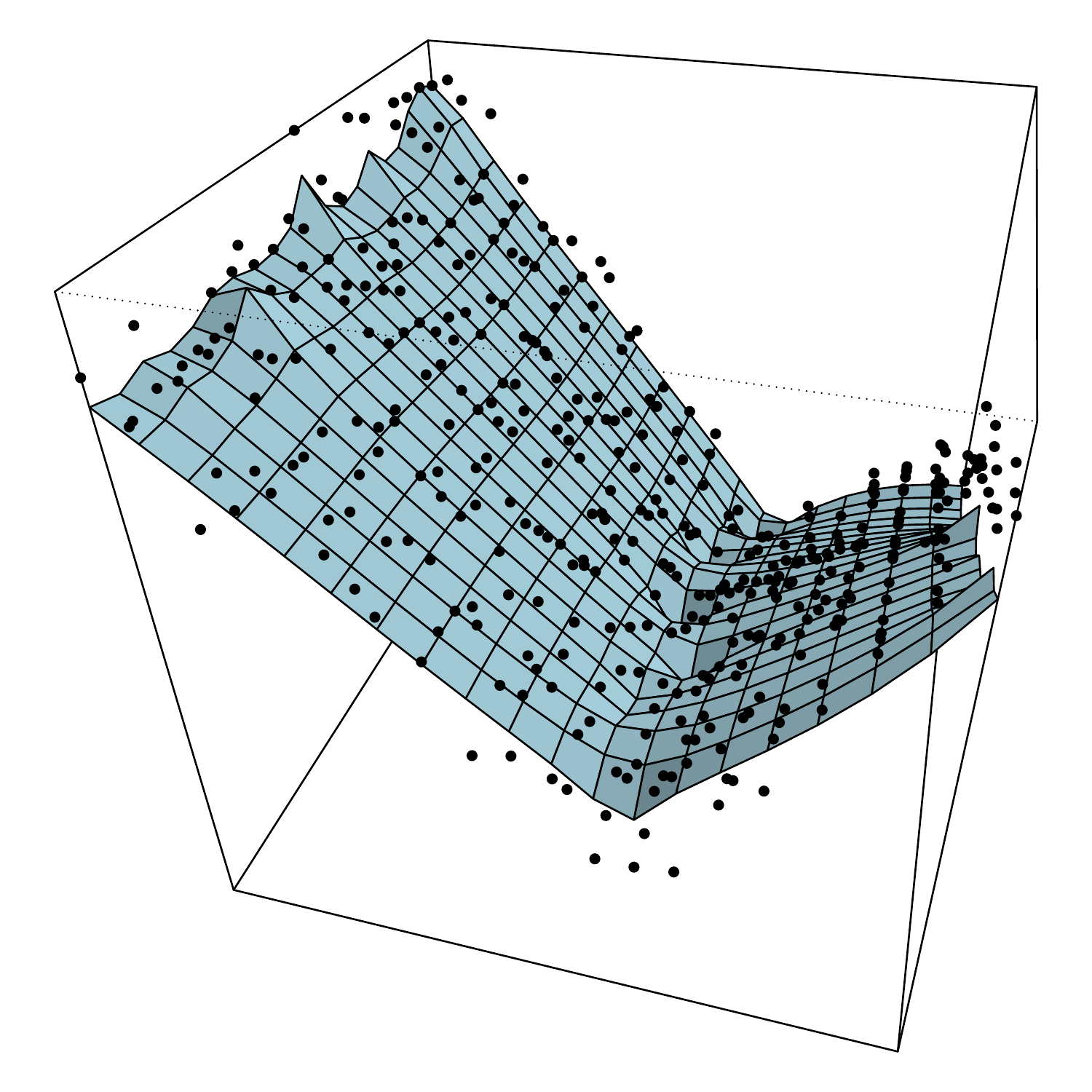} \\
			Penalty: $\displaystyle \sum_{(i,j)\in E} |\beta_i-\beta_j|$ &
			$\displaystyle \sum_{i=1}^n n_i
			\bigg|\beta_i-\frac{1}{n_i} \sum_{j : (i,j)\in E} \beta_j \bigg|$
		\end{tabular}
		
		\bigskip
		GTF with $k=2$ \\
		\includegraphics[width=0.45\textwidth]{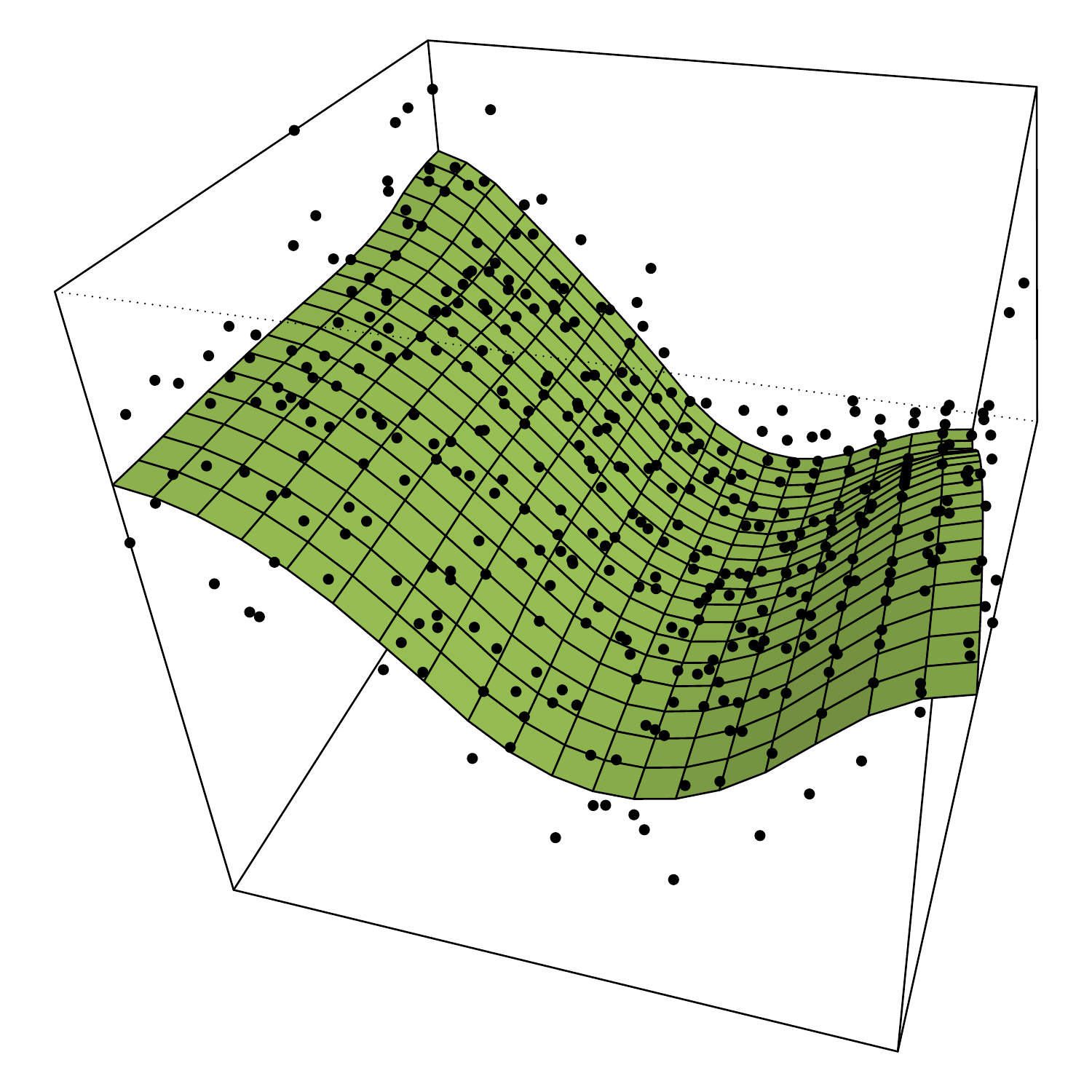} \\
		$\displaystyle
		\sum_{(i,j) \in E} \Bigg| n_i\bigg(\beta_i-\frac{1}{n_i}
		\sum_{\ell : (i,\ell)\in E}  \beta_\ell \bigg) - 
		n_j \bigg(\beta_j-\frac{1}{n_j}
		\sum_{\ell : (j,\ell)\in E} \beta_\ell \bigg)\Bigg|$
		
		\bigskip
		\caption{Graph trend filtering estimates of orders $k=0,1,2$ on a 2d
			grid. The utilized $\ell_1$ graph difference penalties are shown in
			elementwise detail below each plot (first, second, and third order
			graph differences).}
		\label{fig:grid}
	\end{figure}
	
	\subsection{$\ell_1$ versus $\ell_2$ Regularization}
	
	It is instructive to compare the graph trend filtering
	estimator, as defined in \eqref{eq:gtf}, \eqref{eq:gd1},
	\eqref{eq:gdk} to Laplacian smoothing \citep{graphlap}.  Standard
	Laplacian smoothing uses the same least squares loss as in
	\eqref{eq:gtf}, but replaces the penalty term with $\beta^\top L
	\beta$.  A natural generalization would be to allow for a power of the
	Laplacian matrix $L$, and define $k$th order graph Laplacian smoothing
	according to
	\begin{equation}
	\label{eq:gls}
	\hbeta = \argmin_{\beta \in \R^n} \,
	\|y-\beta\|_2^2 + \lambda \beta^\top L^{k+1} \beta.
	\end{equation}
	The above penalty term can be written as
	\smash{$\|L^{(k+1)/2} \beta\|_2^2$} for odd $k$, and
	\smash{$\|DL^{k/2}\beta\|_2^2$} for even $k$; i.e., this
	penalty is exactly \smash{$\|\op^{(k+1)} \beta\|_2^2$} for the
	graph difference operator \smash{$\op^{(k+1)}$} defined previously.
	
	As we can see, the critical difference between graph Laplacian
	smoothing \eqref{eq:gls} and graph trend filtering \eqref{eq:gtf} lies
	in the choice of penalty norm: $\ell_2$ in the former, and $\ell_1$ in
	the latter. The effect of the
	$\ell_1$ penalty is that the GTF program can set
	many (higher order) graph differences to zero exactly, and leave
	others at large nonzero values; i.e., the GTF estimate can
	simultaneously be smooth in some parts of the graph and wiggly in
	others.  On the other hand, due to the (squared) $\ell_2$ penalty,
	the graph Laplacian smoother cannot set any graph
	differences to zero exactly, and roughly speaking, must choose between
	making all graph differences small or large.  The relevant
	analogy here is the comparison between the lasso and ridge regression,
	or univariate trend filtering and smoothing splines
	\citep{trendfilter}, and the suggestion is that GTF can
	adapt to the proper local degree of smoothness, while Laplacian
	smoothing cannot.  This is strongly supported by the examples given
	throughout this paper.
	
	\subsection{Related Work}

	Some authors from the signal processing community, e.g.,
	\citet{bredies2010total,setzer2011infimal}, have studied total
	generalized variation (TGV), a higher order variant of total
	variation regularization.  Moreover, several discrete versions of
	these operators have been proposed. They are often similar to the
	construction that we have.
	However, the focus of these works is mostly on how well a discrete 
	functional approximates its continuous counterpart. This is quite
	different from our concern, as a signal on a graph (say a social
	network) may not have any meaningful continuous-space embedding at
	all. In addition, we are not aware of any study on the statistical
	properties of these regularizers. In fact, our theoretical analysis in
	Section~\ref{sec:theory} may be extended to these methods too.

	\section{Properties and Extensions}
	\label{sec:properties}
	
	We first study the structure of graph trend filtering
	estimates, then discuss interpretations and extensions.
	
	\subsection{Basic Structure and Degrees of Freedom}
	
	We describe the basic structure of graph trend filtering
	estimates and present an unbiased estimate for their degrees of
	freedom.  Let the tuning parameter $\lambda$ be
	arbitrary but fixed. By virtue of the $\ell_1$ penalty in
	\eqref{eq:gtf}, the solution \smash{$\hbeta$} satisfies
	$\supp(\op^{(k+1)} \hbeta) = A$ for some active set $A$
	(typically $A$ is smaller when $\lambda$ is larger).  Trivially,
	we can reexpress this as \smash{$\op^{(k+1)}_{-A} \hbeta = 0$}, or
	\smash{$\hbeta \in \nul(\op^{(k+1)}_{-A})$}.  Therefore, the basic
	structure of GTF estimates is revealed by analyzing the null space of
	the suboperator \smash{$\op^{(k+1)}_{-A}$}.
	
	\begin{lemma}
		\label{lem:nullspace}
		Assume without a loss of generality that $G$ is connected
		(otherwise the results apply to each connected component of $G$).
		Let $D,L$ be the oriented incidence matrix and Laplacian matrix of
		$G$.
		For even $k$, let $A \subseteq \{1,\ldots m\}$, and let $G_{-A}$
		denote the subgraph induced by removing the edges indexed by $A$
		(i.e., removing edges $e_\ell$, $\ell \in A$).  Let $C_1,\ldots C_s$
		be the connected components of $G_{-A}$.  Then
		\begin{equation*}
		\nul(\op^{(k+1)}_{-A}) = \spa\{\mathds{1}\} +
		(L^\dag)^{\frac{k}{2}} \spa\{\mathds{1}_{C_1},\ldots
		\mathds{1}_{C_s}\},
		\end{equation*}
		where $\mathds{1}=(1,\ldots 1) \in \R^n$, and
		$\mathds{1}_{C_1},\ldots \mathds{1}_{C_s} \in \R^n$ are the indicator
		vectors over connected components.  For odd $k$, let
		$A \subseteq \{1,\ldots n\}$.  Then
		\begin{equation*}
		\nul(\op^{(k+1)}_{-A}) = \spa\{\mathds{1}\} +
		\{ (L^\dag)^{\frac{k+1}{2}} v : v_{-A}=0 \}.
		\end{equation*}
	\end{lemma}
	
	The proof of Lemma \ref{lem:nullspace} appears in the Appendix. 
	The lemma is useful for a few reasons.  First, as motivated 
	above, it describes the coarse structure of GTF solutions.  When
	$k=0$, we can see (as
	\smash{$(L^\dag)^{0/2}=I$}) that \smash{$\hbeta$} will indeed be 
	piecewise constant over groups of nodes $C_1,\ldots C_s$ of $G$.  For
	$k=2,4,\ldots$, this structure is smoothed by multiplying such
	piecewise constant  levels by \smash{$(L^\dag)^{k/2}$}.
	Meanwhile, for $k=1,3\ldots$, the structure of the GTF estimate is
	based on assigning nonzero values to a subset $A$ of nodes, and then
	smoothing through multiplication by
	\smash{$(L^\dag)^{(k+1)/2}$}. Both of these smoothing operations,
	which depend on $L^\dag$, have interesting interpretations in terms of
	to the electrical network perspective for graphs. This is developed in
	the next subsection.
	
	
	A second use of Lemma \ref{lem:nullspace} is that it leads to a simple
	expression for the degrees of freedom, i.e., the effective number of
	parameters, of the
	GTF estimate \smash{$\hbeta$}. From results on generalized lasso
	problems \citep{genlasso,lassodf2}, we have \smash{$\df(\hbeta) =
		\E[\nuli(\op^{(k+1)}_{-A})]$}, with $A$ denoting the support of
	\smash{$\op^{(k+1)} \hbeta$}, and $\nuli(X)$ the dimension of the null
	space of a matrix $X$.  Applying Lemma \ref{lem:nullspace} then gives
	the following.
	
	\begin{lemma}
		\label{lem:df}
		Assume that $G$ is connected. Let \smash{$\hbeta$} denote the
		GTF estimate at a fixed but arbitrary
		value of $\lambda$. Under the normal error model $y\sim
		\cN(\beta_0,\sigma^2 I)$, the GTF estimate \smash{$\hbeta$} has
		degrees of freedom given by
		\begin{align}
		\nonumber
		\df(\hbeta) =
		\begin{cases}
		\E\sbr{\max\cbr{|A|,1}} &
		\text{odd $k$}, \\
		\E\sbr{
			\text{\small number of connected components of }\, G_{-A}} &
		\text{even $k$}.
		\end{cases}
		\end{align}
		Here \smash{$A=\supp(\op^{(k+1)}\hbeta)$} denotes the active set of
		\smash{$\hbeta$}.
	\end{lemma}
	
	As a result of Lemma \ref{lem:df}, we can form simple unbiased
	estimate for \smash{$\df(\hbeta)$}; for $k$ odd, this is
	$\max\{|A|,1\}$, and for $k$ even, this is the number of connected
	components of $G_{-A}$, where $A$ is the support of
	\smash{$\op^{(k+1)}\hbeta$}.  When reporting degrees of freedom for
	graph trend filtering (as in the example in the introduction), we use
	these unbiased estimates.
	
	\subsection{Electrical Network Interpretation}
	
	Lemma \ref{lem:nullspace} reveals a mathematical
	structure for GTF estimates \smash{$\hbeta$}, which satisfy
	\smash{$\hbeta \in \nul(\op^{(k+1)}_{-A})$} for some set $A$.  It is 
	interesting to interpret the results using the electrical network
	perspective for graphs \citep{lxb}.  In this perspective, we imagine 
	replacing each edge in the graph with a resistor of value 1.  If $u
	\in \R^n$ describes how much current is going in at
	each node in the graph, then $v=Lu$ describes the induced
	voltage at each node. Provided that $\mathds{1}^\top c = 0$, which means
	that the total accumulation of current in the network is 0, we can
	solve for the current values from the voltage values: $u=L^\dag v$.
	
	The odd case in Lemma \ref{lem:nullspace} asserts that
	\begin{equation*}
	\nul(\op^{(k+1)}_{-A}) = \spa\{\mathds{1}\} +
	\{ (L^\dag)^{\frac{k+1}{2}} v : v_{-A}=0 \}.
	\end{equation*}
	For $k=1$, this says that GTF estimates are formed by assigning a
	sparse number of nodes in the graph a nonzero voltage $v$, then
	solving for the induced current $L^\dag v$ (and shifting this entire
	current vector by a constant amount).  For $k=3$, we assign a sparse
	number of nodes a nonzero voltage, solve for the induced current, and
	then {\it repeat this}: we relabel the induced current as input
	voltages to the nodes, and compute the new induced current.  This
	process is again iterated for $k=5,7,\ldots$.
	
	
	The even case in Lemma \ref{lem:nullspace} asserts that
	\begin{equation*}
	\nul(\op^{(k+1)}_{-A}) = \spa\{\mathds{1}\} +
	(L^\dag)^{\frac{k}{2}} \spa\{\mathds{1}_{C_1},\ldots
	\mathds{1}_{C_s}\}.
	\end{equation*}
	For $k=2$, this result says that GTF estimates are given by choosing a
	partition $C_1,\ldots C_s$ of the nodes, and assigning a constant
	input voltage to each element of the partition.  We then solve for the
	induced current (and potentially shift this by an overall constant
	amount). The process is iterated for $k=4,6,\ldots$ by relabeling the
	induced current as input voltage.
	
	The comparison between the
	structure of estimates for $k=2$ and $k=3$ is informative: in a sense,
	the above tells us that 2nd order GTF estimates will be
	{\it smoother} than 3rd order estimates, as a sparse input
	voltage vector need not induce a
	current that is piecewise constant over nodes in the graph. For
	example, an input voltage vector that has only a few nodes with very
	large nonzero values will induce a current that is peaked around these
	nodes, but not piecewise constant.
	
	\subsection{Extensions}
	
	Several extensions of the proposed graph trend filtering model are
	possible. Trend filtering over a weighted graph, for example, could be
	performed by using a properly weighted version of the edge incidence
	matrix in \eqref{eq:gd1}, and carrying forward the same recursion in
	\eqref{eq:gdk} for the higher order difference operators.
	As another example, the Gaussian regression loss in \eqref{eq:gtf}
	could be changed to another suitable likelihood-derived losses
	in order to accommodate a different data type for $y$, say, logistic
	regression loss for binary data, or Poisson regression loss for count
	data.
	
	In Section \ref{sec:madgtf}, we explore a modest extension of GTF,
	where we add a strongly convex prior term to the criterion in
	\eqref{eq:gtf} to assist in performing graph-based imputation from
	partially observed data over the nodes.
	In Section
	\ref{sec:sparsegtf}, we investigate a modification of the proposed
	regularization scheme, where we add a pure $\ell_1$ penalty on $\beta$
	in \eqref{eq:gtf}, hence forming a sparse variant of GTF.  Other
	potentially interesting penalty extensions include: mixing graph
	difference penalties of various orders, and tying together several
	denoising tasks with a group penalty.
	Extensions such as these are easily built, recall, as a result of the
	analysis framework used by the GTF program, wherein the estimate
	defined through direct regularization via an analyzing operator, the
	$\ell_1$-based graph difference penalty
	\smash{$\|\op^{(k+1)}\beta\|_1$}.

	\section{Computation}
	\label{sec:computation}
	
	Graph trend filtering is defined by a convex optimization problem
	\eqref{eq:gtf}.  In principle this means that, at least for small
	or moderately sized problems, its solutions can be reliably computed 
	using a variety of standard algorithms. In order to handle
	larger scale problems, we describe two specialized algorithms that
	improve on generic procedures by taking advantage of the
	structure of \smash{$\op^{(k+1)}$}.
	
	\subsection{A Fast ADMM Algorithm}
	
	We reparametrize \eqref{eq:gtf} by introducing auxiliary variables, so
	that we can apply ADMM. For even $k$, we use a special transformation
	that is critical for fast computation (following \citet{fasttf} in
	univariate trend filtering); for odd $k$, this is not possible.  The
	reparametrizations for even and odd $k$ are
	\begin{align*}
	&\min_{\beta,z \in \R^n} \, \half \|y-\beta\|_2^2 + \lambda \|Dz\|_1
	\;\,\st\;\, z=L^{\frac{k}{2}} x, \\
	&\min_{\beta, z\in\R^n} \, \half \|y-\beta\|_2^2 + \lambda \|z\|_1
	\;\,\st\;\, z=L^{\frac{k+1}{2}} x,
	\end{align*}
	respectively. Recall that $D$ is the oriented incidence matrix and $L$
	is the graph Laplacian. The augmented Lagrangian is
	\begin{equation*}
	\half \|y-\beta\|^2_2 + \lambda\|S z\|_1 + \frac{\rho}{2} \| z-L^q
	\beta + u\|_2^2 -\frac{\rho}{2}\|u\|_2^2,
	\end{equation*}
	where $S=D$ or $S=I$ depending on whether $k$ is even or odd, and
	likewise $q=k/2$ or $q=(k+1)/2$.  ADMM then proceeds by iteratively
	minimizing the augmented Lagrangian over $\beta$, minimizing over $z$,
	and performing a dual update over $u$.  The $\beta$ and $z$ updates
	are of the form, for some $b$,
	\begin{align}
	\label{eq:admm_linearsystem}
	&\beta \leftarrow (I+\rho L^{2q})^{-1} b, \\
	\label{eq:admm_prox}
	&z \leftarrow \argmin_{x \in \R^n} \, \half \|b-x\|_2^2 +
	\frac{\lambda}{\rho} \|S x\|_1,
	\end{align}
	The linear system in
	\eqref{eq:admm_linearsystem} is well-conditioned, sparse, and can
	be solved efficiently using the preconditioned conjugate gradient
	method. This involves only multiplication with Laplacian matrices.
	For a small enough choices of $\rho>0$ (the augmented Lagrangian 
	parameter), the system in \eqref{eq:admm_linearsystem} is
	diagonally dominant, special Laplacian/SDD solvers can be applied,
	which run in almost linear time
	\citep{spielmanTeng2004,koutis2011,kelner2013}.  
	
	For $S=I$, the update in \eqref{eq:admm_prox} is simply given by
	soft-thresholding, and for $S=D$, it is given by
	graph TV denoising, i.e., given by solving a graph fused lasso
	problem.  Note that this subproblem has the exact structure of
	the graph trend filtering problem \eqref{eq:gtf} with $k=0$.  A direct
	approach for graph TV denoising is available based on parametric
	max-flow \citep{chambolle2009total}, and this algorithm is empirically
	much faster than its worst-case complexity
	\citep{boykov2004experimental}.  In the special case that the
	underlying graph is a grid, a promising alternative method employs
	proximal stacking techniques \citep{barberosra2014}.
	
	\subsection{A Fast Newton Method}
	
	As an alternative to ADMM, a projected Newton-type method
	\citep{bertsekas1982projected,barbero2011fast} can be used to solve
	\eqref{eq:gtf} via its dual problem:
	\begin{equation*}
	\hat{v} = \argmin_{v \in \R^r} \, \|y-(\op^{(k+1)})^\top v\|_2^2
	\;\,  \st \;\, \|v\|_\infty \leq \lambda.
	\end{equation*}
	The solution of \eqref{eq:gtf} is then given by
	\smash{$\hbeta = y-(\op^{(k+1)})^\top \hat{v}$}. (For univariate trend
	filtering, \citet{l1tf} adopt a similar strategy, but instead use an
	interior point method.)  The projected Newton method performs
	updates using a reduced Hessian, so abbreviating
	\smash{$\op=\op^{(k+1)}$}, each iteration boils down to
	\begin{equation}
	\label{eq:pn_linearsystem}
	v \leftarrow a+(\op_I^\top)^\dag b,
	\end{equation}
	for some $a,b$ and set of indices $I$.  The linear system in
	\eqref{eq:pn_linearsystem} is always sparse, but conditioning becomes
	an issue as $k$ grows (note that the same problem does not occur in  
	\eqref{eq:admm_linearsystem} because of the addition of the identity 
	matrix $I$).   We have found empirically that a preconditioned
	conjugate gradient method works quite well for
	\eqref{eq:pn_linearsystem} for $k=1$, but struggles for larger $k$.
	
	\subsection{Computation Summary}
	
	In our experience, the following algorithms work well for the various
	order $k$ of graph trend filtering.  We remark that orders $k=0,1,2$
	are of most practical interest (and solutions of polynomial order $k
	\geq 3$ are less likely to be sought in practice).\footnote{Loosely
		speaking, each  order $k=0,1,2$ provides solutions that
		exhibit a different class of structure: $k=0$ gives piecewise
		constant solutions, $k=1$ gives piecewise linear, and $k=2$ gives
		piecewise smooth.  All orders $k \geq 3$ continue to give piecewise
		smooth fits, with less and less transparent differences (the
		practical differences between piecewise quadratic versus piecewise 
		linear fits is greater than piecewise cubic versus piecewise
		quadratic, etc.). Since the conditioning of the graph
		trend filtering operator \smash{$\op^{(k+1)}$} worsens as $k$
		increases, which makes computation more difficult, it makes most
		practical sense to simply choose $k=2$ whenever a piecewise smooth
		fit is desired.} 
	
	\vspace{1mm}
	\begin{center}
		{\centering
			\begin{tabular}{l|l}
				Order & Algorithm \\
				\hline
				$k=0$ & Parametric max-flow \\
				$k=1$ & Projected Newton method \\
				$k=2,4,\ldots$ &ADMM with parametric max-flow \\ 
				$k=3,5,\ldots$ &ADMM with soft-thresholding
			\end{tabular}
		}
	\end{center}
	\vspace{1mm}
	
	Figure \ref{fig:convergence} compares performances of the
	described algorithms on a moderately large simulated example, using a
	2d grid graph.  We see that when $k=1$, the projected Newton
	method converges faster than ADMM (superlinear versus at best linear 
	convergence).  When $k=2$, the story is reversed,
	as the projected Newton iterations quickly become stagnant, and the
	ADMM enjoys better convergence.
	
	\begin{figure}[htb]
		\centering
		\resizebox{\columnwidth}{!}{
			\begin{tabular}{cc}
				GTF with $k=1$ & GTF with $k=2$ \\
				\includegraphics[width=0.475\textwidth]{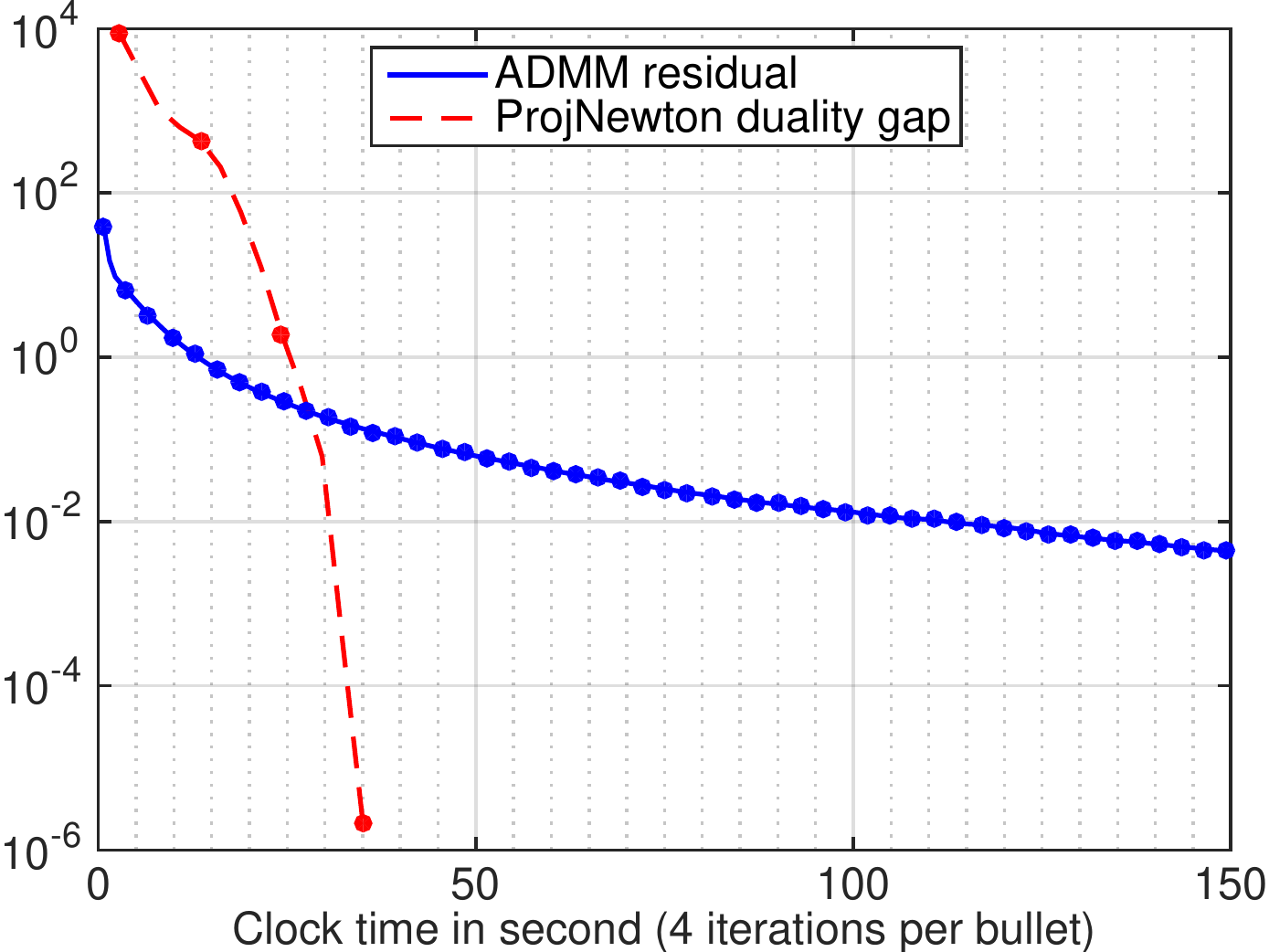} &
				\includegraphics[width=0.475\textwidth]{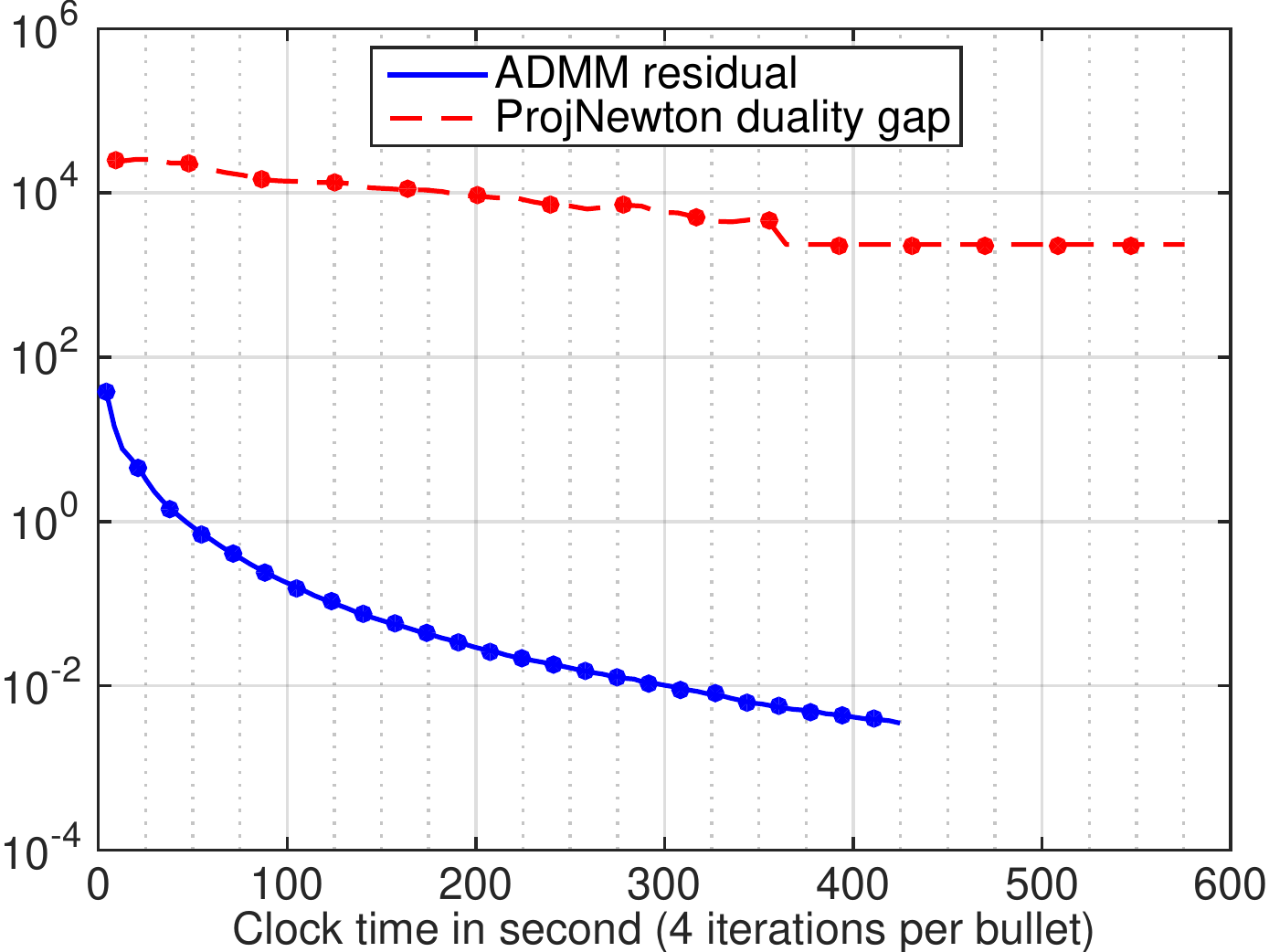}
			\end{tabular}
		}
		\caption{Convergence plots for projected Newton method and ADMM
			for solving GTF with $k=1$ and $k=2$. The algorithms are all run on
			a 2d grid graph (an $512\times 512$ image) with 262,144 nodes and
			523,264 edges. For projected Newton, we plot the duality gap across
			iterations; for ADMM, we plot the average of the primal
			and dual residuals (which also serves as a
			valid suboptimality bound in the ADMM framework).} 
		\label{fig:convergence}
	\end{figure}

	\section{Examples}
	\label{sec:examples}
	
	In this section, we present a variety of examples of running graph
	trend filtering on real graphs. 
	
	\subsection{Trend Filtering over the Facebook Graph}
	\label{sec:facebook}
	
	In the Introduction, we examined the denoising power of graph trend
	filtering in a spatial setting.  Here we examine the behavior of graph
	trend filtering on a nonplanar graph: the
	Facebook graph from the Stanford Network Analysis Project
	(\url{http://snap.stanford.edu}). This is composed of
	4039 nodes representing Facebook users, and 88,234
	edges representing friendships, collected from real survey
	participants; the graph has one connected component, but the observed
	degree sequence is very mixed, ranging from 1 to 1045 (refer to
	\citet{leskovec2012} for more details).
	
	We generated synthetic
	measurements over the Facebook nodes (users) based on three
	different ground truth models, so that we can precisely evaluate and
	compare the estimation accuracy of GTF, Laplacian smoothing, and
	wavelet smoothing.  For the latter, we again used the spanning tree
	wavelet design of \citet{graphwave}, because it performed among the
	best of wavelets designs in all data settings considered here.
	Results from other wavelet designs are presented in the Appendix. 
	The three ground truth models represent very
	different scenarios for the underlying signal $x$, each one favorable
	to different estimation methods. These are:
	
	\begin{enumerate*}
		\item {\bf Dense Poisson equation:} we solved the Poisson equation
		$Lx=b$ for $x$, where $b$ is arbitrary and dense (its entries were
		i.i.d. normal draws).
		
		\vspace{1mm}
		\item {\bf Sparse Poisson equation:} we solved the Poisson equation
		$Lx=b$ for $x$, where $b$ is sparse and has 30 nonzero entries
		(again i.i.d.\ normal draws).
		
		\vspace{1mm}
		\item {\bf Inhomogeneous random walk:} we ran a set of decaying
		random walks at different starter nodes in the graph, and recorded
		in $x$ the total number of visits at each
		node. Specifically, we chose 10 nodes as starter nodes, and
		assigned each starter node a decay probability uniformly at random
		between 0 and 1 (this is the probability that the walk terminates at
		each step instead of travelling to a neighboring node).  At each
		starter node, we also sent out a varying number of random walks,
		chosen uniformly between 0 and 1000.
	\end{enumerate*}
	
	In each case, the synthetic measurements were formed by adding
	noise to $x$. We note that model 1 is designed to be favorable for
	Laplace smoothing; model 2 is designed to be favorable for GTF; and
	in the inhomogeneity in model 3 is designed to be challenging for
	Laplacian smoothing, and favorable for the more adaptive GTF and
	wavelet methods.
	
	\begin{figure}[htbp]
		\centering
		\resizebox{\columnwidth}{!}{
			\begin{tabular}{cc}
				Dense Poisson equation & Sparse Poisson equation \\
				\includegraphics[width=0.475\textwidth]{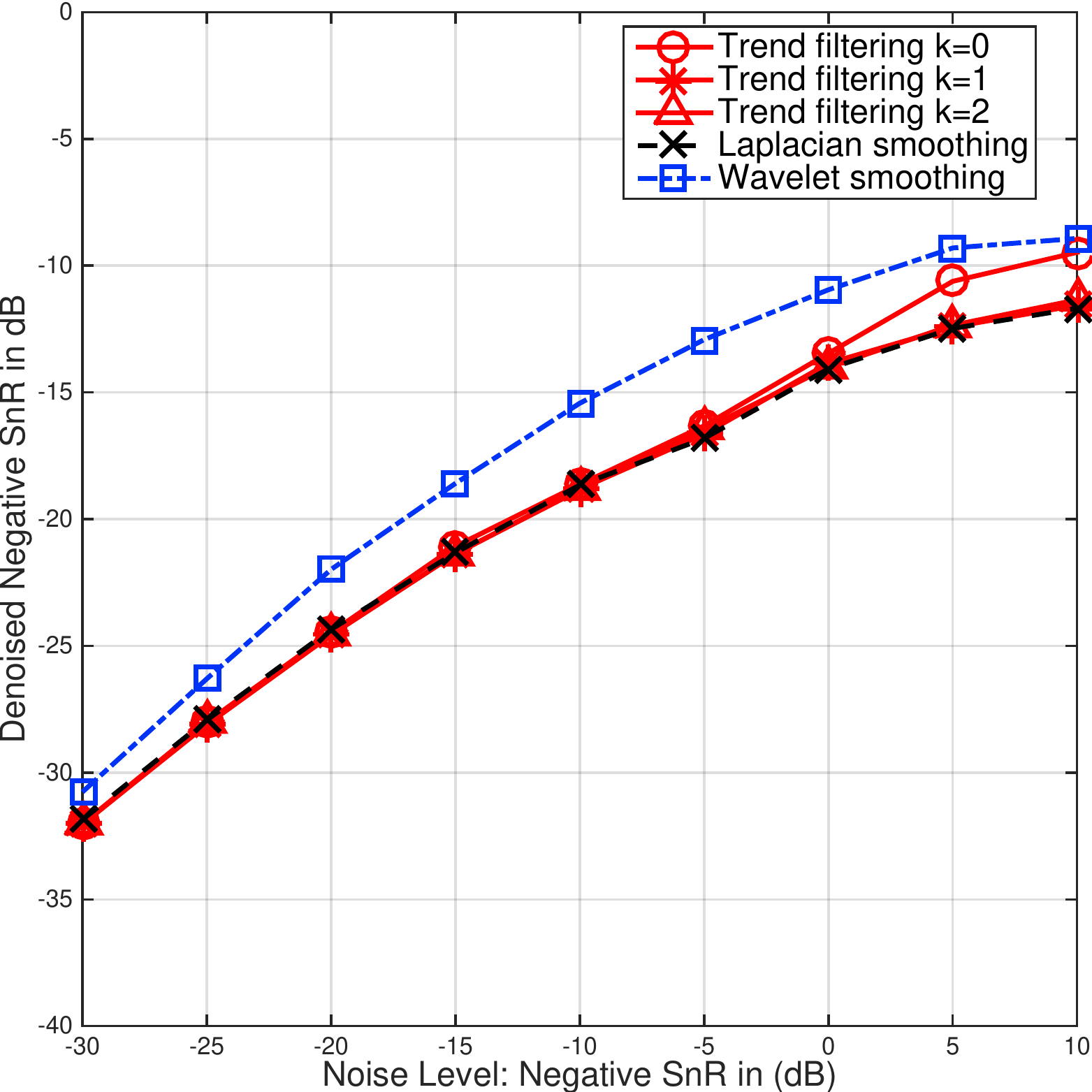} &
				\includegraphics[width=0.475\textwidth]{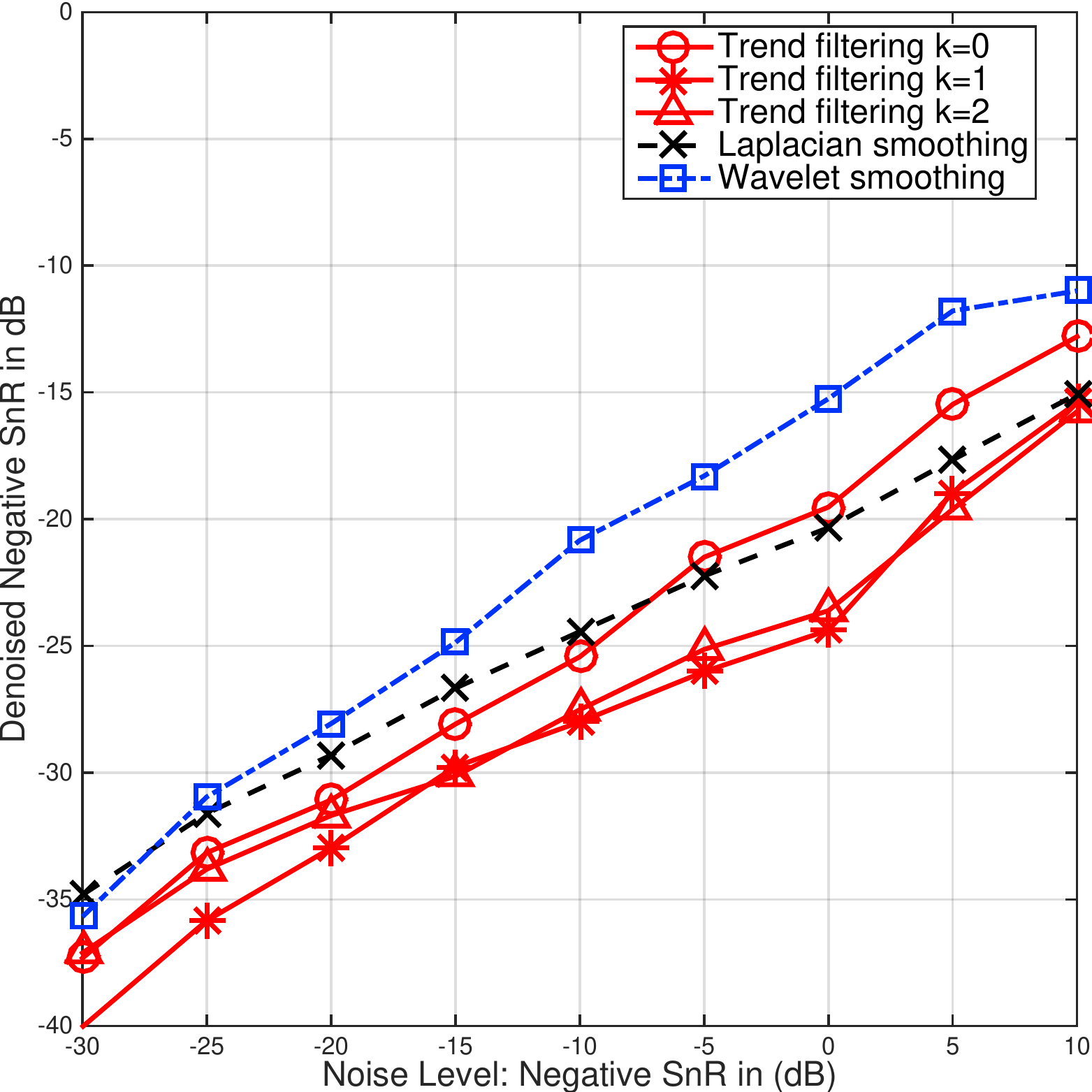}
			\end{tabular} 
		}
		
		\smallskip\smallskip 
		Inhomogeneous random walk \\
		\includegraphics[width=0.475\textwidth]{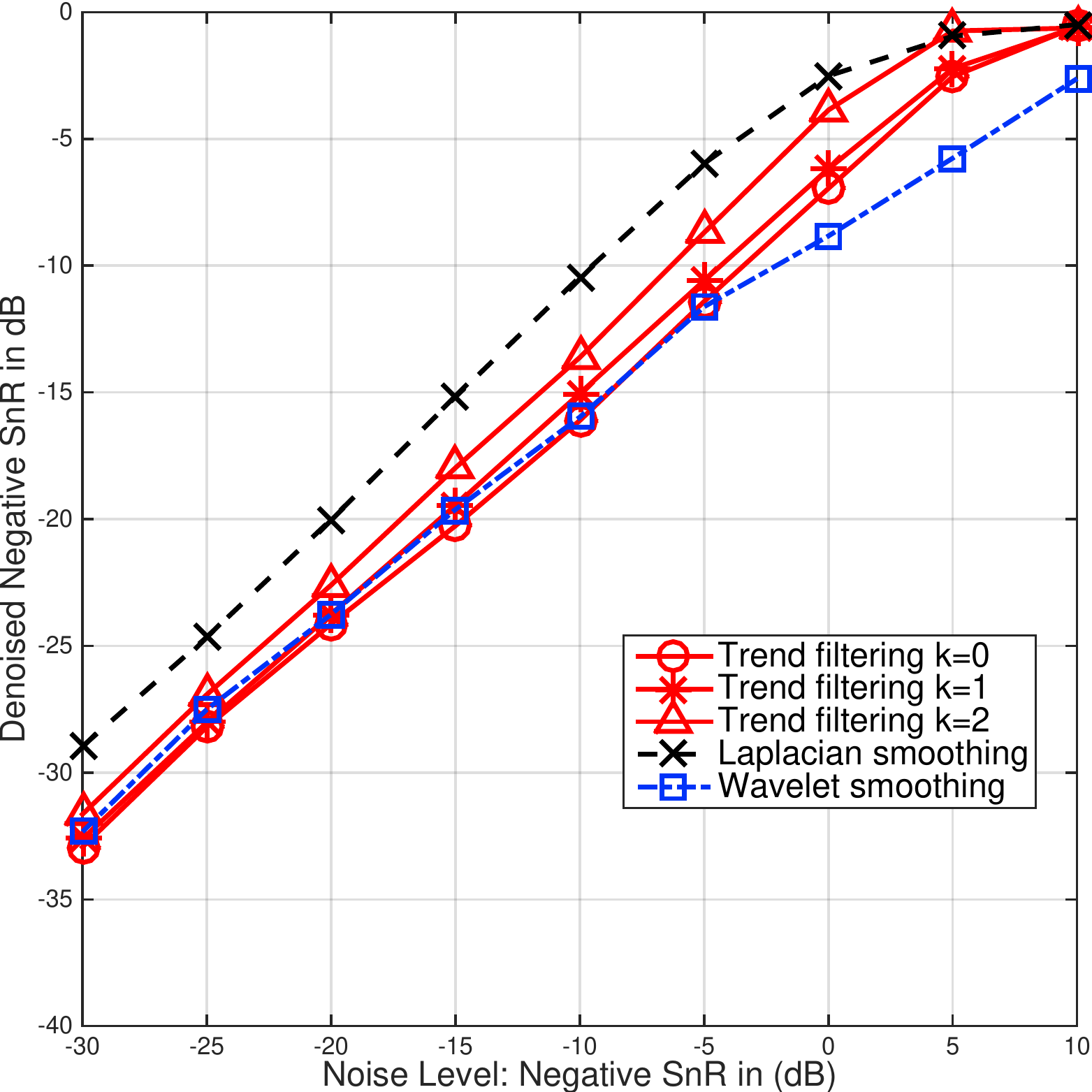}
		
		\caption{Performance of GTF and others for three generative models
			on the Facebook graph. The x-axis
			shows the negative SnR: $10 \log_{10} (n\sigma^2/\|x\|_2^2)$, where
			$n=4039$, $x$ is the underlying signal, and $\sigma^2$ is the noise
			variance. Hence the noise level is increasing from left to right.
			The y-axis shows the denoised negative SnR: $10 \log_{10}
			(\MSE/\|x\|_2^2)$, where MSE denotes mean squared error, so the
			achieved MSE is increasing from bottom to top.}
		\label{fig:facebook}
	\end{figure}
	
	Figure \ref{fig:facebook} shows the performance of the
	three estimation methods, over a wide range of noise levels in the
	synthetic measurements; performance here is measured by the best
	achieved mean squared error, allowing each method to be tuned
	optimally at each noise level.  The summary: GTF estimates
	are (expectedly) superior when the Laplacian-based sparsity pattern
	is in effect (model 2), but are nonetheless highly competitive in both
	other settings---the dense case, in which Laplacian smoothing thrives,
	and the inhomogeneous random walk case, in which wavelets thrive.
	
	\subsection{Graph-Based Transductive Learning over UCI Data}
	\label{sec:madgtf}
	
	Graph trend filtering can used for graph-based transductive learning, as
	motivated by the work of
	\citet{talukdar2009new,talukdar2010experiments}, who rely on Laplacian
	regularization.  Consider a semi-supervised learning setting, where we
	are given only a small number of seed labels over nodes of a graph,
	and the goal is to impute the labels on the remaining nodes.
	Write $O \subseteq \{1,\ldots n\}$ for the set of observed nodes, and
	assume that each observed label falls into $\{1,\ldots K\}$.  Then
	we can define the modified absorption problem under graph trend
	filtering regularization (MAD-GTF) by
	\begin{equation}
	\label{eq:madgtf}
	\hat{B} = \argmin_{B \in \R^{n\times K}} \;
	\sum_{j=1}^K \sum_{i \in O} (Y_{ij}-B_{ij})^2  +
	\lambda \sum_{j=1}^K \|\op^{(k+1)} B_j\|_1 +
	\epsilon \sum_{j=1}^K \|R_j-B_j\|_2^2.
	\end{equation}
	The matrix $Y\in\R^{n\times K}$ is an indicator matrix: each observed
	row $i \in O$ is described by $Y_{ij}=1$ if class $j$ is observed and
	$Y_{ij}=0$ otherwise.  The matrix $B \in \R^{n\times K}$ contains
	fitted probabilities, with $B_{ij}$ giving the probability that node
	$i$ is of class $j$.  We write $B_j$ for its $j$th column, and hence
	the middle term in the above criterion encourages each set of class
	probabilities to behave smoothly over the graph.  The last term in the
	above criterion ties the fitted probabilities to some given prior
	weights $R \in \R^{n\times K}$.  In principle
	$\epsilon$ could act as a second tuning parameter, but for simplicity
	we take $\epsilon$ to be small and fixed, with any $\epsilon>0$
	guaranteeing that the criterion in \eqref{eq:madgtf} is strictly
	convex, and thus has a unique solution \smash{$\hat{B}$}. The entries
	of \smash{$\hat{B}$} need not be probabilites in any strict sense, but
	we can still interpret them as relative probabilities, and imputation
	can be performed by assigning each unobserved node $i \notin O$ a
	class label $j$ such that $\hat{B}_{ij}$ is largest.
	
	\begin{figure}[!htb]
		\centering
		\includegraphics[width=0.7\textwidth]{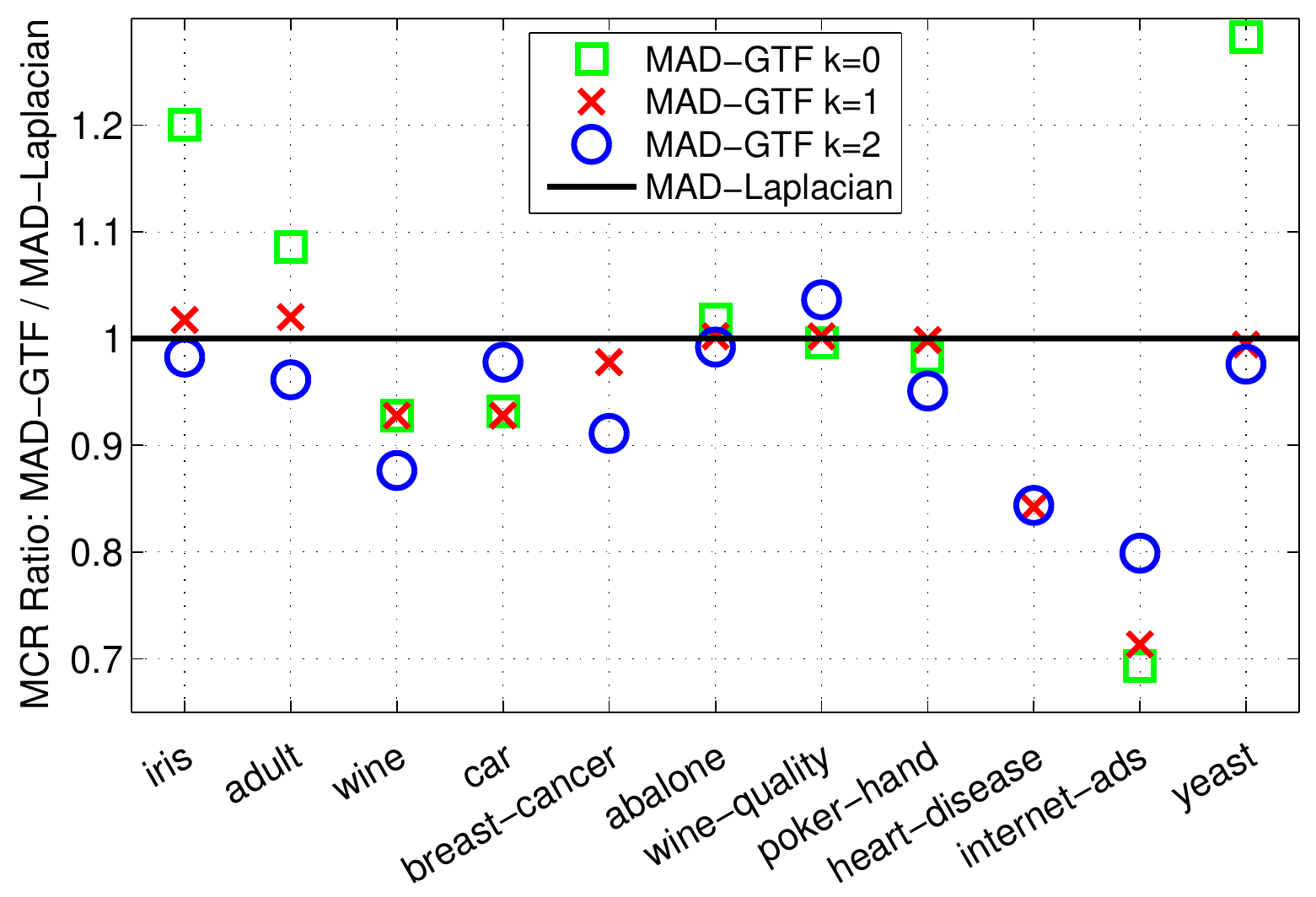} \\
		\captionof{figure}{Ratio of the misclassification rate of MAD-GTF to
			MAD-Laplacian, for graph-based imputation, on the 11 most popular
			UCI classification data sets.}
		\label{fig:uci}
	\end{figure}
	
	\begin{table}[!h]
		\centering
		\resizebox{\columnwidth}{!}{
			\begin{tabular}{|c|c|c|c|c|c|c|c|c|c|c|c|}
				\hline
				& iris & adult & wine & car & breast & abalone &
				wine-qual. & poker & heart & ads & yeast \\
				\hline
				\# of classes& 3 & 2 & 3 & 4 & 2\ & 29 &
				6 & 10 & 2 & 2 & 10
				\\
				\# of samples & 150 &  32,561 & 178 & 1,728 & 569 & 4,177 & 1,599 
				& 3,000 & 303 & 3,279 & 1,484 \\ 
				\hline\hline
				Laplacian & 0.085 & 0.270 & 0.060 & 0.316 & 0.064 & 0.872 &
				0.712 &0.814 & 0.208 & 0.306 & 0.566 \\\hline
				GTF, $k=0$ & 0.102 & 0.293 &0.055 & \textcolor{green}{0.294} &
				\textcolor{red}{0.500} 
				&\textcolor{red}{ 0.888} & 0.709 & 0.801 & \textcolor{red}{0.472} &
				\textcolor{green}{0.212}&\textcolor{red}{0.726} \\
				p-value & 0.254 & 0.648 &0.406 & \textcolor{green}{0.091} &
				\textcolor{red}{0.000} & 
				\textcolor{red}{0.090} & 0.953 & 0.732 & \textcolor{red}{0.000} &
				\textcolor{green}{0.006}&\textcolor{red}{0.000} \\\hline
				GTF, $k=1$  & 0.087 & 0.275 & \textcolor{green}{0.055} &
				\textcolor{green}{0.293} 
				& 0.063 & 0.874 & 0.713 & 0.813 & 0.175 & \textcolor{green}{0.218}
				& 0.563 \\ 
				p-value & 0.443 & 0.413 &\textcolor{green}{0.025} &
				\textcolor{green}{0.012} & 0.498 
				& 0.699 & 0.920 & 0.801 & 0.134 & \textcolor{green}{0.054}&0.636
				\\\hline 
				GTF, $k=2$ & 0.084 & 0.259 & \textcolor{green}{0.052} &0.309 & 
				\textcolor{green}{0.059} & 0.865 & 0.738 & 0.774 & 0.175 & 0.244 & 
				\textcolor{green}{0.552} \\  
				p-value & 0.798 & 0.482 &\textcolor{green}{0.024} & 0.523 & 
				\textcolor{green}{0.073} 
				& 0.144 & 0.479 & 0.138 & 0.301 & 0.212&\textcolor{green}{0.100}
				\\ 
				\hline
			\end{tabular}
		}
		\caption{Misclassification rates of MAD-Laplacian and MAD-GTF
			for imputation in the UCI data sets.  We also compute p-values over 
			the 10 repetitions for each data set (10 draws of nodes to serve
			as seed labels) via paired t-tests. Cases where MAD-GTF achieves
			significantly better misclassification rate, at the 0.1 level,
			are highlighted in \textcolor{green}{green}; cases where MAD-GTF
			achieves a significantly worse miclassification rate, at the 0.1
			level, are highlighted in \textcolor{red}{red}.}
		\label{tab:uci}
	\end{table}

	Our specification of MAD-GTF only deviates from the
	MAD proposal of \citet{talukdar2009new} in that these authors used the
	Laplacian regularization term $\sum_{j=1}^K B_j^\top L B_j$, in place
	of $\ell_1$-based graph difference regularizer in \eqref{eq:madgtf}. 
	If the underlying class probabilities are thought to have
	heterogeneous smoothness over the graph, then 
	replacing the Laplacian regularizer with the GTF-designed one might
	lead to better performance.
	As a broad comparison of the two methods, we ran them on the
	11 most popular classification data sets from the UCI Machine Learning
	repository (\url{http://archive.ics.uci.edu/ml/}).\footnote{We
		used all data sets here, except the ``forest-fires'' data set, which
		is a regression problem. Also, we zero-filled the missing data in
		``internet-ads'' data set and used a random one third of the data in
		the ``poker'' data set.}   For each data set, we constructed a
	$5$-nearest-neighbor graph based on the Euclidean distance between
	provided features, and randomly selected 5 seeds per class to
	serve as the observed class labels. Then we set $\epsilon=0.01$, used
	prior weights $R_{ij}=1/K$ for all $i$ and $j$, and chose the tuning
	parameter $\lambda$ over a wide grid of values to represent the best
	achievable performance by each method, on each experiment.
	Figure \ref{fig:uci} and Table~\ref{tab:uci} summarize the
	misclassification rates from imputation using MAD-Laplacian and
	MAD-GTF, averaged over 10 repetitions of the randomly selected seed
	labels.  We see that MAD-GTF with $k=0$ (basically a graph
	partition akin to MRF-based graph cut, using an Ising model)
	does not seem to work as well as the smoother alternatives.
	Importantly, MAD-GTF with $k=1$ and $k=2$ both perform at least as
	well, and sometimes better, than MAD-Laplacian on each one of 
	the UCI data sets.  Recall that these data sets were selected entirely
	based on their popularity,
	and not at all on the belief that they might represent favorable
	scenarios for GTF (i.e., not on the prospect that they might
	exhibit some heterogeneity in the distribution of class labels over
	their respective graphs).  Therefore, the fact that MAD-GTF nonetheless
	performs competitively in such a broad range of experiments is 
	convincing evidence for the utility of the GTF regularizer.
	
	\subsection{Event Detection with NYC Taxi Trips Data}
	\label{sec:sparsegtf}
	
	We illustrate a sparse variant of our proposed regularizers, given by
	adding a pure $\ell_1$ penalty to the coefficients in \eqref{eq:gtf},
	as in
	\begin{equation}
	\label{eq:sgtf}
	\hbeta = \argmin_{\beta \in \R^n} \,
	\half\|y-\beta\|_2^2 + \lambda_1 \|\op^{(k+1)}\beta\|_1 +
	\lambda_2 \|\beta\|_1.
	\end{equation}
	We call this {\it sparse graph trend filtering}, now with
	two tuning parameters $\lambda_1,\lambda_2 \geq 0$. Under the
	proper tuning, the sparse GTF estimate will be zero at many nodes in
	the graph, and will otherwise deviate smoothly from zero.   This can
	be useful in situations where the observed signal represents a
	difference between two smooth processes that are mostly
	similar, but exhibit (perhaps significant) differences over a few
	regions of the graph.  Here we
	apply it to the problem of detecting events based on
	abnormalities in the number of taxi trips at different locations of
	New York city.  This data set was kindly provided by authors of
	\citet{ferreira2014Vtopological}, who obtained the data from NYC Taxi
	\& Limosine Commission.\footnote{These authors also considered event
		detection, but their topological definition of an ``event'' is very
		different from what we considered here, and hence our results not
		directly comparable.}
	Specifically, we consider the graph to be the road
	network of Manhattan, which contains 3874 nodes (junctions) and 7070
	edges (sections of roads that connect two junctions).
	For measurements over the nodes, we used the number of taxi
	pickups and dropoffs over a particular time period of interest:
	12:00--2:00 pm on June 26, 2011, corresponding to the Gay Pride
	parade.  As pickups and
	dropoffs do not generically occur at road junctions, we used
	interpolation to form counts over the graph nodes.  A baseline
	seasonal average was calculated by considering data from the same time
	block 12:00--2:00 pm on the same day of the week across the nearest eight
	weeks.  Thus the measurements $y$ were then taken to be the difference
	between the counts observed during the Gay Pride parade and the
	seasonal averages.
	
	Note that the nonzero node estimates from sparse GTF applied to $y$,
	after proper tuning, mark events of interest, because
	they convey substantial differences between the observed and
	expected taxi counts. According to descriptions in the news,
	we know that the Gay Pride parade was a giant march down at noon from
	36th St.\ and Fifth Ave.\ all the way to Christopher St.\ in
	Greenwich Village, and traffic was blocked over the entire route for
	two hours (meaning no pickups and dropoffs could occur).   We therefore
	hand-labeled this route as a crude ``ground truth'' for the event of
	interest, illustrated in the left panel of Figure~\ref{fig:taxi}.
	
	
	\begin{figure}[htbp]
		\centering
		\begin{tabular}{cc}
			True parade route & Unfiltered signal \\
			\includegraphics[width=0.39\textwidth]{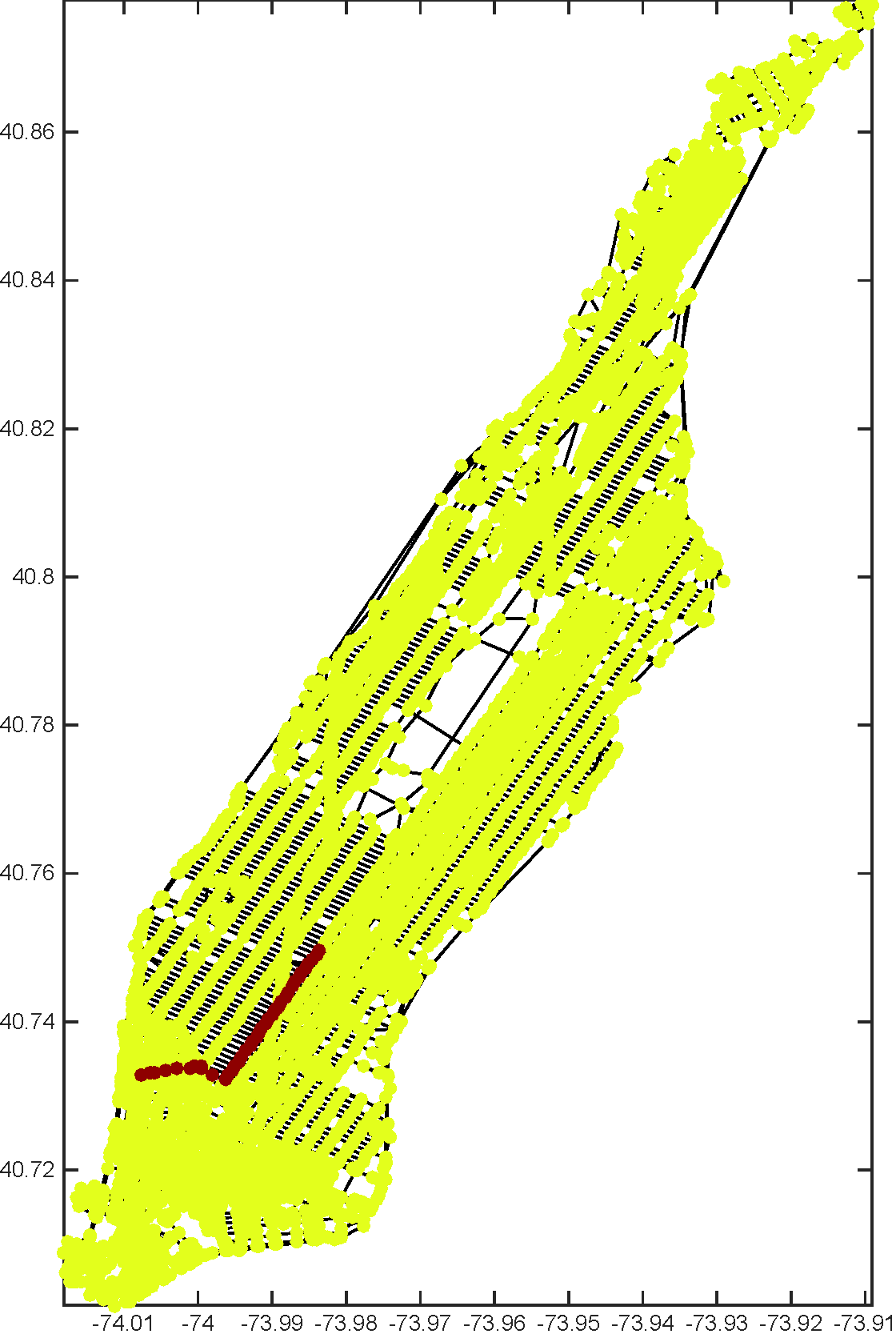} &
			\includegraphics[width=0.39\textwidth]{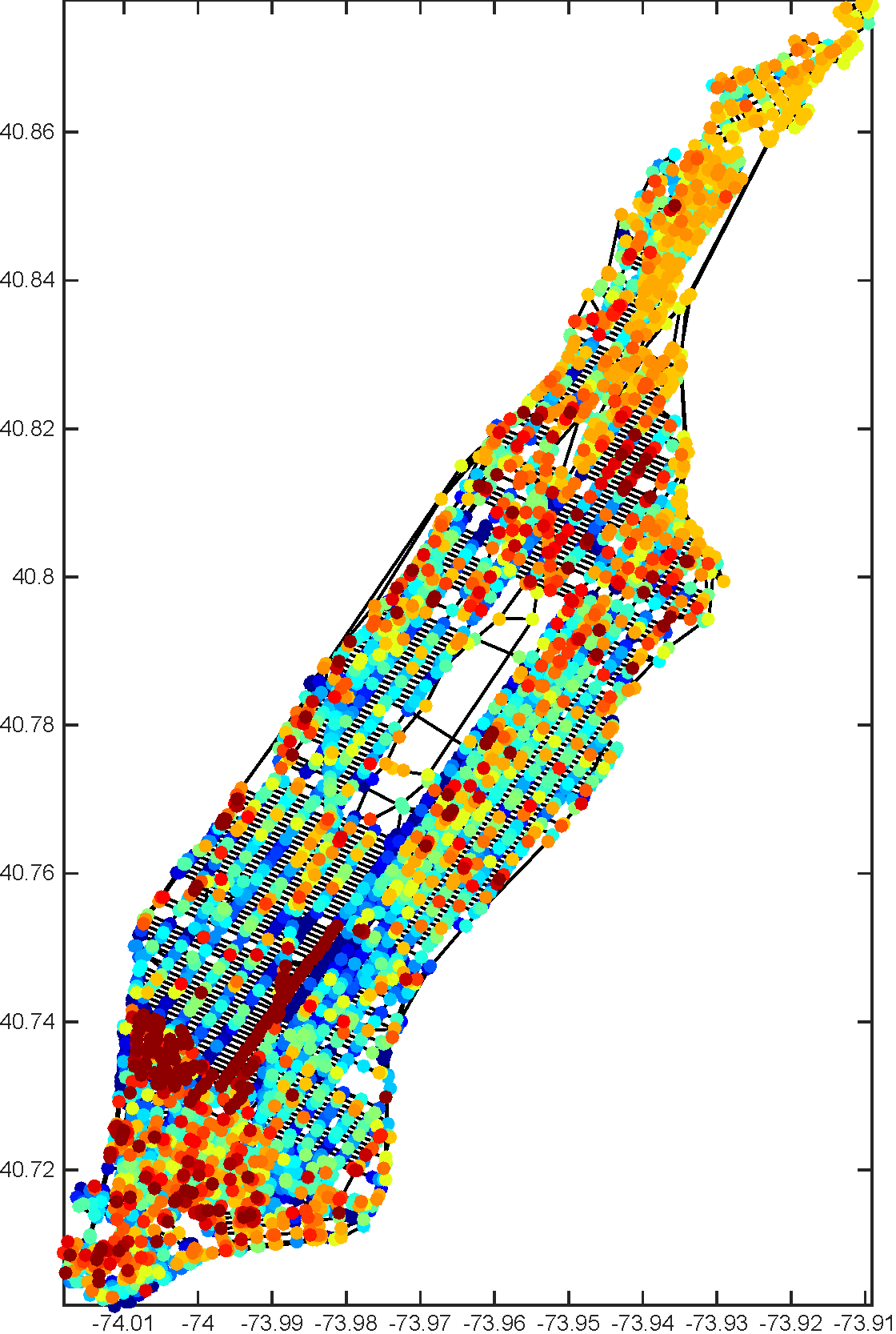}
			\smallskip \\
			Sparse trend filtering & Sparse Laplacian smoothing \\
			\includegraphics[width=0.39\textwidth]{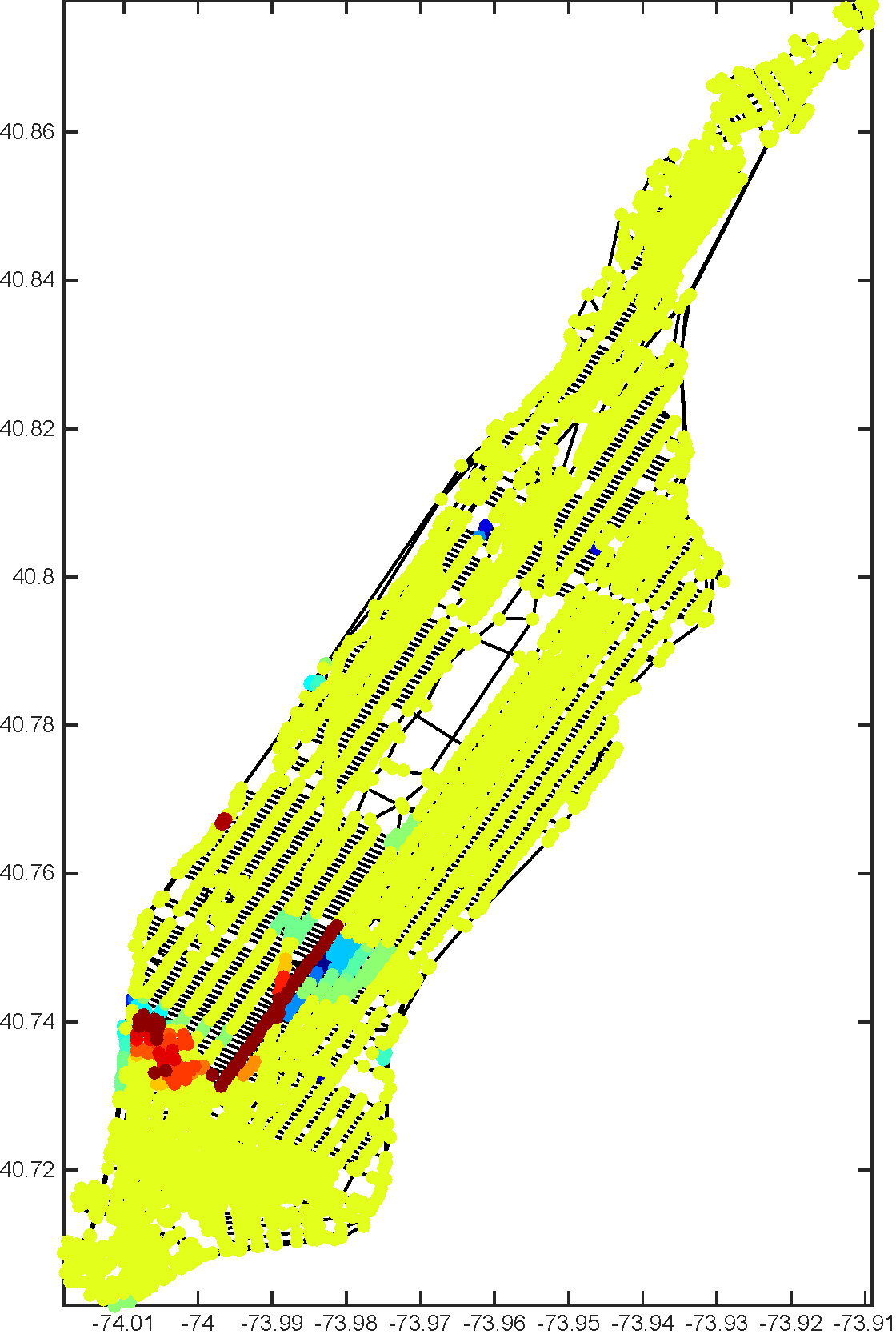} &
			\includegraphics[width=0.39\textwidth]{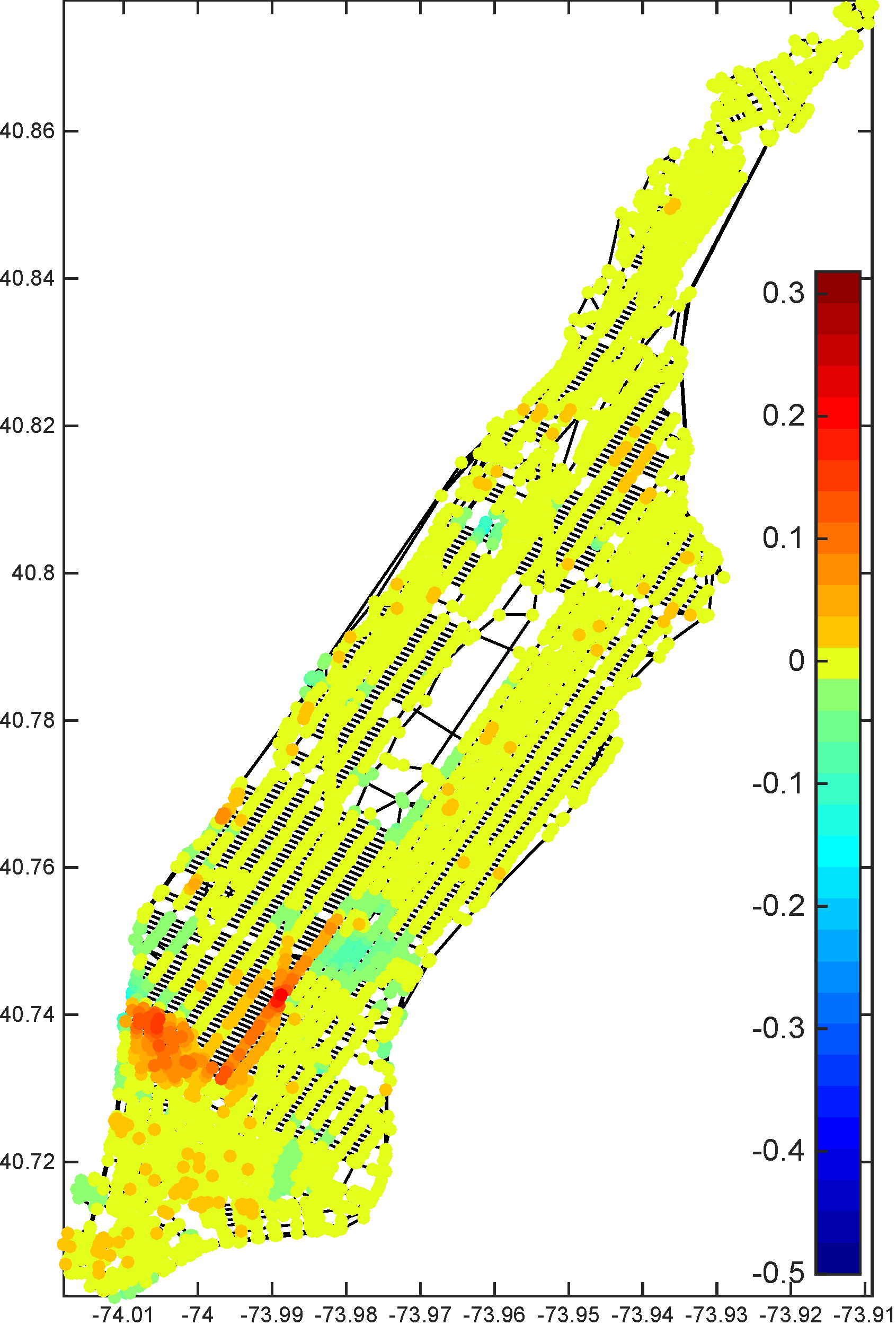}
		\end{tabular}
		\caption{Comparison of sparse GTF and sparse Laplacian smoothing. 
			We can see qualitatively that sparse GTF
			delivers better event detection with fewer false positives
			(zoomed-in, the sparse Laplacian plot shows a scattering of many
			non-zero colors). }
		\label{fig:taxi}
	\end{figure}
	
	In the bottom two panels of Figure~\ref{fig:taxi}, we compare
	sparse GTF with $k=0$ (i.e., the sparse graph fused lasso) and a
	sparse variant of Laplacian smoothing, obtained by replacing the first 
	regularization term in \eqref{eq:sgtf} by $\beta^\top L \beta$. 
	For a qualitative visual comparison, the smoothing parameter
	$\lambda_1$ was chosen so that both methods have 200 degrees of
	freedom (without any sparsity imposed).  The sparsity parameter was
	then set as $\lambda_2=0.2$. Similar to what we have seen already,
	GTF is able to better localize its estimates around strong
	inhomogenous spikes in the measurements, and is
	able to better capture the event of interest.  The result of sparse
	Laplacian smoothing is far from localized around the ground truth
	event, and displays many
	nonzero node estimates throughout distant regions of the graph.  If we
	were to decrease its flexibility (increase the smoothing parameter
	$\lambda_1$ in its problem formulation), then the sparse Laplacian
	output would display more smoothness over the graph, but the node
	estimates around the ground truth region would also be grossly
	shrunken.

	\section{Estimation Error Bounds}
	\label{sec:theory}
	
	In this section, we assume that
	$y \sim \cN(\beta_0,\sigma^2 I)$,
	and study asymptotic error rates for graph trend filtering. (The
	assumption of a normal error model could be relaxed, but is used for
	simplicity).  Our analysis actually focuses more broadly on the
	generalized lasso problem
	\begin{equation}
	\label{eq:genlasso}
	\hbeta = \argmin_{\beta \in \R^n} \,
	\half\|y-\beta\|_2^2 + \lambda \|\op \beta\|_1,
	\end{equation}
	where $\op \in \R^{r\times n}$ is an arbitrary linear operator, and
	$r$ denotes its number of rows.  Throughout, we specialize the derived
	results to the graph difference operator $\op=\op^{(k+1)}$, to obtain
	concrete statements about GTF over particular graphs. All proofs are
	deferred to the Appendix.
	
	\subsection{Basic Error Bounds}
	
	Using similar arguments to the basic inequality for the lasso
	\citep{statshd}, we have the following preliminary bound.
	
	\begin{theorem}
		\label{thm:basic}
		Let $M$ denote the maximum $\ell_2$ norm of the columns of
		$\op^\dag$. Then for a tuning parameter value
		$\lambda=\Theta(M\sqrt{\log{r}})$, the generalized lasso estimate
		\smash{$\hbeta$} in \eqref{eq:genlasso} has average squared error 
		\begin{equation*}
		\frac{\|\hbeta-\beta_0\|_2^2}{n} =
		O_\P \left( \frac{\nuli(\op)}{n} + \frac{M\sqrt{\log{r}}}{n} \cdot
		\|\op \beta_0\|_1 \right).
		\end{equation*}
	\end{theorem}
	
	Recall that $\nuli(\op)$ denotes the dimension of the null space of
	$\op$. For the GTF operator \smash{$\op^{(k+1)}$} of any
	order $k$, note that \smash{$\nuli(\op^{(k+1)})$} is the number of
	connected components in the underlying graph.
	
	When both $\|\op \beta_0\|_1=O(1)$ and $\nuli(\op)=O(1)$,
	Theorem~\ref{thm:basic}
	says that the estimate \smash{$\hbeta$} converges in average squared
	error at the rate $M\sqrt{\log{r}}/n$, in probability.
	This theorem is quite general, as it applies to any
	linear operator $\op$, and one might therefore think that it cannot 
	yield fast rates. Still, as we show next, it does imply consistency
	for graph trend filtering in certain cases.
	
	\begin{corollary}
		\label{cor:basic}
		Consider the trend filtering estimator \smash{$\hbeta$} of order $k$,
		and the choice of the tuning parameter $\lambda$ as in Theorem
		\ref{thm:basic}.  Then:
		\begin{enumerate*}
			\item for univariate trend filtering (i.e., essentially GTF on a chain
			graph), 
			\begin{equation*}
			\frac{\|\hbeta-\beta_0\|_2^2}{n} = O_\P\left(\sqrt{\frac{\log{n}}{n}} \cdot
			n^k \|D^{(k+1)}\beta_0\|_1\right);
			\end{equation*}
			\item for GTF on an Erdos-Renyi random graph, with edge probability
			$p$, and expected degree $d=np \geq 1$,
			\begin{equation*}
			\frac{\|\hbeta-\beta_0\|_2^2}{n} = O_\P\left(
			\frac{\sqrt{\log(nd)}}{nd^{\frac{k+1}{2}}}
			\cdot \|\op^{(k+1)}\beta_0\|_1\right);
			\end{equation*}
			\item for GTF on a Ramanujan $d$-regular graph, and $d \geq 1$,
			\begin{equation*}
			\frac{\|\hbeta-\beta_0\|_2^2}{n} = O_\P\left(
			\frac{\sqrt{\log(nd)}}{nd^{\frac{k+1}{2}}}
			\cdot \|\op^{(k+1)}\beta_0\|_1\right).
			\end{equation*}
		\end{enumerate*}
	\end{corollary}

	Cases 2 and 3 of Corollary \ref{cor:basic} stem from
	the simple inequality $M \leq \|\op^\dag\|_2$, the largest singular
	value of $\op^\dag$.  When \smash{$\op=\op^{(k+1)}$}, the GTF operator
	of order $k+1$, we have 
	$$\|(\op^{(k+1)})^\dag\|_2 \leq 
	1/\lambda_{\min}(L)^{(k+1)/2},$$
	where $\lambda_{\min}(L)$ is
	the smallest nonzero eigenvalue of the Laplacian $L$ (also known as
	the Fiedler eigenvalue \citep{Fiedler73}). In general,
	$\lambda_{\min}(L)$ can be very small, leading to loose error bounds,
	but for the particular graphs in question, it is well-controlled.
	When \smash{$\|\op^{(k+1)}\beta_0\|_1$} is bounded, cases 2 and 3 of
	the corollary show that the average squared error of GTF 
	converges at the rate \smash{$\sqrt{\log(nd)}/(nd^{(k+1)/2})$}.
	As $k$ increases, this rate is stronger, but so is the
	assumption that \smash{$\|\op^{(k+1)}\beta_0\|_1$} is bounded.
	
	Case 1 in Corollary \ref{cor:basic} covers univariate trend filtering
	(which, recall, is basically the same as GTF over
	a chain graph; the only differences between the two are boundary terms
	in the construction of the difference operators).
	The result in case 1 is based on direct calculation of $M$, using
	specific facts that are known about the univariate difference
	operators. It is natural in the univariate setting to assume that 
	\smash{$n^k\|D^{(k+1)}\beta_0\|_1$} is bounded (this is the 
	scaling that would link $\beta_0$ to the evaluations of a piecewise
	polynomial function $f_0$ over $[0,1]$, with \smash{$\TV(f_0^{(k)})$}
	bounded). Under such an assumption, the above corollary yields a 
	convergence rate of $\sqrt{\log{n}/n}$ for univariate trend
	filtering, which is not tight.  A more refined analysis
	shows the univariate trend filtering
	estimator to converge at the minimax optimal rate
	\smash{$n^{-(2k+2)/(2k+3)}$}, proved in \citet{trendfilter} by
	using a connection between univariate trend filtering and locally
	adaptive regression splines, and relying on sharp entropy-based rates
	for locally adaptive regression splines from \citet{locadapt}.   
	We note that in a pure graph-centric setting, the latter 
	strategy is not generally applicable, as the notion of a
	spline function does not obviously extend to the nodes of an
	arbitrary graph structure.  
	
	In the next subsections, we develop more advanced strategies 
	for deriving fast GTF error rates, based on incoherence, and entropy. 
	These can provide substantial improvements over the basic error bound
	established in this subsection, but are only applicable to certain
	graph models.  Fortunately, this includes common graphs of interest,
	such as regular grids. To verify the sharpness of these
	alternative strategies, we will show that they can be used to recover 
	optimal rates of convergence for trend filtering in the 1d setting. 
	
	\subsection{Strong Error Bounds Based on Incoherence}
	\label{sec:incoherence}
	
	A key step in the proof of Theorem \ref{thm:basic} argues, roughly
	speaking, that 
	\begin{equation}
	\label{eq:holderbd}
	\epsilon^\top \op^\dag \op x
	\leq \|(\op^\dag)^\top \epsilon\|_\infty \|\op x\|_1
	= O_\P( M\sqrt{\log r} \|\op x\|_1),
	\end{equation}
	where $\epsilon \sim \cN(0,\sigma^2 I)$.
	The second bound holds by a standard
	result on maxima of Gaussians (recall that $M$ is largest
	$\ell_2$ norm of the columns of $\op^\dag$).  The first bound above
	uses Holder's inequality; note that this applies to any
	$\epsilon,\op$, i.e., it does not use any information about the
	distribution of $\epsilon$, or the properties of $\op$. The next lemma
	reveals a potential advantage that can be gained from replacing
	the bound \eqref{eq:holderbd}, stemming from Holder's
	inequality,  with a ``linearized'' bound.
	
	\begin{lemma}
		\label{lem:linearized}
		Denote $\epsilon \sim \cN(0,\sigma^2 I)$, and assume that
		\begin{equation}
		\label{eq:linearized}
		\max_{x \in \cS_\op(1)}
		\frac{\epsilon^\top x  - A}{\|x\|_2} = O_\P(B),
		\end{equation}
		where $\cS_\op(1) = \{x \in \row(\op) : \|\op x\|_1 \leq 1\}$.
		With $\lambda=\Theta(A)$, the generalized lasso estimate 
		\smash{$\hbeta$} satisfies
		\begin{equation*}
		\frac{\|\hbeta-\beta_0\|_2^2}{n} =
		O_\P \left( \frac{\nuli(\op)}{n} + \frac{B^2}{n} + \frac{A}{n}
		\cdot \|\op \beta_0\|_1 \right).
		\end{equation*}
	\end{lemma}
	
	The inequality in \eqref{eq:linearized} is referred to as a ``linearized''
	bound because it implies that for $x \in \cS_\op(1)$,
	\begin{equation*}
	\epsilon^\top  x = O_\P(A + B \|x\|_2),
	\end{equation*}
	and the right-hand side is a linear function of
	$\|x\|_2$.  Indeed, for $A= M\sqrt{2 \log r}$ and $B=0$, this
	encompasses the bound in \eqref{eq:holderbd} as a special case,
	and the result of Lemma \ref{lem:linearized} reduces
	to that of Theorem \ref{thm:basic}.  But the result in Lemma
	\ref{lem:linearized} can be much stronger, if $A,B$ can be
	adjusted so that $A$ is smaller than $M\sqrt{2 \log r}$, and $B$
	is also small.  Such an arrangement is possible for certain
	operators $\op$; e.g., it is possible under an incoherence-type
	assumption on $\op$.
	
	\begin{theorem}
		\label{thm:incoherence}
		Let $q=\rank(\op)$, and 
		let $\xi_1 \leq \ldots \leq \xi_q$ denote the singular values of
		$\op$, in increasing order. Also let $u_1,\ldots u_q$ be
		the corresponding left singular vectors. 
		Assume that these vectors are incoherent:  
		\begin{equation*}
		\|u_i\|_\infty \leq \mu/\sqrt{n}, \;\;\; i=1,\ldots q, 
		\end{equation*}
		for some constant $\mu \geq 1$.  For $i_0 \in \{1,\ldots q\}$, let 
		\begin{equation*}
		\lambda = \Theta \left(\mu
		\sqrt{\frac{\log{r}}{n} \sum_{i=i_0+1}^q \frac{1}{\xi_i^2}}
		\,\right).
		\end{equation*}
		Then the generalized lasso estimate \smash{$\hbeta$} has average
		squared error
		\begin{equation*}
		\frac{\|\hbeta-\beta_0\|_2^2}{n} = O_\P
		\left( \frac{\nuli(\op)}{n} + \frac{i_0}{n} + \frac{\mu}{n}
		\sqrt{\frac{\log{r}}{n} \sum_{i=i_0+1}^q \frac{1}{\xi_i^2}}
		\cdot \|\op\beta_0\|_1 \right).
		\end{equation*}
	\end{theorem}
	
	Theorem \ref{thm:incoherence} is proved by leveraging the linearized
	bound \eqref{eq:linearized}, which holds under the
	incoherence condition on the singular vectors of $\op$.  Compared
	to the basic result in Theorem \ref{thm:basic}, the result in Theorem
	\ref{thm:incoherence} is
	clearly stronger as it allows us to replace $M$---which can grow
	like the reciprocal of the minimum nonzero singular value of
	$\op$---with something akin to the average reciprocal of larger
	singular values.
	But it does, of course, also make stronger assumptions
	(incoherence). It is interesting to note that the functional in the
	theorem, \smash{$\sum_{i=i_0+1}^q \xi_i^{-2}$}, was also determined to
	play a leading role in error bounds for a graph Fourier based scan
	statistic in the hypothesis testing framework
	\citep{sharpnack2013changepoint}.
	
	Applying the above theorem to the GTF estimator requires knowledge of
	the singular vectors of $\op=\op^{(k+1)}$, the $(k+1)$st order
	graph difference operator.  The validity of an incoherence assumption
	on these singular vectors depend on the graph $G$ in question. 
	When $k$ is odd, these singular vectors are eigenvectors of the 
	Laplacian $L$; when $k$ is even, they are left singular vectors of the
	edge incidence matrix $D$.  Loosely speaking, these vectors will be
	incoherent when neighborhoods of different vertices  
	look roughly the same. Most social networks will
	have this property for the bulk of their vertices (i.e., with
	the exception of a small number of high degree vertices).   
	Grid graphs also have this property.  First, we consider trend
	filtering over a 1d grid, i.e., a chain (which, recall, is essentially
	equivalent to univariate trend filtering).  
	
	\begin{corollary}
		\label{cor:grid1}
		Consider the GTF estimator \smash{$\hbeta$} of order $k$, over a chain
		graph, i.e., a 1d grid graph.  Letting 
		\begin{equation*}
		\lambda=\Theta\left(n^{\frac{2k+1}{2k+3}} (\log{n})^{\frac{1}{2k+3}}  
		\|\op^{(k+1)}\beta_0\|_1^{-\frac{2k+1}{2k+3}}\right), 
		\end{equation*}
		the estimator \smash{$\hbeta$} (here, essentially, the univariate
		trend filtering estimator) satisfies    
		\begin{equation*}
		\frac{\|\hbeta-\beta_0\|_2^2}{n} =
		O_\P \left( n^{-\frac{2k+2}{2k+3}}
		(\log{n})^{\frac{1}{2k+3}} \cdot
		\left(n^k\|\op^{(k+1)}\beta_0\|_1\right)^{\frac{2}{2k+3}}\right).   
		\end{equation*}
	\end{corollary}
	
	We note that the above corollary essentially recovers the optimal rate
	of convergence for the univariate trend filtering estimator, for all
	orders $k$. (To be precise, it studies GTF on a chain graph instead,
	but this is basically the same problem.)  When 
	\smash{$n^k\|\op^{(k+1)}\beta_0\|_1$} is assumed to be bounded, a
	natural assumption in the univariate setting, the corollary shows the 
	estimator to converge at the rate 
	$n^{-(2k+2)/(2k+3)} (\log{n})^{1/(2k+3)}$.
	Ignoring the log factor, this matches the minimax optimal rate as 
	established in \citet{trendfilter,fallfact}.  Importantly, the proof
	of Corollary \ref{cor:grid1}, unlike that used in
	previous works, is free from any dependence on univariate spline
	functions; it is completely graph-theoretic, and only uses on the 
	incoherence properties of the 1d grid graph.  The strength of this
	approach is its wider applicability, which we demonstrate by moving up 
	to 2d grids. 
	
	\begin{corollary}
		\label{cor:grid2}
		Consider the GTF estimator \smash{$\hbeta$} of order $k$, over a 2d
		grid graph, of size \smash{$\sqrt{n} \times \sqrt{n}$}.
		Letting 
		\begin{equation*}
		\lambda=\Theta\left(n^{\frac{2k+1}{2k+5}} (\log{n})^{\frac{1}{2k+5}}  
		\|\op^{(k+1)}\beta_0\|_1^{-\frac{2k+1}{2k+5}}\right), 
		\end{equation*}
		the estimator \smash{$\hbeta$} satisfies    
		\begin{equation*}
		\frac{\|\hbeta-\beta_0\|_2^2}{n} =
		O_\P \left( n^{-\frac{2k+4}{2k+5}}
		(\log{n})^{\frac{1}{2k+5}} \cdot
		\left(n^{\frac{k}{2}}\|\op^{(k+1)}\beta_0\|_1\right)^{\frac{4}{2k+5}}\right).    
		\end{equation*}
	\end{corollary}
	
	The 2d result in Corollary \ref{cor:grid2} is written in a form that
	mimics the 1d result in Corollary \ref{cor:grid1}, as we claim that
	the analog of boundedness of \smash{$n^k \|\op^{(k+1)}\beta_0\|_1$} in
	1d is boundedness of  \smash{$n^{k/2} \|\op^{(k+1)}\beta_0\|_1$} in  
	2d.\footnote{This is because \smash{$1/\sqrt{n}$} is the distance  
		between adjacent 2d grid points, when viewed as a 2d lattice over
		$[0,1]^2$.} Thus, under the appropriate boundedness condition, the
	2d rate shows improvement over the 1d rate, which makes sense, since 
	regularization here is being enforced in a richer manner.  It is
	worthwhile highlighting the result for $k=0$ in particular: this says
	that, when the sum of absolute discrete differences
	\smash{$\|\op^{(1)}\beta_0\|_1$} is bounded over a 2d grid, the 2d
	fused lasso (i.e., 2d total variation denoising) has error rate
	$n^{-4/5}$.  This is faster than the $n^{-2/3}$ rate for the 1d fused
	lasso, when the sum of absolute differences
	\smash{$\|D^{(1)}\beta_0\|_1$} is bounded. 
	Rates for higher dimensional grid graphs (for all $k$) follow from
	analogous arguments, but we omit the details.      
	
	
	
	\subsection{Strong Error Bounds Based on Entropy}
	\label{sec:entropy}
	
	A different ``fractional'' bound on the Gaussian contrast
	$\epsilon^\top x$, over $x \in \cS_\op(1),$ provides an
	alternate route to deriving sharper rates.  This style of bound is
	inspired by the seminal work of \citet{vandegeer1990}.
	
	\begin{lemma}
		\label{lem:fractional}
		Denote $\epsilon \sim \cN(0,\sigma^2 I)$, and assume that for a
		constant $w < 2$,
		\begin{equation}
		\label{eq:fractional}
		\max_{x \in \cS_\op(1)} \frac{\epsilon^\top x}{\|x\|_2^{1-w/2}}
		= O_\P(K),
		\end{equation}
		where recall
		$\cS_\op(1) = \{x \in \row(\op) : \|\op x\|_1 \leq 1\}$.
		Then with
		\begin{equation*}
		\lambda = \Theta \left(K^{\frac{2}{1+w/2}} \cdot
		\|\op \beta_0\|_1^{-\frac{1-w/2}{1+w/2}} \right),
		\end{equation*}
		the generalized lasso estimate \smash{$\hbeta$} satisfies
		\begin{equation*}
		\frac{\|\hbeta-\beta_0\|_2^2}{n} =
		O_\P \left( \frac{\nuli(\op)}{n} +
		\frac{K^{\frac{2}{1+w/2}}}{n} \cdot
		\|\op\beta_0\|_1^{\frac{w}{1+w/2}} \right).
		\end{equation*}
	\end{lemma}
	
	The main motivation for bounds of the form \eqref{eq:fractional} is
	that they follow from entropy bounds on the set 
	$\cS_\op(1)$.  Recall that for a set $S$, the covering number
	$N(\delta,S,\|\cdot\|)$ is the fewest number of balls of radius
	$\delta$ that cover $S$, with respect to the norm $\|\cdot\|$. 
	The log covering or entropy number is $\log N(\delta,S,\|\cdot\|)$.
	In the next result, we make the connection between between entropy and 
	fractional bounds precise; this follows closely from Lemma 3.5 of
	\citet{vandegeer1990}. 
	
	\begin{theorem}
		\label{thm:entropy}
		Suppose that there exist a constant $w < 2$ such that for $n$ large
		enough, 
		\begin{equation}
		\label{eq:entropy}
		\log N(\delta, \cS_\op(1),\|\cdot\|_2) \leq E
		\Big(\frac{\sqrt{n}}{\delta}\Big)^w,
		\end{equation}
		for $\delta>0$, where $E$ can depend on $n$. 
		Then the fractional bound in \eqref{eq:fractional} holds with 
		\smash{$K= \sqrt{E} n^{w/4}$}, and as a result, for
		\begin{equation*}
		\lambda=\Theta \left(E^{\frac{1}{1+w/2}} n^{\frac{w/2}{1+w/2}} 
		\|\op \beta_0\|_1^{-\frac{1-w/2}{1+w/2}}\right),
		\end{equation*}
		the generalized lasso estimate \smash{$\hbeta$} has average squared
		error 
		\begin{equation*}
		\frac{\|\hbeta-\beta_0\|_2^2}{n} =
		O_\P \left( \frac{\nuli(\op)}{n} +
		E^{\frac{1}{1+w/2}} n^{-\frac{1}{1+w/2}} \cdot 
		\|\op\beta_0\|_1^{\frac{w}{1+w/2}} \right).
		\end{equation*}
	\end{theorem}

	To make use of the result in Theorem \ref{thm:entropy}, we
	must obtain an entropy bound as in \eqref{eq:entropy}, on the set
	\smash{$\cS_\op(1)$}.  The literature on entropy numbers is rich,   
	and there are various methods for computing entropy bounds, any of
	which can be used for these purposes as long as the bounds fit the
	form of \eqref{eq:entropy}, as required by the theorem.  
	For bounding the entropy of a set like \smash{$\cS_\op(1)$}, two
	common techniques are to use a characterization of the spectral decay
	of $\op^\dag$, 
	or an analysis of the correlations between columns of $\op^\dag$. For
	a nice review of such strategies and their applications, we refer the 
	reader to Section 6 of \citet{lassocorrelated} and Section 14.12 of 
	\citet{statshd}. We do not pursue either of these two strategies in the 
	current paper.  We instead consider a third, somewhat more
	transparent strategy, based on a  covering number bound of the columns
	of $\op^\dag$.   
	
	\begin{lemma}
		\label{lem:atoms}
		Let $g_1,\ldots g_r$ denote the ``atoms'' associated with the operator
		$\op$, i.e., the columns of $\op^\dag$, and let $\cG=\{\pm g_i :
		i=1,\ldots r\}$ denote the symmetrized set of atoms.  Suppose that
		there exists constants $\zeta,C_0$ with the following property: for
		each  $j=1,\ldots 2r$, there is an arrangement of $j$ balls having
		radius at most 
		\begin{equation*}
		C_0 \sqrt{n} j^{-1/\zeta},
		\end{equation*}
		with respect to the norm $\|\cdot\|_2$, that covers $\cG$.  Then the
		entropy bound in \eqref{eq:entropy} is met with $w=2 \zeta / (2 +
		\zeta)$ and $E=O(1)$.  Therefore, the generalized lasso estimate 
		\smash{$\hbeta$}, with
		\begin{equation*}
		\lambda = \Theta \left( n^{\frac{\zeta}{2+2\zeta}} 
		\|\op \beta_0\|_1^{-\frac{1}{1+\zeta}}\right),
		\end{equation*}
		satisfies
		\begin{equation*}
		\frac{\|\hbeta-\beta_0\|_2^2}{n} =
		O_\P \left( \frac{\nuli(\op)}{n} +
		n^{-\frac{2+\zeta}{2+2\zeta}} \cdot 
		\|\op\beta_0\|_1^{\frac{\zeta}{1+\zeta}} \right). 
		\end{equation*}
	\end{lemma}
	
	The entropy-based results in this subsection (Lemma
	\ref{lem:fractional}, Theorem \ref{thm:entropy}, and Lemma
	\ref{lem:atoms}) may appear more complex than those involving 
	incoherence in the previous subsection (Lemma \ref{lem:linearized} and
	Theorem \ref{thm:incoherence}).  Indeed, the same can be said of their
	proofs, which can be found in the Appendix.  But after all this
	entropy machinery has all been established, it can actually be
	remarkably easy to use, say, Lemma \ref{lem:atoms} to produce sharp
	results.  We conclude by giving an example.
	
	\begin{corollary}
		\label{cor:1dfl}
		Consider the 1d fused lasso, i.e., the GTF estimator with $k=0$ over a
		chain graph.  In this case, we have \smash{$\op=D^{(1)}$}, the
		univariate difference operator, and the symmetrized set 
		$\cG$ of atoms can be covered by $j$ balls with radius at most
		$\sqrt{2n/j}$, for $j=1,\ldots 2(n-1)$.   Hence, with 
		\smash{$\lambda=\Theta(n^{1/3} \|D^{(1)}\beta_0\|_1^{-1/3})$}, the
		1d fused lasso estimate \smash{$\hbeta$} satisfies
		\begin{equation*}
		\frac{\|\hbeta-\beta_0\|_2^2}{n} =
		O_\P \left( n^{-2/3} \cdot
		\|D^{(1)} \beta_0\|_1^{2/3} \right).
		\end{equation*}
	\end{corollary}
	
	This corollary rederives the optimal convergence rate of 
	$n^{-2/3}$ for the univariate fused lasso, assuming boundedness of
	\smash{$\|D^{(1)}\beta_0\|_1$}, as has been already shown in
	\citet{locadapt,trendfilter}.  Like Corollary \ref{cor:grid1} (but
	unlike previous works), its proof does not rely on any special facts  
	about 1d functions of bounded variation.  It only uses a covering
	number bound on the columns of the operator \smash{$(D^{(1)})^+$}, a
	strategy that, in principle, extends to many other settings
	(graphs). It is worth emphasizing just how simple this covering number 
	construction is, compared to the incoherence-based arguments that lead
	to the same result; we invite the curious reader to compare the proofs
	of Corollaries \ref{cor:grid1} and  \ref{cor:1dfl}.

	\section{Discussion}
	\label{sec:discussion}
	
	In this work, we proposed graph trend filtering as a
	useful alternative to Laplacian and wavelet smoothers on
	graphs.  This is analogous to the usefulness of univariate trend
	filtering in nonparametric regression, as an alternative to smoothing
	splines and wavelets \citep{trendfilter}.  We have documented
	empirical evidence for the superior local adaptivity of the
	$\ell_1$-based GTF over the $\ell_2$-based graph Laplacian smoother,
	and the superior robustness of GTF over wavelet smoothing in
	high-noise scenarios. Our theoretical analysis provides a basis for a
	deeper understanding of the estimation properties of GTF.
	More precise theoretical characterizations involving entropy 
	will be the topic of future work, as will comparisons between the error
	rates achieved by GTF and other common estimators, such as Laplacian  
	smoothing.  These extensions, and many others, are well within reach.   
	
	\subsubsection*{Acknowledgments}
	
	The authors would like to thank Harish Doraiswamy, Nivan Ferreira,
	Theodoros Damoulas and Claudio Silva for sharing the pre-processed NYC
	taxi data, Jeff Irion and Naoki Saito for their help with the
	implementation of the graph wavelet algorithms, as well as the 
	associate editor and anonymous reviewers for the valuable feedback.  
	
	YW was supported by the Singapore National Research Foundation under
	its International Research Centre @ Singapore Funding Initiative and
	administered by the IDM Programme Office.  JS was supported by
	NSF Grant DMS-1223137.  AS was supported by a Google Faculty
	Research Grant.  RT was supported by NSF Grant DMS-1309174.
	
	\appendix
	
	\section{Additional Analysis from Alternative Wavelet Designs}
	
	We provide detailed comparisons to a few recently proposed wavelet
	approaches for graph smoothing.  
	
	\subsection{Allegheny County Example}  
	\label{sec:allegheny}
	
	In addition to considering the wavelet design of \citet{graphwave} for
	the Allegheny County example, we also considered designs of
	\citet{coifman2006}---a
	classic method that builds diffusion wavelets on a graph, and
	\citet{irion2015}---a more recent graph wavelet construction.  
	In contrast to \citet{graphwave}, which produces a single
	signal-independent orthogonal basis for a graph, both 
	\citet{coifman2006,irion2015} build wavelet packets from a given graph  
	structure. A wavelet packet is an overcomplete basis indexed by a
	hierarchical data structure that can be used to generate
	an exponential number of orthogonal bases. This construction is
	computationally expensive as it typically involves computing 
	eigendecompositions of large matrices.  
	Once the wavelet packet has been constructed, for each input signal
	that one observes over the graph in question, one runs a ``best basis'' 
	algorithm to choose a particular orthogonal basis from the wavelet 
	packet by optimizing a particular cost function of the eventual
	wavelet coefficients. This is based on a message-passing-like dynamic
	programming algorithm, and can be quite efficient. Lastly,
	the denoising procedure is defined as usual (e.g., as in \citet{sure}),
	namely, one performs the basis transformation, soft-thresholds (or
	hard-thresholds) the coefficients, and then reconstructs the denoised
	signal.  
	
	In our experiments, we used the wavelet implementations released by
	the authors of \citet{coifman2006,irion2015} with their default
	settings. In particular, the former implementation of
	\citet{coifman2006} builds wavelets from a diffusion operator
	constructed from the adjacency matrix of a graph, and the cost
	function for the best basis is defined by the $\ell_1$ norm of the
	wavelet coefficients.  The latter implementation of \citet{irion2015}
	uses a more exhaustive search, building wavelet packets
	through a hierarchical partitioning and eigentransform of three
	different Laplacian matrices and a fourth generalized
	Haar-Walsh transform (GHWT), then choosing the best basis from all 
	four collections by optimizing a meta cost function of the
	$\ell_p$ norm of wavelet coefficients over $p \in \{0.1,0.2,\ldots
	2\}$.  This is the ``cumulative relative error'' defined in equation
	(7.5) of \citet{irion2015}.  
	
	In the left panel of Figure \ref{fig:pittsburgh-more}, we plot the
	mean squared errors for these new wavelet methods over the same 10
	simulations from the Allegheny County example in Figure
	\ref{fig:pittsburgh-errs} of
	Section \ref{sec:allegheny}.  The middle and right panels of
	the figure show the denoised signals from the
	new methods fit to the data in Figure
	\ref{fig:pittsburgh-maps}, at their optimal degrees of freedom (df)
	values (in terms of the achieved MSE).  
	We can see that the spanning tree
	wavelet design of \citet{graphwave} is the best performer among the
	three candidate wavelet designs. In a
	rough sense, the construction of \citet{irion2015} seems to perform
	similarly to that of \citet{graphwave}, in that the MSE is best for
	larger df values (corresponding to more nonzero wavelet coefficients,
	i.e., complex fitted models), whereas the construction of
	\citet{coifman2006} performs best for smaller df values 
	(fewer nonzero wavelet coefficients, i.e., simpler
	fitted models).
	
	\begin{figure}[!htb]
		\centering
		\begin{tabular}{ccc}
			MSE comparison & Irion wavelets, 194 df & Coifman wavelets, 78 df \\
			\includegraphics[width=0.33\textwidth]{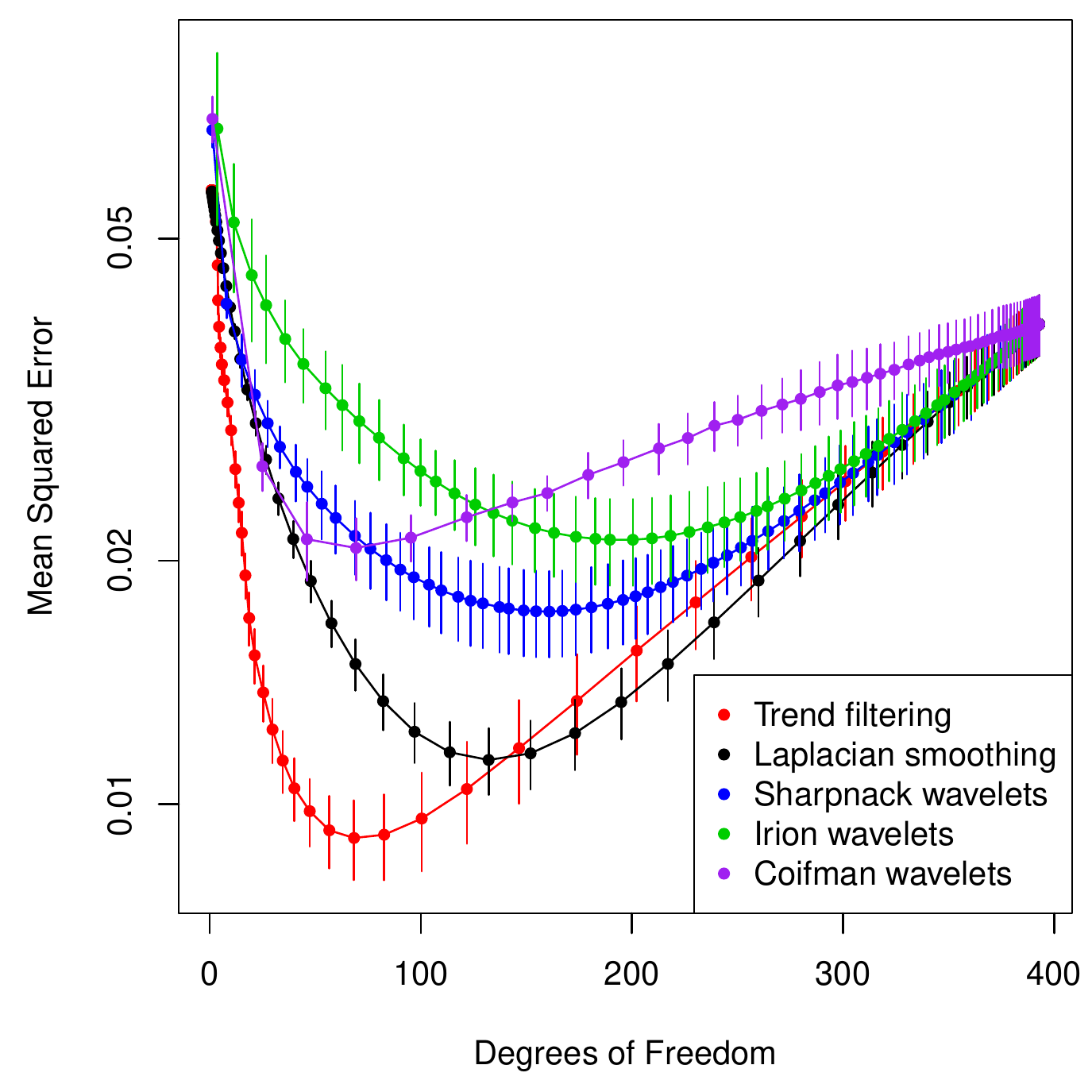} &
			\hspace{-5mm }\includegraphics[width=0.33\textwidth]{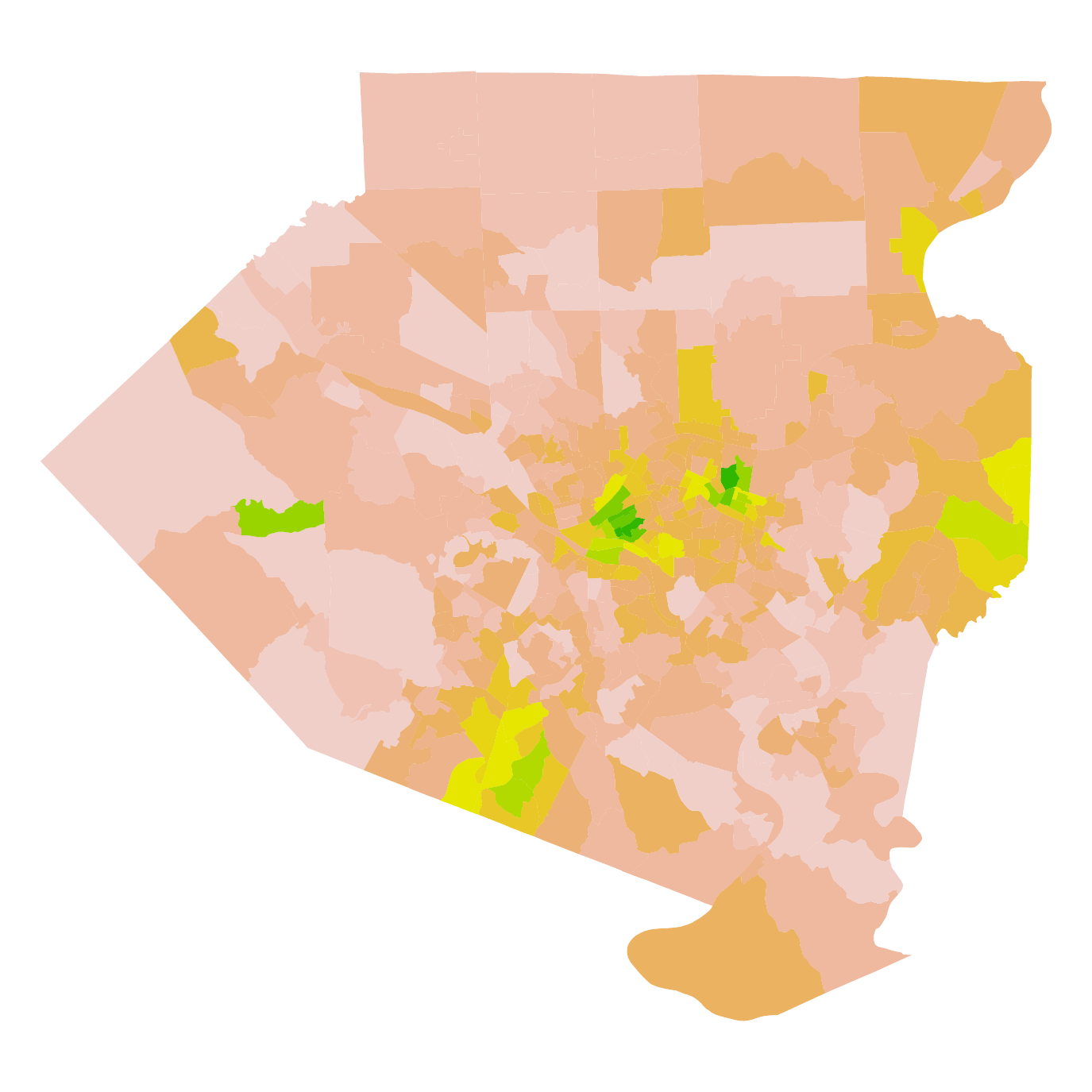} &
			\hspace{-5mm }\includegraphics[width=0.33\textwidth]{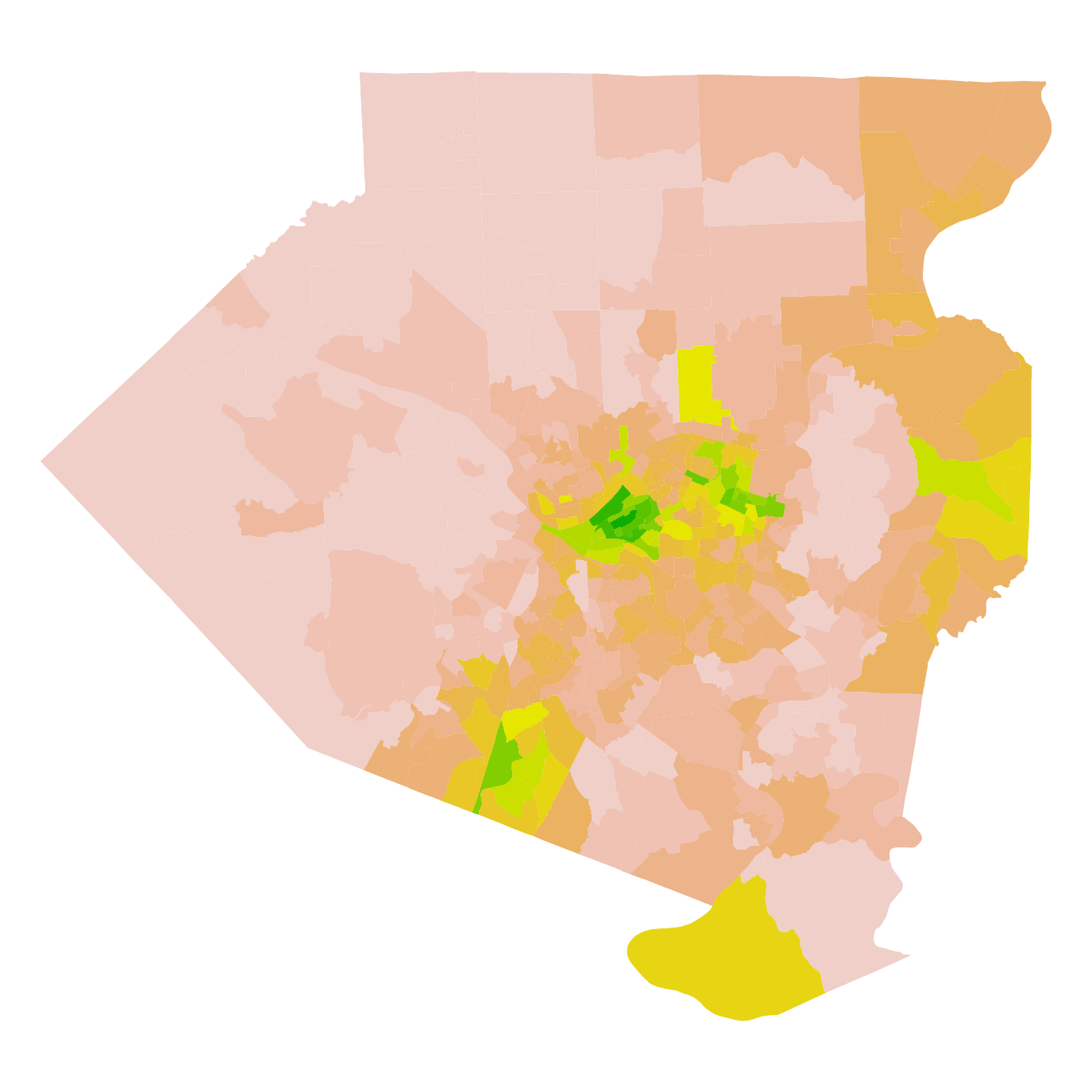} \\
		\end{tabular}
		\caption{Additional wavelet analysis of the Allegheny County example.}
		\label{fig:pittsburgh-more}
	\end{figure}
	
	\subsection{Facebook Graph Example}  
	
	Again, we consider the designs of \citet{coifman2006,irion2015} for the
	Facebook graph example of Section \ref{sec:facebook}.  Due to
	practical reasons, we had to change some of the default settings in
	the implementations provided by the authors of these 
	wavelet methods; in the wavelet implementation of \citet{coifman2006},
	we took the power of the diffusion operator to be 1 instead of 4
	(since the latter choice threw an error in the provided code); and in
	the wavelet implementation of \citet{irion2015}, we used another ``best
	basis'' algorithm that only searches within the basis collection 
	from the GHWT eigendecomposition, as the original algorithm was
	too slow due to the larger scale considered in this example.  (In most
	examples in \citet{irion2015}, the chosen bases come from the
	GHWT eigendecomposition.) We view these changes as minor, because when
	the same changes were applied to the methods of
	\citet{coifman2006,irion2015} on the smaller Allegheny County example,
	there are no obvious differences in the results. 
	
	Figure \ref{fig:facebook-more} shows the results for the two new
	wavelet methods on the Facebook graph simulation, using the same   
	setup as in Figure \ref{fig:facebook}.  Once again, we find that the
	spanning tree wavelets of \citet{graphwave} perform better or on par
	with the other two wavelet methods across essentially all scenarios.
	
	\begin{figure}[!htb]
		\centering
		\begin{tabular}{cc}
			Dense Poisson equation & Sparse Poisson equation \\
			\includegraphics[width=0.4\textwidth]{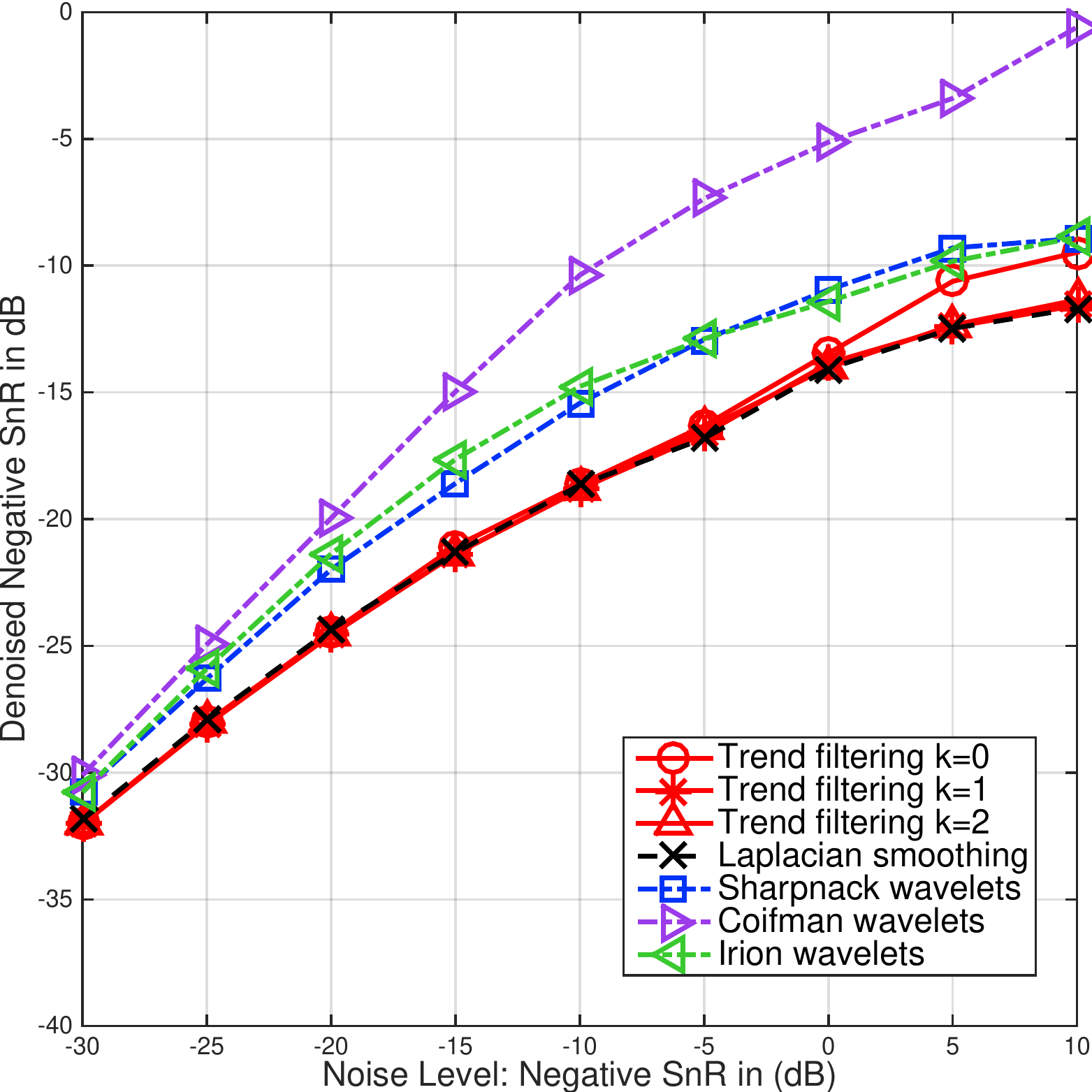} & 
			\includegraphics[width=0.4\textwidth]{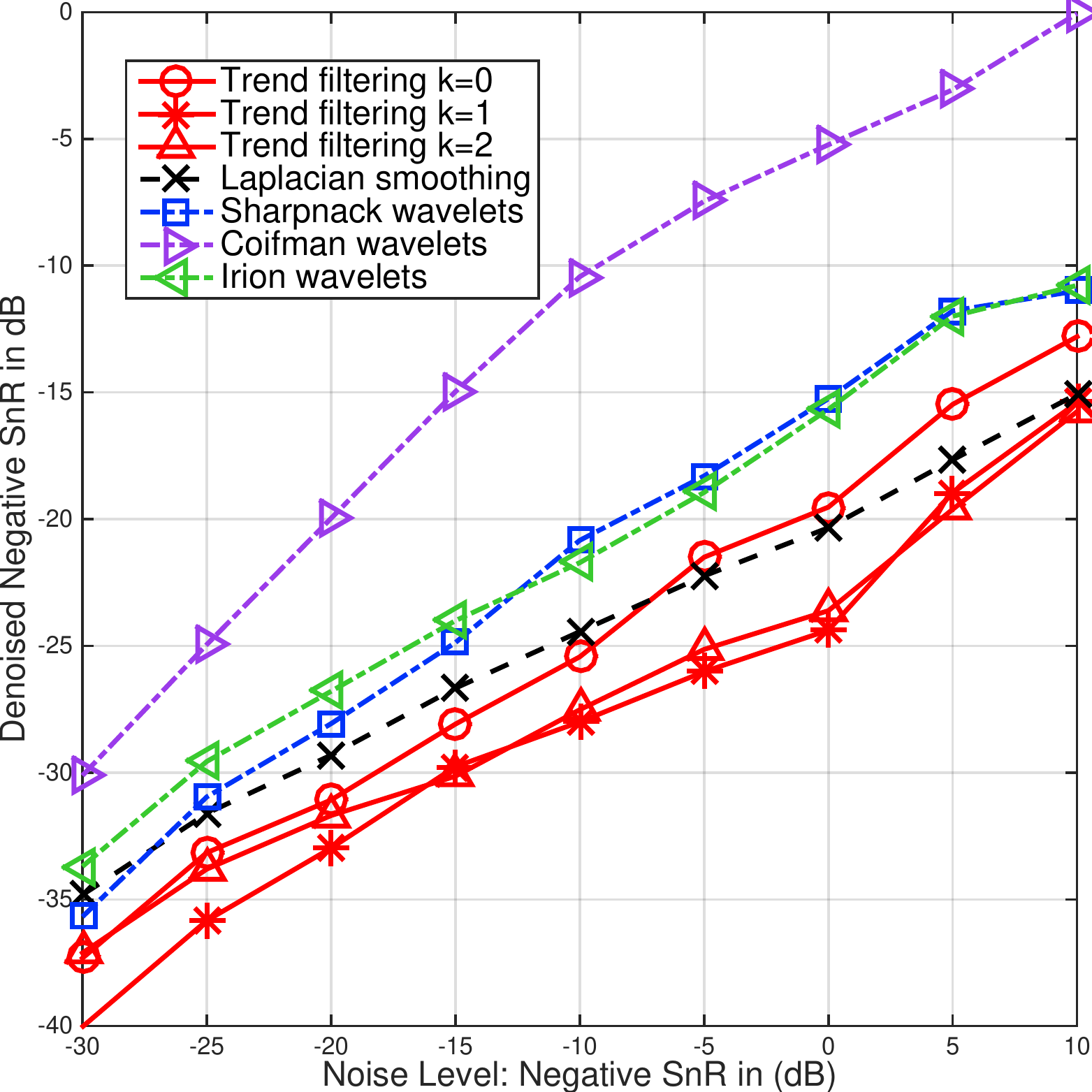}  
		\end{tabular} 
		
		\smallskip\smallskip 
		Inhomogeneous random walk \\ 
		\includegraphics[width=0.4\textwidth]{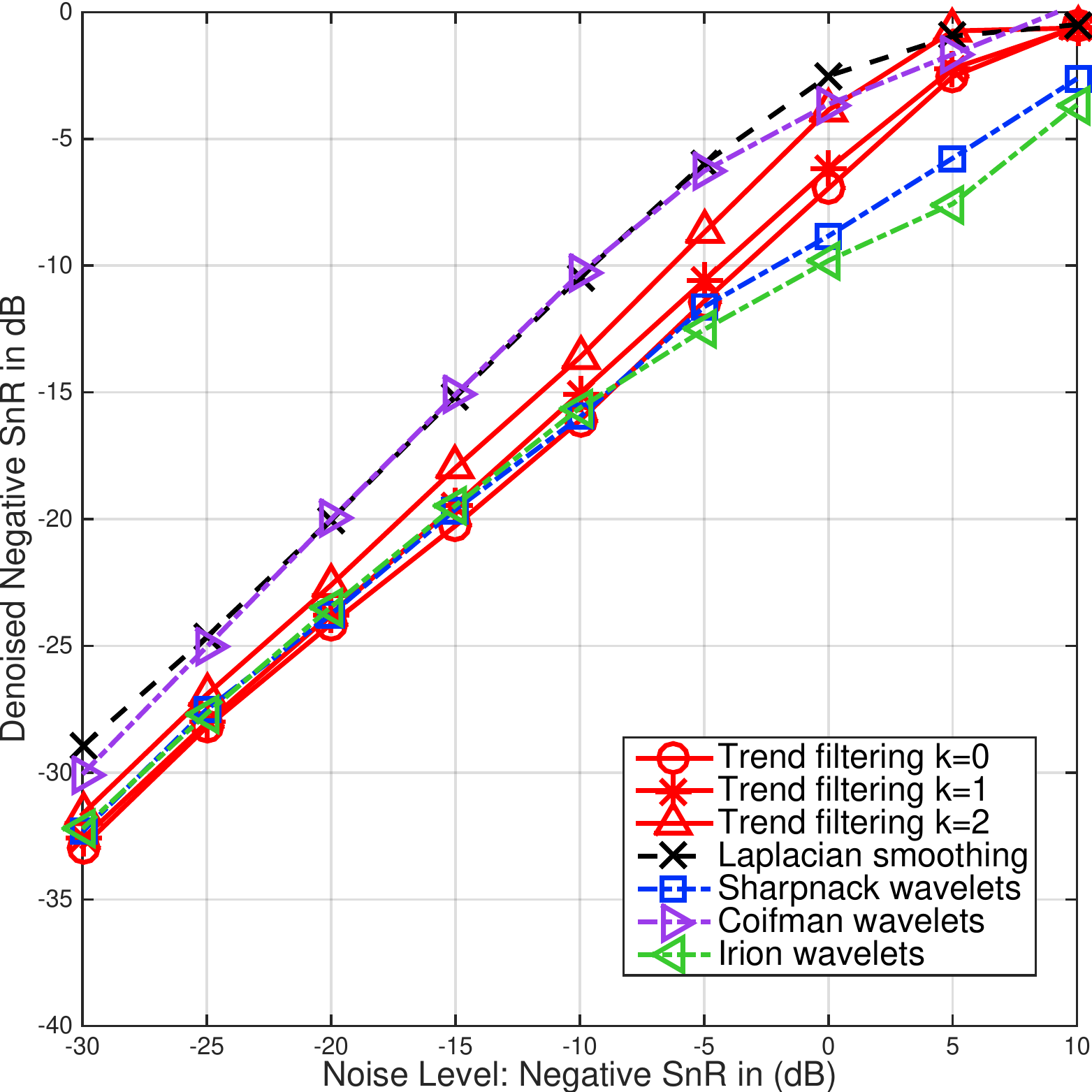}
		
		\caption{Additional wavelet analysis of the Facebook graph example.} 
		\label{fig:facebook-more}
	\end{figure}

	\section{Proofs of Theoretical Results}
	
	Here we present proofs of our theoretical results
	presented in Sections~\ref{sec:properties}~and~\ref{sec:theory}.
	
	\subsection{Proof of Lemma \ref{lem:nullspace}} 
	
	For even $k$, we have \smash{$\op^{(k+1)}=DL^{k/2}$}, so if
	$A$ denotes a subset of edges, then
	\smash{$\op^{(k+1)}_{-A}=D_{-A} L^{k/2}$}. Recall
	that for a connected graph, \smash{$\nul(L)=\spa\{\mathds{1}\}$}, and
	the same is true for any power of $L$.  This means that we can write 
	\begin{equation*}
	\nul(\op^{(k+1)}) = \spa\{\mathds{1}\} + 
	\spa\{\mathds{1}\}^\perp \cap \{ u : DL^{\frac{k}{2}} u = 0\}.  
	\end{equation*}
	Note that if $\mathds{1}^\top u=0$, then \smash{$v=L^{\frac{k}{2}} u 
		\iff u = (L^\dag)^{\frac{k}{2}} u$}.  Moreover, if $G_{-A}$ has   
	connected components $C_1,\ldots C_s$, then 
	\smash{$\nul(D_{-A})=\spa\{\mathds{1}_{C_1},\ldots
		\mathds{1}_{C_s}\}$}.  Putting these statements together
	proves the result for even $k$.  For $k$ odd, the arguments are
	similar.  
	
	\subsection{Proof of Theorem \ref{thm:basic}}
	
	By assumption we can write
	\begin{equation*}
	y = \beta_0 + \epsilon, \;\;\; \epsilon \sim \cN(0,\sigma^2 I).
	\end{equation*}
	Denote $R=\row(\op)$, the row
	space of $\op$, and $R^\perp=\nul(\op)$, the null space of
	$\op$.  Also let $P_R$ be the projection onto $R$,  and
	$P_{R^\perp}$ the projection onto $R^\perp$. Consider
	\begin{align*}
	\hbeta &= \argmin_{\beta\in\R^n} \, \half\|y-\beta\|_2^2 +
	\lambda \|\op\beta\|_1, \\
	\tbeta &= \argmin_{\beta\in\R^n} \, \half\|P_R y-\beta\|_2^2 +
	\lambda \|\op\beta\|_1.
	\end{align*}
	The first quantity \smash{$\hbeta \in \R^n$} is the estimate of
	interest, the second one \smash{$\tbeta \in R$} is easier to analyze.
	Note that
	\begin{equation*}
	\hbeta = P_{R^\perp} y + \tbeta,
	\end{equation*}
	and write $\|x\|_R = \|P_R x\|_2$,
	$\|x\|_{R^\perp} = \|P_{R^\perp} x\|_2$. Then
	\begin{equation*}
	\|\hbeta-\beta_0\|_2^2 = \|\epsilon\|_{R^\perp}^2 +
	\|\tbeta-\beta_0\|_R^2.
	\end{equation*}
	The first term is on the order
	$\mathrm{dim}(R^\perp)=\nuli(\op)$, and it suffices to
	bound the second term.
	Now we establish a basic inequality for \smash{$\tbeta$}.  By
	optimality of \smash{$\tbeta$}, we have
	\begin{equation*}
	\half\|y-\tbeta\|_R^2 + \lambda\|\op\tbeta\|_1 \leq
	\half\|y-\beta_0\|_R^2 + \lambda\|\op\beta_0\|_1,
	\end{equation*}
	and after rearranging terms,
	\begin{equation}
	\label{eq:basic}
	\|\tbeta-\beta_0\|_R^2 \leq 2\epsilon^\top P_R (\tbeta-\beta_0) +
	2\lambda\|\op\beta_0\|_1 - 2\lambda\|\op\tbeta\|_1.
	\end{equation}
	This is our basic inequality. In the first term above, we use
	$P_R = \op^\dag\op$, and apply Holder's inequality:
	\begin{equation}
	\label{eq:holder}
	\epsilon^\top \op^\dag \op (\tbeta-\beta_0) \leq
	\|(\op^\dag)^\top \epsilon\|_\infty \|\op(\tbeta-\beta_0)\|_1.
	\end{equation}
	If \smash{$\lambda \geq \|(\op^\dag)^\top\epsilon\|_\infty$}, then
	from \eqref{eq:basic}, \eqref{eq:holder}, and the triangle inequality,
	we see that  
	\begin{equation*}
	\|\tbeta-\beta_0\|_R^2 \leq 4 \lambda \|\op \beta_0\|_1.
	\end{equation*}
	Well, 
	\smash{$\|(\op^\dag)^\top\epsilon\|_\infty = O_\P(M \sqrt{\log{r}})$}
	by a standard result on the maximum of Gaussians (derived using the
	union bound, and Mills' bound on the Gaussian tail), where
	recall $M$ is the maximum $\ell_2$ norm of the columns of $\op^\dag$.
	Thus with $\lambda=\Theta(M\sqrt{\log{r}})$, we have from the above
	that 
	\begin{equation*}
	\|\tbeta-\beta_0\|_R^2 = O_\P \big( M \sqrt{\log{r}}
	\|\op\beta_0\|_1\big),
	\end{equation*}
	as desired.
	
	\subsection{Proof of Corollary \ref{cor:basic}}
	
	{\bf Case 1.} When \smash{$\hbeta$} is the univariate trend filtering
	estimator of order $k$, we are considering a penalty matrix
	$\op=D^{(k+1)}$, the univariate difference operator of order $k+1$.
	Note that $D^{(k+1)} \in \R^{(n-k-1) \times n}$, and its
	null space has constant dimension $k+1$.
	We show in Lemma \ref{lem:dpinv} of Appendix
	\ref{app:dpinv} that 
	\smash{$(D^{(k+1)})^\dag = P_R H_2^{(k)} / k!$}, where
	$R=\row(D^{(k+1)})$, and \smash{$H_2^{(k)} \in \R^{n\times (n-k-1)}$}
	contains the last $n-k-1$ columns of the order $k$ falling factorial
	basis matrix \citep{fallfact}, evaluated over the input points
	$x_1=1,\ldots x_n=n$. The largest column norm of
	\smash{$P_R H_2^{(k)} / k!$} is on the order of $n^{k+1/2}$, which
	proves the result.
	
	\smallskip\smallskip
	\noindent
	{\bf Cases 2 and 3.} When $G$ is the Ramanujan $d$-regular graph,
	the number of edges in the graph is $O(nd)$. The operator
	$\op=\op^{(k+1)}$ has number of rows $r=n$ when $k$ is odd
	and $r=O(nd)$ when $k$ is even; overall this is $O(nd)$.  The
	dimension of the null space of $\op$ is constant (it is in fact 1,
	since the graph is connected).  When $G$ is the Erdos-Renyi random
	graph, the same bounds apply to the number of rows and the dimension
	of the null space, except that the bounds become probabilistic ones.
	
	Now we apply the crude inequality, with $e_i$, $i=1,\ldots r$ denoting
	the standard basis vectors,
	\begin{equation*}
	M = \max_{i=1,\ldots r} \, \op^\dag e_i \leq \max_{\|x\|_2 \leq 1}
	\op^\dag x = \|\op^\dag\|_2,
	\end{equation*}
	the right-hand side being the maximum singular value of $\op^\dag$.
	As $\op=\op^{(k+1)}$, the graph difference operator of order $k+1$,
	we claim that
	\begin{equation}
	\label{eq:claim}
	\|\op^\dag\|_2 \leq 1/\lambda_{\min}(L)^{\frac{k+1}{2}},
	\end{equation}
	where $\lambda_{\min}(L)$ denotes the smallest nonzero eigenvalue of
	the graph Laplacian $L$.  To see this, note first that
	$\|\op^\dag\|_2=1/\sigma_{\min}(\op)$, where the denominator is the
	smallest nonzero singular value of $\op$.  Now for odd $k$, we have
	\smash{$\op^{(k+1)}=L^{(k+1)/2}$}, and the claim follows as
	\begin{equation*}
	\sigma_{\min}(L^{\frac{k+1}{2}}) =
	\min_{x \in R: \|x\|_2 \leq 1} \|L^{\frac{k+1}{2}} x\|_2 \geq 
	\big(\sigma_{\min}(L)\big)^{\frac{k+1}{2}},
	\end{equation*}
	and \smash{$\sigma_{\min}(L)=\lambda_{\min}(L)$}, since $L$ is
	symmetric. Above, $R$ denotes the row space of $L$ (the space
	orthogonal to the vector $\mathds{1}$ of all 1s). For even $k$, we
	have \smash{$\op^{(k+1)}=DL^{k/2}$}, and again
	\begin{equation*}
	\sigma_{\min}(DL^{\frac{k}{2}}) =
	\min_{x \in R: \|x\|_2 \leq 1} \| DL^{\frac{k+1}{2}} x\|_2 \geq 
	\sigma_{\min}(D) \big(\sigma_{\min}(L)\big)^{\frac{k}{2}},
	\end{equation*}
	where $\sigma_{\min}(D)=\sqrt{\lambda_{\min}(L)}$, since
	$D^\top D = L$. This verifies the claim.
	
	Having established \eqref{eq:claim}, it suffices to lower bound
	$\lambda_{\min}(L)$ for the two graphs in question.  Indeed, for both
	graphs, we have the lower bound
	\begin{equation*}
	\lambda_{\min}(L) = \Omega(d-\sqrt{d}).
	\end{equation*}
	e.g., see \citet{lubotzky1988ramanujan, marcus2014ramanujan} for the
	Ramanujan graph and \citet{feige2005spectral, chung2011spectra} for the
	Erdos-Renyi graph. This completes the proof.
	
	\subsection{Calculation of \smash{$(D^{(k+1)})^\dag$}}
	\label{app:dpinv}
	
	\begin{lemma}
		\label{lem:dpinv}
		The $(k+1)$st order discrete difference operator has pseudoinverse   
		$$(D^{(k+1)})^\dag =  P_R H_2^{(k)} / k!,$$
		where we denote $R=\row(D^{(k+1)})$, and 
		\smash{$H_2^{(k)} \in \R^{n\times (n-k-1)}$} the last $n-k-1$   
		columns of the $k$th order falling factorial basis matrix.  
	\end{lemma}
	
	\begin{proof}
		We abbreviate $D=D^{(k+1)}$, and consider the linear system
		\begin{equation}
		\label{eq:ddt}
		DD^\top x = Db
		\end{equation}
		in $x$, where $b\in\R^n$ is arbitrary. We seek an expression for
		$x=(DD^\top)^{-1} D^\top = (D^\dag)^\top b$, and this will tell us the 
		form of $D^\dag$.  Define
		\begin{equation*}
		\tD = \left[\begin{array}{c} C \\ D \end{array}\right] \in
		\R^{n\times n},
		\end{equation*}
		where $C \in \R^{(k+1)\times n}$ is the matrix that collects the first
		row of each lower order difference operator, defined in Lemma 2 of
		\citet{fallfact}. From this same lemma, we know that
		\begin{equation*}
		\tD^{-1} = H/k!,
		\end{equation*}
		where $H=H^{(k)}$ is falling factorial basis matrix of order $k$,
		evaluated over $x_1,\ldots x_n$.  With this in mind, consider the
		expanded linear system
		\begin{equation}
		\label{eq:ddt2}
		\left[\begin{array}{cc} CC^\top & CD^\top \\
		DC^\top & DD^\top \end{array}\right]
		\left[\begin{array}{c} w \\ x \end{array}\right] =
		\left[\begin{array}{c} a \\ Db \end{array}\right].
		\end{equation}
		The second equation reads
		\begin{equation*}
		DC^\top w + DD^\top x = Db,
		\end{equation*}
		and so if we can choose $a$ in \eqref{eq:ddt2} so that at the
		solution we have $w=0$, then $x$ is the solution in \eqref{eq:ddt}.
		The first equation in \eqref{eq:ddt2} reads
		\begin{equation*}
		CC^\top w + CD^\top x = a,
		\end{equation*}
		i.e.,
		\begin{equation*}
		w = (CC^\top)^{-1} ( a - CD^\top x ).
		\end{equation*}
		That is, we want to choose
		\begin{equation*}
		a = CD^\top x =CD^\top(DD^\top)^{-1} D b = CP_R b, 
		\end{equation*}
		where $P_R$ is the projection onto row space of $D$.
		Thus we can reexpress \eqref{eq:ddt2} as
		\begin{equation*}
		\tD\tD^\top
		\left[\begin{array}{c} w \\ x \end{array}\right] =
		\left[\begin{array}{c} C P_R b \\ Db \end{array}\right]
		= \tD P_R b
		\end{equation*}
		and, using \smash{$\tD^{-1} = H/k!$},
		\begin{equation*}
		\left[\begin{array}{c} w \\ x \end{array}\right] = H^\top
		P_R b/k!.
		\end{equation*}
		Finally, writing $H_2$ for the last $n-k-1$ columns of $H$, we
		have $x = H_2^\top P_R b/k!$, as desired.
	\end{proof}
	
	\noindent
	{\it Remark.} The above proof did not rely on the input points
	$x_1,\ldots x_n$; indeed, the result holds true for any sequence
	of inputs used to define the discrete difference matrix and falling 
	factorial basis matrix.
	
	\subsection{Proof of Lemma \ref{lem:linearized}}
	
	We follow the proof of Theorem \ref{thm:basic}, up until the
	application of Holder's inequality in \eqref{eq:holder}. In place of
	this step, we use the linearized bound in \eqref{eq:linearized}, which
	we claim implies that
	\begin{equation*}
	\epsilon^\top P_R (\tbeta-\beta_0) \leq
	\tB \|\tbeta-\beta_0\|_R + A \|\op(\tbeta-\beta_0)\|_1, 
	\end{equation*}
	where \smash{$\tB=O_\P(B)$}.
	This simply follows from applying \eqref{eq:linearized} to
	\smash{$x=P_R(\tbeta-\beta_0)/\|\op(\tbeta-\beta_0)\|_1$}, which is
	easily seen to be an element of $\cS_\op(1)$.
	Hence we can take take $\lambda = \Theta(A)$, and argue as in the 
	proof of Theorem \ref{thm:basic} to arrive at
	\begin{equation*}
	\|\tbeta-\beta_0\|_R^2 \leq \tB \|\tbeta-\beta_0\|_R +
	\tA \|\op\beta_0\|_1,
	\end{equation*}
	where \smash{$\tA=O_\P(A)$}.  Note that the above is a quadratic
	inequality of the form $ax^2-bx-c \leq 0$ with
	\smash{$x=\|\tbeta-\beta_0\|_R$}.  As $a>0$, the  
	larger of its two roots serves as a bound for $x$, i.e.,
	$x \leq (b + \sqrt{b^2+4ac})/(2a) \leq b/a + \sqrt{c/a}$, or 
	$x^2 \leq 2b^2/a^2 + 2c/a$, which means that
	\begin{equation*}
	\|\tbeta-\beta_0\|_R^2 \leq 2\tB^2 + 2\tA\|\op\beta_0\|_1 
	= O_\P \big( B^2 + A\|\op\beta_0\|_1\big),
	\end{equation*}
	completing the proof.
	
	\subsection{Proof of Theorem \ref{thm:incoherence}}
	
	For an index $i_0 \in \{1,\ldots q\}$, let
	\begin{equation*}
	C = \mu \sqrt{\frac{2\log{2r}}{n}
		\sum_{i=i_0+1}^q \frac{1}{\xi_i^2}}.
	\end{equation*}
	We will show that
	\begin{equation*}
	\max_{x \in \cS_\op(1)}
	\frac{ \epsilon^\top x  - 1.001\sigma C}{\|x\|_2} = O_\P(\sqrt{i_0}). 
	\end{equation*}
	Invoking Lemma \ref{lem:linearized} with $A=1.001\sigma C$ and
	\smash{$b=\sqrt{i_0}$}  would then give the result.  
	
	Henceforth we denote $[i] = \{1,\ldots i\}$. 
	Recall that $q=\rank(\op)$. Let the singular value decomposition of
	$\op$ be 
	\begin{equation*}
	\op = U \Sigma V^\top,
	\end{equation*}
	where $U \in \R^{r\times q}$, $V \in \R^{n\times q}$ are
	orthogonal, and $\Sigma \in \R^{q \times q}$ has diagonal 
	elements $(\Sigma)_{ii} = \xi_i >0$ for $i \in [q]$.
	First, let us establish that
	\begin{equation*}
	\op^\dag = V \Sigma^{-1} U^\top.
	\end{equation*}
	Consider an arbitrary point $x = P_R z \in \cS_\op(1)$.
	Denote the projection \smash{$P_{[i_0]} =  V_{[i_0]} V_{[i_0]}^\top$}
	where $V_{[i_0]}$ contains the first $i_0$ right singular vectors.  We   
	can decompose
	\begin{equation*}
	\epsilon^\top  P_R z = \epsilon^\top P_{[i_0]} P_R z +
	\epsilon^\top (I - P_{[i_0]}) P_R z.
	\end{equation*}
	The first term can be bounded by
	\begin{equation*}
	\epsilon^\top P_{[i_0]} P_R z \le \| P_{[i_0]} \epsilon \|_2 
	\| z \|_R = O_\P(\sqrt{i_0} \|z\|_R),
	\end{equation*}
	using the fact that $\| P_{[i_0]} \epsilon\|_2^2 \overset{d}{=} 
	\sum_{i=1}^{i_0} \epsilon_i^2$. We can bound the second term by 
	\begin{equation*}
	\epsilon^\top (I - P_{[i_0]}) P_R z = \epsilon^\top
	(I - P_{[i_0]}) \op^\dag \op z \le
	\| (\op^{\dag})^\top (I - P_{[i_0]})
	\epsilon \|_\infty,
	\end{equation*}
	using $P_R = \op^\dag \op$, Holder's inequality, and the fact that
	$\|\op z\|_1 \leq 1$.
	Define $g_j = (I - P_{[i_0]}) \op^\dag e_j$ for $j \in [r]$ with
	$e_j$ the $j$th canonical basis vector. So,
	\begin{equation*}
	\| g_j \|_2^2 =
	\|[\, 0 \;\, V_{[n]\backslash [i_0]} \,] \cdot \Sigma^{-1} U^\top e_j 
	\|_2^2 \le \frac{\mu^2}{n} \sum_{i = i_0+1}^q \frac{1}{\xi_i^2},  
	\end{equation*}
	by rotational invariance of $\|\cdot\|_2$ and the incoherence
	assumption on the columns of $U$.  By a standard result on maxima
	of Gaussians,
	\begin{equation*}
	\|(\op^{\dag})^\top (I - P_{[i_0]}) \epsilon \|_\infty =
	\max_{j \in [r]} \, |g_j^\top \epsilon| \leq 1.001 \sigma 
	\sqrt{2 \log (2r)  \frac{\mu^2}{n} \sum_{i =  i_0+1}^q
		\frac{1}{\xi_i^2} } = 1.001 \sigma C,
	\end{equation*}
	with probability approaching 1.  Putting these two terms together
	completes the proof, as we have shown that
	\begin{equation*}
	\frac{\epsilon^\top P_R z - 1.001\sigma C}{\|z\|_R} =
	O_\P(\sqrt{i_0}),
	\end{equation*}
	with the probability bound on the right-hand side not depending on
	$z$.   
	
	\subsection{Proof of Corollary \ref{cor:grid1}}
	
	We focus on the $k$ odd and $k$ even cases separately.
	
	\smallskip\smallskip
	\noindent
	{\bf Case for $k$ odd.}  When $k$ is odd, we have 
	\smash{$\op=\op^{(k+1)}=L^{(k+1)/2}$}, where $L$ the    
	graph Laplacian of a chain graph (i.e., 1d grid graph), to be
	perfectly explicit,
	\begin{equation*}
	L = \left[\begin{array}{rrrrrr}
	1 & -1 & 0 & \ldots & 0 & 0  \\
	-1 & 2 & -1 & \ldots & 0 & 0 \\
	0 & -1 & 2 & \ldots & 0 & 0 \\
	\vdots & & \ddots & \ddots & \ddots & \\
	0 & 0 & \ldots & -1 & 2 & -1 \\
	0 & 0 & \ldots & 0 & -1 & 1 \\
	\end{array}\right].
	\end{equation*}
	In numerical methods for differential equations, this matrix $L$ is
	called the finite difference operator for the
	1d Laplace equation with Neumann boundary conditions  
	\citep[e.g.,][]{conte1980,godunov1987}, and is known to have
	eigenvalues and eigenvectors 
	\begin{gather*}
	\xi_i = 4 \sin^2\Big(\frac{\pi(i-1)}{2n}\Big), \;\;\; 
	\text{for $i=1,\ldots n$}, \\
	u_{ij} = \begin{cases}
	\frac{1}{\sqrt{n}} & \text{if $i=1$} \smallskip\\
	\sqrt{\frac{2}{n}} \cos\Big(\frac{\pi(i-1)(j-1/2)}{n}\Big) & 
	\text{otherwise}
	\end{cases},
	\;\;\; \text{for $i,j = 1,\ldots n$}.
	\end{gather*}
	Therefore, the eigenvectors of $L$ are incoherent with constant 
	$\mu=\sqrt{2}$.   
	This of course implies the same of $L^{(k+1)/2}$, which shares the
	eigenvectors of $L$.  Meanwhile, the eigenvalues of $L^{(k+1)/2}$ are
	just given by raising those of $L$ to the power of $(k+1)/2$, and for 
	$i_0 \in \{1,\ldots n\}$, we compute the partial sum of their squared 
	reciprocals, as in
	\begin{equation*}
	\frac{1}{n} \sum_{i=i_0+1}^n \frac{1}{\xi_i^{k+1}} = 
	\frac{1}{n} \sum_{i=i_0+1}^n 
	\frac{1}{4^{k+1} \sin^{2k+2}(\pi(i-1)/(2n))}  
	\leq \int_{(i_0-1)/n}^{(n-2)/n} \frac{1}{4^{k+1} \sin^{2k+2}(\pi x/2)}
	dx,  
	\end{equation*}
	where we have used the fact that the right-endpoint Riemann sum, for 
	a monotone nonincreasing function, is an underestimate of its
	integral.  Continuing on, the above integral can be bounded by
	\begin{equation*}
	\frac{1}{4^{k+1} \sin^{2k}(\pi i_0/(2n))}
	\int_{(i_0-1)/n}^1 \frac{1}{\sin^2(\pi x/2)} dx =
	\frac{2\cot(\pi i_0/(2n))}{4^{k+1} \pi \sin^{2k}(\pi i_0/(2n))}  
	\leq \frac{1}{4^{k+1}\pi} \left(\frac{2n}{\pi i_0}\right)^{2k+1},  
	\end{equation*}
	the last step using a Taylor expansion around 0. Hence to choose a
	tight a bound as possible in Theorem \ref{thm:incoherence}, we seek to
	balance $i_0$ with $\sqrt{(n/i_0)^{2k+1} \log{n}} \cdot \|\op^{(k+1)}
	\beta_0\|_1$. This is 
	accomplished by choosing
	\begin{equation*}
	i_0 =  n^{\frac{2k+1}{2k+3}} (\log{n})^{\frac{1}{2k+3}} 
	\|\op^{(k+1)}\beta_0\|_1^{\frac{2}{2k+3}}, 
	\end{equation*}
	and applying Theorem \ref{thm:incoherence} gives the result for $k$
	odd.
	
	\smallskip\smallskip
	\noindent
	{\bf Case for $k$ even.} When $k$ is even, we instead have
	\smash{$\op=\op^{(k+1)} = DL^{k/2}$}, where $D$ is the edge incidence
	matrix of a 1d chain, and $L=D^\top D$. It is
	clear that the left singular vectors of $DL^{k/2}$ are simply the left
	singular vectors of $D$, or equivalently, the eigenvectors of
	$DD^\top$. 
	To be explicit, 
	\begin{equation*}
	DD^\top = \left[\begin{array}{rrrrrr}
	2 & -1 & 0 & \ldots & 0 & 0  \\
	-1 & 2 & -1 & \ldots & 0 & 0 \\
	0 & -1 & 2 & \ldots & 0 & 0 \\
	\vdots & & \ddots & \ddots & \ddots & \\
	0 & 0 & \ldots & -1 & 2 & -1 \\
	0 & 0 & \ldots & 0 & -1 & 2 \\
	\end{array}\right],
	\end{equation*}
	which is called the finite
	difference operator associated with the 1d Laplace equation under
	Dirichlet boundary conditions in numerical methods \citep[e.g.,][]{conte1980,godunov1987},
	and is known to have eigenvectors 
	\begin{equation*}
	u_{ij} = \sqrt{\frac{2}{n}} \sin\Big(\frac{\pi ij}{n} \Big),
	\;\;\; \text{for $i,j = 1,\ldots n-1$}.
	\end{equation*}
	It is evident that these vectors are incoherent, with constant
	$\mu=\sqrt{2}$.  Furthermore, the singular values of $DL^{k/2}$ are 
	exactly the eigenvalues of $L$ raised to the power of $(k+1)/2$, and
	the remainder of the proof goes through as in the $k$ odd case.
	
	\subsection{Proof of Corollary \ref{cor:grid2}}
	
	Again we treat the $k$ odd and even cases separately.
	
	\smallskip\smallskip
	\noindent
	{\bf Case for $k$ odd.}  As $k$ is odd, the GTF operator is 
	\smash{$\op=\op^{(k+1)}=L^{(k+1)/2}$}, where the $L$ is the
	Laplacian matrix of a 2d grid graph.  Writing \smash{$L_{\mathrm{1d}} 
		\in \R^{\ell \times \ell}$} for the Laplacian matrix over a 1d grid of
	size $\ell=\sqrt{n}$ (and $I \in \R^{\ell\times \ell}$ for the identity
	matrix), we note that 
	\begin{equation*}
	L = I \otimes L_{\mathrm{1d}} + L_{\mathrm{1d}} \otimes I,
	\end{equation*}
	i.e., the 2d grid Laplacian $L$ is the Kronecker sum of the 1d grid
	Laplacian \smash{$L_{\mathrm{1d}}$}, so its eigenvectors are given by
	all pairwise Kronecker products of eigenvectors of
	\smash{$L_{\mathrm{1d}}$}, of the form $u_i \otimes u_j$.  Moreover, it 
	is not hard to see that each $u_i \otimes u_j$ has unit norm (since 
	$u_i,u_j$ do) and $\|u_i\otimes u_j\|_{\infty}\leq 2/\sqrt{n}$. This
	allows us to conclude that the eigenvectors of $L$ obey the
	incoherence property with $\mu=2$. 
	
	The eigenvalues of $L$ are given by all
	pairwise sums of eigenvalues in the 1d case. Indexing by 2d grid
	coordinates, we may write these as
	\begin{equation*}
	\xi_{j_1,j_2} = 4 \sin^2\Big(\frac{\pi(j_1-1)}{2\ell}\Big) + 
	4 \sin^2\Big(\frac{\pi(j_2-1)}{2\ell}\Big), \;\;\;
	\text{for $j_1,j_2 = 1,\ldots \ell$}.
	\end{equation*}
	Eigenvalues of $L^{(k+1)/2}$ are just given by raising the above 
	to the power of $(k+1)/2$, and for $j_0 \in \{1,\ldots \ell\}$,
	we let $i_0=j_0^2$, and compute the sum
	\begin{equation*}
	\frac{1}{n} \sum_{\max\{j_1 j_2\} > j_0}  
	\frac{1}{\xi_{j_1,j_2}^{k+1}} \leq
	\frac{2}{n} \sum_{j_1=j_0+1}^\ell \sum_{j_2=1}^\ell 
	\frac{1}{\xi_{j_1,j_2}^{k+1}} \leq
	\frac{2}{\ell} \sum_{j_1=j_0+1}^\ell \frac{1}
	{4^{k+1} \sin^{2k+2}(\pi (j_1-1)/(2\ell))}.
	\end{equation*}
	Just as we argued in the 1d case (for $k$ odd), the above is bounded by  
	\begin{equation*}
	\frac{2}{4^{k+1}\pi}\left(\frac{2\ell}{\pi j_0}\right)^{2k+1},  
	\end{equation*}
	and thus we seek to balance $i_0=j_0^2$ with 
	\smash{$\sqrt{(\ell/j_0)^{2k+1} \log{n}} \cdot
		\|\op^{(k+1)}\beta_0\|_1$}.  This yields
	\begin{equation*}
	j_0 = \ell^{\frac{2k+1}{2k+5}} (\log{n})^{\frac{1}{2k+5}}
	\|\op^{(k+1)}\beta_0\|_1^{\frac{2}{2k+5}},  
	\end{equation*}
	i.e.,
	\begin{equation*}
	i_0 = n^{\frac{2k+1}{2k+5}} (\log{n})^{\frac{2}{2k+5}} 
	\|\op^{(k+1)}\beta_0\|_1^{\frac{4}{2k+5}},
	\end{equation*}
	and applying Theorem \ref{thm:incoherence} gives the result for $k$
	odd. 
	
	\smallskip\smallskip
	\noindent
	{\bf Case for $k$ even.} For $k$ even, we have the GTF operator being 
	\smash{$\op=\op^{(k+1)}=DL^{k/2}$}, where $D$ is the edge incidence
	matrix of a 2d grid, and $L=D^\top D$.  It will be helpful to write 
	\begin{equation*}
	D = \left[\begin{array}{c} 
	I \otimes D_{\mathrm{1d}} \\
	D_{\mathrm{1d}} \otimes I 
	\end{array}\right],
	\end{equation*}
	where \smash{$D_{\mathrm{1d}} \in \R^{(\ell-1)\times \ell}$} is the
	difference operator for a 1d grid of size $\ell=\sqrt{n}$ (and $I \in  
	\R^{\ell\times \ell}$ is the identity matrix).  It suffices
	to check the incoherence of the left singular vectors of $DL^{k/2}$,
	since the eigenvalues of $DL^{k/2}$ are those of $L$ raised to
	the power of $(k+1)/2$, and so the rest of the proof then follows
	precisely as in the case when $k$ is odd.  The left singular vectors
	of $DL^{k/2}$  are the same as the left singular vectors of $D$, which 
	are the eigenvectors of $DD^\top$.  Observe that 
	\begin{equation*}
	DD^\top = \left[\begin{array}{cc} 
	I \otimes D_{\mathrm{1d}}D_{\mathrm{1d}}^\top  & 
	D_{\mathrm{1d}}^\top \otimes D_{\mathrm{1d}}  \smallskip \\
	D_{\mathrm{1d}} \otimes D_{\mathrm{1d}}^\top & 
	D_{\mathrm{1d}}D_{\mathrm{1d}}^\top \otimes I 
	\end{array}\right].
	\end{equation*}
	Let $u_i$, $i=1,\ldots  \ell-1$ be the eigenvectors of 
	\smash{$D_{\mathrm{1d}}D_{\mathrm{1d}}^\top$}, corresponding to
	eigenvalues $\lambda_i$, $i=1,\ldots \ell-1$. Define
	\smash{$v_i=D_{\mathrm{1d}}^\top u_i/\sqrt{\lambda_i}$}, 
	$i=1,\ldots \ell-1$, and \smash{$e=\mathds{1}/\sqrt{\ell}$}, where    
	\smash{$\mathds{1}=(1,\ldots 1) \in \R^{\ell}$} is the vector of all
	1s. A straightforward calculation verifies that     
	\begin{align*}
	DD^\top 
	\left[\begin{array}{c}  
	v_i \otimes u_i \\
	u_i \otimes v_i 
	\end{array}\right] &= 
	2\lambda_i
	\left[\begin{array}{c}  
	v_i \otimes u_i \\
	u_i \otimes v_i 
	\end{array}\right], \;\;\;\text{for $i=1,\ldots \ell-1$}, \\
	DD^\top 
	\left[\begin{array}{c}  
	e \otimes u_i \\
	0
	\end{array}\right] &= 
	\lambda_i
	\left[\begin{array}{c}  
	e \otimes u_i \\
	0
	\end{array}\right], \;\;\;\text{for $i=1,\ldots \ell-1$}, \\
	DD^\top 
	\left[\begin{array}{c}  
	0 \\
	u_i \otimes e 
	\end{array}\right] &= 
	\lambda_i
	\left[\begin{array}{c}  
	0 \\
	u_i \otimes e
	\end{array}\right], \;\;\;\text{for $i=1,\ldots \ell-1$}. 
	\end{align*}
	Hence we have derived $3(\ell-1)$ eigenvectors of $DD^\top$. 
	Note that the vectors $v_i$, $i=1,\ldots \ell-1$ are actually the
	eigenvectors of 
	\smash{$L_{\mathrm{1d}}=D_{\mathrm{1d}}^\top D_{\mathrm{1d}}$} 
	(corresponding to the $\ell-1$ nonzero eigenvalues), and from our work
	in the 1d case, recall, both $v_i$, $i=1,\ldots \ell-1$ (studied for
	$k$ odd) and $u_i$, $i=1,\ldots \ell-1$ (studied for $k$ even) are
	unit vectors satisfying the incoherence property with $\mu=\sqrt{2}$.
	This means that the above eigenvectors are all unit norm, and are also 
	incoherent, with constant $\mu=2$. 
	
	There are
	$(\ell-1)(\ell-2)$ more eigenvectors of $DD^\top$, as the rank of
	$DD^\top$ is $n-1=\ell^2-1$. A somewhat longer but still
	straightforward calculation verifies that 
	\begin{align*}
	DD^\top 
	\left[\begin{array}{c}  
	v_i \otimes u_j + v_j \otimes u_i \smallskip \\
	\sqrt{\frac{\lambda_i}{\lambda_j}} u_i \otimes v_j + 
	\sqrt{\frac{\lambda_j}{\lambda_i}}  u_j \otimes v_i
	\end{array}\right] &= 
	(\lambda_i+\lambda_j)
	\left[\begin{array}{c}  
	v_i \otimes u_j + v_j \otimes u_i \smallskip \\
	\sqrt{\frac{\lambda_i}{\lambda_j}} u_i \otimes v_j + 
	\sqrt{\frac{\lambda_j}{\lambda_i}}  u_j \otimes v_i
	\end{array}\right],\text{ for $i<j$}, \\
	DD^\top 
	\left[\begin{array}{c}  
	\sqrt{\frac{\lambda_j}{\lambda_i}} v_i \otimes u_j +
	\sqrt{\frac{\lambda_i}{\lambda_j}} v_j \otimes u_i \smallskip \\
	u_i \otimes v_j + u_j \otimes v_i
	\end{array}\right] &= 
	(\lambda_i+\lambda_j)
	\left[\begin{array}{c}  
	\sqrt{\frac{\lambda_j}{\lambda_i}} v_i \otimes u_j +
	\sqrt{\frac{\lambda_i}{\lambda_j}} v_j \otimes u_i \smallskip \\
	u_i \otimes v_j + u_j \otimes v_i
	\end{array}\right],\text{ for $i<j$}.
	\end{align*}
	Modulo the appropriate normalization, we have derived the remaining
	$(\ell-1)(\ell-2)$ eigenvectors of $DD^\top$.  It remains to check
	their incoherence, once we have normalized them (to have unit norm).
	As the eigenvectors in the first and second expressions above are
	simply (block) rearrangements of each other, it does not matter which
	form we study; consider, say, those in the second expression, and fix
	$i<j$.  The entrywise absolute maximum of the eigenvector in question
	is at most \smash{$\sqrt{\lambda_j/\lambda_i}(4/\sqrt{n})$}.  Thus it 
	suffices show that the normalization constant for this eigenvector is
	on the order of \smash{$\sqrt{\lambda_j/\lambda_i}$}.  Observe that 
	\begin{equation*}
	\left\|\left[\begin{array}{c}  
	\sqrt{\frac{\lambda_j}{\lambda_i}} v_i \otimes u_j +
	\sqrt{\frac{\lambda_i}{\lambda_j}} v_j \otimes u_i \smallskip \\ 
	u_i \otimes v_j + u_j \otimes v_i
	\end{array}\right]\right\|_2^2 = \left\|\left[\begin{array}{c}   
	\sqrt{\frac{\lambda_j}{\lambda_i}} v_i \otimes u_j\smallskip \\
	u_i \otimes v_j
	\end{array}\right]\right\|_2^2 + \left\|\left[\begin{array}{c}  
	\sqrt{\frac{\lambda_i}{\lambda_j}} v_j \otimes u_i \smallskip \\
	u_j \otimes v_i
	\end{array}\right]\right\|_2^2.
	\end{equation*}
	Here the cross-term is
	\smash{$(v_i^\top
		\otimes u_j^\top)(v_j\otimes u_i)=(v_i^\top v_j)(u_i^\top u_j)=0$},
	as \smash{$v_i^\top v_j = 0$} and \smash{$u_i^\top u_j=0$}.
	This means that the normalization constant lies within 
	\smash{$[\sqrt{\lambda_j/\lambda_i+2},
		\sqrt{2\lambda_j/\lambda_i+2}]$}. In particular, the lower bound 
	shows that the incoherence property holds with $\mu = 4$. This 
	completes the proof.

	\subsection{Proof of Lemma \ref{lem:fractional}}
	
	As before, we follow the proof of Theorem \ref{thm:basic} up until the
	application of Holder's inequality in \eqref{eq:holder}, but we use
	the fractional bound in \eqref{eq:fractional} instead.  We claim that
	this implies
	\begin{equation*}
	\epsilon^\top P_R (\tbeta-\beta_0) \leq
	\tK \|\tbeta-\beta_0\|_R^{1-w/2}
	(\|\op\tbeta\|_1 + \|\op\beta_0\|_1)^{w/2},
	\end{equation*}
	where \smash{$\tK=O_\P(K)$}.  This is verified by noting that
	\smash{$x = P_R(\tbeta-\beta_0)/(\|\op\tbeta\|_1 + \|\op\beta_0\|_1)
		\in \cS_\op(1)$}, applying \eqref{eq:fractional} to $x$, and then
	rearranging.  Therefore, as in the proof of Theorem \ref{thm:basic},
	we have
	\begin{equation}
	\label{eq:basic_frac}
	\|\tbeta-\beta_0\|_R^2 \leq
	2 \tK \|\tbeta-\beta_0\|_R^{1-w/2}
	(\|\op\tbeta\|_1 + \|\op\beta_0\|_1)^{w/2} +
	2\lambda (\|\op\beta_0\|_1 - \|\op\tbeta\|_1),
	\end{equation}
	We now set
	\begin{equation*}
	\lambda = \Theta \left(K^{\frac{2}{1+w/2}}
	\|\op \beta_0\|_1^{-\frac{1-w/2}{1+w/2}} \right),
	\end{equation*}
	and in the spirit of \citet{vandegeer1990,locadapt}, we proceed to
	argue in cases.   
	
	\smallskip\smallskip
	\noindent
	{\bf Case 1.}  Suppose that
	\smash{$\frac{1}{2}\|\op\tbeta\|_1 \geq \|\op\beta_0\|_1$}.
	Then we see that \eqref{eq:basic_frac} implies
	\begin{equation}
	\label{eq:basic_frac_case1}
	0 \leq \|\tbeta-\beta_0\|_R^2 \leq
	\tK \|\tbeta-\beta_0\|_R^{1-w/2}
	\Big(\frac{3}{2}\Big)^{w/2}
	\|\op\tbeta\|_1^{w/2} -\lambda \|\op\tbeta\|_1,
	\end{equation}
	so that
	\begin{equation*}
	\lambda \|\op\tbeta\|_1 \leq \tK \|\tbeta-\beta_0\|_R^{1-w/2} 
	\|\op\tbeta\|_1^{w/2},
	\end{equation*}
	where for simplicity have absorbed a constant factor $2(3/2)^{w/2}$ 
	into \smash{$\tK$} (since this does not change the fact that
	\smash{$\tK=O_\P(K)$}), and thus   
	\begin{equation*}
	\|\op\tbeta\|_1 \leq
	\Big(\frac{\tK}{\lambda}\Big)^{\frac{1}{1-w/2}} \|\tbeta-\beta_0\|_R.
	\end{equation*}
	Plugging this back into \eqref{eq:basic_frac_case1} gives
	\begin{equation*}
	\|\tbeta-\beta_0\|_R^2 \leq \tK \|\tbeta-\beta_0\|^{1-w/2}_R
	\Big(\frac{\tK}{\lambda}\Big)^{\frac{w/2}{1-w/2}}
	\|\tbeta-\beta_0\|_R^{w/2},
	\end{equation*}
	or
	\begin{equation*}
	\|\tbeta-\beta_0\|_R \leq \tK^{\frac{1}{1+w/2}}
	\Big(\frac{1}{\lambda}\Big)^{\frac{w/2}{1-w/2}}
	= O_\P\left(K^{\frac{1}{1+w/2}}
	\|\op\beta_0\|_1^{\frac{w/2}{1+w/2}} \right),
	\end{equation*}
	as desired.
	
	\smallskip\smallskip
	\noindent
	{\bf Case 2.} Suppose that
	\smash{$\frac{1}{2}\|\op\tbeta\|_1 \leq \|\op\beta_0\|_1$}.
	Then from \eqref{eq:basic_frac},
	\begin{equation*}
	\|\tbeta-\beta_0\|_R^2 \leq
	\underbrace{2\lambda \|\op\beta_0\|_1}_{a} +
	\underbrace{2\tK \|\tbeta-\beta_0\|_R^{1-w/2}
		3^{w/2} \|\op\beta_0\|_1^{w/2}}_{b},
	\end{equation*}
	and hence either \smash{$\|\tbeta-\beta_0\|_R^2 \leq 2a$}, or
	\smash{$\|\tbeta-\beta_0\|_R^2 \leq 2b$}, and $a \leq b$.  The first
	subcase is straightforward and leads to
	\begin{equation*}
	\|\tbeta-\beta_0\|_R \leq 2\sqrt{\lambda\|\op\beta_0\|_1} = 
	O_\P\left(K^{\frac{1}{1+w/2}}
	\|\op\beta_0\|_1^{\frac{w/2}{1+w/2}} \right),
	\end{equation*}
	as desired.  In the second subcase, we have by assumption
	\begin{align}
	\label{eq:lossbd}
	\|\tbeta-\beta_0\|_R^2 &\leq
	2\tK \|\tbeta-\beta_0\|_R^{1-w/2}
	\|\op\beta_0\|_1^{w/2}, \\
	\label{eq:penbd}
	2\lambda \|\op\beta_0\|_1 & \leq
	\tK \|\tbeta-\beta_0\|_R^{1-w/2}
	\|\op\beta_0\|_1^{w/2},
	\end{align}
	where again we have absorbed a constant factor $2(3^{w/2})$ into  
	\smash{$\tK$}.  Working from \eqref{eq:penbd}, we derive
	\begin{equation*}
	\|\op\beta_0\|_1 \leq
	\Big(\frac{\tK}{2\lambda}\Big)^{\frac{1}{1-w/2}} \|\tbeta-\beta_0\|_R,
	\end{equation*}
	and plugging this back into \eqref{eq:lossbd}, we see
	\begin{equation*}
	\|\tbeta-\beta_0\|_R^2 \leq 2\tK \|\tbeta-\beta_0\|^{1-w/2}_R
	\Big(\frac{\tK}{2\lambda}\Big)^{\frac{w/2}{1-w/2}}
	\|\tbeta-\beta_0\|_R^{w/2},
	\end{equation*}
	and finally
	\begin{equation*}
	\|\tbeta-\beta_0\|_R \leq 2\tK^{\frac{1}{1+w/2}}
	\Big(\frac{1}{\lambda}\Big)^{\frac{w/2}{1-w/2}}
	= O_\P\left(K^{\frac{1}{1+w/2}}
	\|\op\beta_0\|_1^{\frac{w/2}{1+w/2}} \right).
	\end{equation*}
	This completes the second case, and the proof.
	
	\subsection{Proof of Theorem \ref{thm:entropy}}
	
	The proof follows closely from Lemma 3.5 of \citet{vandegeer1990}.
	However, this author uses a different problem scaling than ours, so
	some care must be taken in applying the lemma.  First we abbreviate 
	\smash{$\cS=\cS_\op(1)$}, and define 
	\smash{$\tS= \cS \cdot \sqrt{n}/M$}, where recall $M$ is the  
	maximum column norm of $\op^\dag$.  Now it is not hard to check
	that 
	\begin{equation*}
	\cS = \{ x \in \row(\op) : \|\op x\|_1 \leq 1\} = 
	\op^\dag \{ \alpha \in \col(\op) : \|\alpha\|_1 \leq 1\},
	\end{equation*}
	so that 
	\smash{$\max_{x \in \cS} \|x\|_2 \leq M$}, and 
	\smash{$\max_{x \in \tS} \|x\|_2 \leq \sqrt{n}$}.  This is important
	because Lemma 3.5 of \citet{vandegeer1990} concerns a form of 
	``local'' entropy that allows for deviations on the order of 
	$\sqrt{n}$ in the norm $\|\cdot\|_2$, or equivalently, constant order
	in the scaled metric $\|\cdot\|_n=\|\cdot\|_2/\sqrt{n}$.  Hence, the 
	entropy bound in \eqref{eq:entropy} translates into
	\begin{equation*}
	\log N(\delta, \tS,\|\cdot\|_2) \leq E 
	\Big(\frac{\sqrt{n}}{M}\Big)^w
	\Big(\frac{\sqrt{n}}{\delta}\Big)^w,
	\end{equation*}
	that is,
	\begin{equation*}
	\log N(\delta, \tS,\|\cdot\|_n) \leq E 
	\Big(\frac{\sqrt{n}}{M}\Big)^w \delta^{-w}.
	\end{equation*}
	Now we apply Lemma 3.5 of \citet{vandegeer1990}:
	in the scaled metric used by this author,
	\begin{equation*}
	\max_{x \in \tS} \frac{\epsilon^\top x}
	{\sqrt{n}\|x\|_n^{1-w/2}} = O_\P\left(\sqrt{E} 
	\Big(\frac{\sqrt{n}}{M}\Big)^{w/2} \right),
	\end{equation*}
	that is,
	\begin{equation*}
	\max_{x \in \tS} \frac{\epsilon^\top x}
	{\|x\|_2^{1-w/2}} = O_\P\left(\sqrt{E} 
	\big(\sqrt{n}\big)^{w/2}
	\Big(\frac{\sqrt{n}}{M}\Big)^{w/2} \right),
	\end{equation*}
	and finally,
	\begin{equation*}
	\max_{x \in \cS} \frac{\epsilon^\top x}
	{\|x\|_2^{1-w/2}} = O_\P\left(\sqrt{E} 
	\big(\sqrt{n}\big)^{w/2}\right),
	\end{equation*}
	as desired.

	\subsection{Proof of Corollary \ref{lem:atoms}}
	
	For each $j=1,\ldots 2r$, 
	if $\cG$ is covered by $j$ balls having radius at most  
	\smash{$C_0 \sqrt{n} j^{-1/\zeta}$}, with respect to the norm 
	$\|\cdot\|_2$, then it is covered by $j$ balls having radius at  
	most $C_0 j^{-1/\zeta}$, with respect to the scaled norm 
	$\|\cdot\|_n=\|\cdot\|_2/\sqrt{n}$.  By Theorem 1 of
	\citet{carl1997metric}, this implies that for each $j=1,2,3,\ldots$, 
	the convex hull $\conv(\cG)$ is covered by $2^j$ balls having radius
	at most $C_0' j^{-(1/2+1/\zeta)}$, with respect to $\|\cdot\|_n$,
	for another constant $C_0'$.  Converting this back to an entropy 
	bound in our original metric, and noting that
	\smash{$\conv(\cG)=\cS_\op(1)$}, we have 
	\begin{equation*}
	\log(\delta, \cS_\op(1), \|\cdot\|_2) \leq C_0''
	\Big(\frac{\sqrt{n}}{\delta}\Big)^{\frac{1}{1/2+1/\zeta}},  
	\end{equation*}
	for a constant $C_0''$, as needed.  This proves the lemma.
	
	\subsection{Proof of Corollary \ref{cor:1dfl}}
	
	According to Lemma \ref{lem:dpinv}, we know
	that \smash{$(D^{(1)})^\dag = P_{\mathds{1}}^\perp H$}, where $H$ is an  
	$n\times (n-1)$ lower triangular matrix with $H_{ij}=1$ if $i > j$ and  
	0 otherwise, and $P_{\mathds{1}}^\perp$ is the projection map
	orthogonal to the all 1s vector.  Thus $g_i = P_{\mathds{1}}^\perp
	h_i$, $i=1,\ldots n-1$, with $h_1,\ldots h_{n-1}$ denoting the columns
	of $H$.  It is immediately 
	apparent that 
	\begin{equation*}
	\|g_i-g_\ell\|_2 \leq \|h_i-h_\ell\|_2 \leq \sqrt{i-\ell},
	\end{equation*}
	for all $i>\ell$.  Now, given $2j$ balls at our disposal, consider
	centering the first $j$ balls at 
	\begin{equation*}
	g_d, g_{2d}, \ldots g_{j d},
	\end{equation*}
	where $d=\lfloor n/j \rfloor$.  Also let these balls have radius
	$\sqrt{n/j}$. By construction, then, we see that 
	\begin{equation*}
	\|g_1-g_d\|_2 \leq \sqrt{n/j}, \;
	\|g_d-g_{2d}\|_2 \leq \sqrt{n/j}, \;
	\ldots
	\|g_{jd}-g_{n-1}\|_2 \leq \sqrt{n/j},
	\end{equation*}
	which means that we have covered $g_1,\ldots g_{n-1}$ with $j$
	balls of radius $\sqrt{n/j}$.
	
	We can cover $-g_1,\ldots,-g_{n-1}$
	with the remaining $j$ balls analogously. Therefore, we have shown
	that $2j$ balls require a radius of $\sqrt{n/j}$, or in other words,
	$j$ balls require a radius of $\sqrt{2n/j}$.

\bibliographystyle{abbrvnat}
\bibliography{ryantibs,bibfile}

\end{document}